\documentclass[final,1p,times]{elsarticle}

\usepackage{amssymb}
\usepackage{amsmath}
\usepackage{amsthm}
\usepackage{graphicx}
\usepackage{subcaption}
\usepackage{algorithm}
\usepackage{algpseudocode}
\usepackage{float}
\usepackage{xcolor}
\usepackage{url}
\usepackage{lineno}
\usepackage{hyperref}
\usepackage{orcidlink}

\newcommand{\comment}[1]{}
\newcommand{\eq}{Eq.}
\newcommand{\fig}{Fig.}
\newcommand{\alg}{Algorithm }
\newcommand{\eg}{\emph{e.g.},\ }

\newcommand{\ie}{\emph{i.e.},\ }
\newcommand{\etal}{~\emph{et al.}\ }

\journal{Journal of Computational Physics}

\begin{document}

\begin{frontmatter}

\title{A Streaming Sparse Cholesky Method for Derivative-Informed Gaussian Process Surrogates Within Digital Twin Applications}

\author[label1,label2]{Shridhar Vashishtha\,\orcidlink{0009-0003-3140-1069}\fnref{equal}}  
\ead{shridhar.vashishtha@utah.edu}
\author[label3]{Krishna Prasath Logakannan\,\orcidlink{0000-0001-7892-1924}\fnref{equal}}  
\ead{krishna.logakannan@utah.edu}
\author[label3]{Jacob Hochhalter\,\orcidlink{0000-0003-3607-4573}} 
\ead{jacob.hochhalter@utah.edu}
\author[label1]{Shandian Zhe\,\orcidlink{0000-0003-0316-9875}} 
\ead{u6014464@utah.edu}
\author[label1,label2]{Robert M. Kirby\,\orcidlink{0000-0001-5712-4141
}\corref{cor1}} 
\ead{mike.kirby@utah.edu}

\fntext[equal]{These authors contributed equally to this work.}
\cortext[cor1]{Corresponding author.}

\affiliation[label1]{organization={Kahlert School of Computing, University of Utah},
           city={Salt Lake City},
           postcode={84112}, 
           state={UT},
           country={USA}}

\affiliation[label2]{organization={Scientific Computing \& Imaging Institute, University of Utah},
            addressline={}, 
           city={Salt Lake City},
           postcode={84112}, 
           state={UT},
           country={USA}}

\affiliation[label3]{organization={Department of Mechanical Engineering, University of Utah},
            addressline={}, 
           city={Salt Lake City},
           postcode={84112}, 
           state={UT},
           country={USA}}    
   
\begin{abstract}
Digital twins are developed to model the behavior of a specific physical asset (or twin), and they can consist of high-fidelity physics-based models or surrogates. A highly accurate surrogate is often preferred over multi-physics models as they enable forecasting the physical twin future state in real-time. To adapt to a specific physical twin, the digital twin model must be updated using in-service data from that physical twin. In this paper,  we combine and extend several previous surrogate-related advancements with the goal of demonstrating an end-to-end digital twin (DT) solution for predicting performance of an aircraft structure (the physical asset). To this end, we extend Gaussian process (GP) models to include derivative data, for improved accuracy, with dynamic updating to ingest physical twin data during service. Including derivative data, however, comes at a prohibitive cost of increased covariance matrix dimension. We circumvent this issue through our modified dynamic sparse Cholesky linear system solver.  
Numerical experiments demonstrate that the prediction accuracy of the derivative-enhanced sparse Cholesky GP method produces improved models upon dynamic data additions. Lastly, we demonstrate the developed algorithm within a DT framework to model fatigue crack growth in an aerospace vehicle, thereby exhibiting through our assembled engineered system how digital twin technologies can be combined in practice. 
\end{abstract}


\end{frontmatter}

\section{Introduction}

Coining of the term digital twin (DT) is often attributed to the US Air Force and NASA in the early 2000s as a means of creating a dynamic, high-fidelity digital framework to monitor, simulate, and predict a specific physical component or system \cite{tuegel2011,glaessgen2012}. Initially, the DT concept was motivated by the need to improve structural reliability estimates during the operational life of aerospace vehicles. The DT concept has since been extended to manufacturing, automation, energy and utilities, healthcare, etc., where physical components or systems would benefit from continuously updated assessments or optimization during their life cycle \cite{grieves2023}. Independent of the particular application, the fundamental objective of a DT model is to improve the accuracy of predictions and reduce uncertainty. This is accomplished by incorporating measurable characteristics of an individual physical asset prior to and during its operational life, as opposed to using a nominal or purely stochastic modeling approach.

In an attempt to remain up-to-date with its physical twin (PT), a DT framework often employs surrogate modeling methods \cite{kapteyn2022}. A central challenge of a DT model, therefore, lies in balancing fidelity with computational efficiency: full-scale physics-based simulations, which in and of themselves are surrogates but ones in which we try to ‘throw away’ as little as possible, would often exceed real-time modeling constraints, while reduced-order models risk losing critical individualized accuracy. Once a DT modeling approach is identified, with an appropriate balance of fidelity and efficiency, an initial DT model can be generated (trained) using, \eg as-manufactured PT characteristics \cite{cerrone2014}, nominal material information used during the design stage, or performance data from preceding PTs.  While capturing both known (measurable) and unknown (nominal or stochastic) individual characteristics is a critical starting point, this only marks the initialization of the DT modeling process, as outlined in black in Figure \ref{fig:DT_workflow_generic}.

As with the study of biological twins — where identical genetic starting material does not guarantee identical outcomes — DTs must account for manufacturing variability, external perturbations, operating environments, and other system variability that act on the PT during its life. To this end, the DT model and its state must be updated periodically to reflect the measured in-service usage and evolution of the PT. This requisite co-evolution occurs through structural modifications to the DT model form and parameter recalibration as informed by new data, as outlined in red in Figure \ref{fig:DT_workflow_generic}.  In this step, there exists an implicit assumption that the DT and PT are not only co-evolving, but also that smoothness in time of the PT is mimicked in the DT. Hence, the DT must not only be responsive but also smoothly adaptive. Lastly, a two-way communication is necessitated in which the DT acquires measured PT usage and state, which are in turn used to generate updated predictions for the PT and used as a basis of decision making (\eg for maintenance or replacement), as outlined in green in Figure \ref{fig:DT_workflow_generic}. The integration of all these requirements and objectives defines both the promise and the frontier of DT research.

\begin{figure}[b!]
	\centering
	\includegraphics[trim=0 10cm 3cm 0, clip,width=1.1\linewidth]{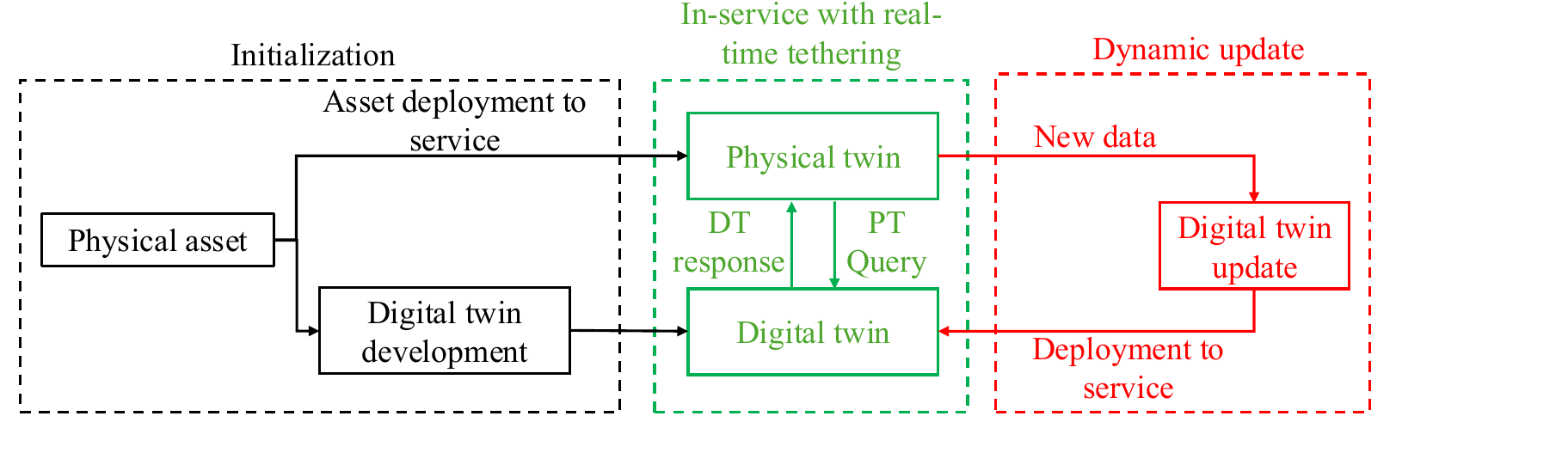}
	\caption{Schematic of workflow of a typical DT system with real-time tethering and dynamic update capability.}
	\label{fig:DT_workflow_generic}
\end{figure}

The general DT concept is abstract, so any quantitative assessment requires an application problem. Here, we select the original motivating application: aircraft structural life prediction \cite{tuegel2011}. The objective of this application problem is to use a DT to establish a condition-based maintenance schedule, as opposed to current conservative methods of scheduled maintenance that are based entirely on flight hours but not individual usage or state, \eg safe life or damage tolerance approaches \cite{FAA_AC_23_13A,FAA_AC_25_571_1D}.  Establishing a reliable DT in this application has direct implications for improved reliability, safety, and cost-effectiveness: an accurate DT model can reduce the frequency of unnecessary maintenance interventions (where issues can be introduced), thereby minimizing system downtime and associated costs. Since the seminal study presented by Tuegel \etal \cite{tuegel2011}, which focused on the use of ``ultrahigh fidelity models,'' research has extended to simulated demonstrations along with the incorporation of surrogate modeling and uncertainty quantification. Liao \etal \cite{Liao2020} assessed the shift from deterministic individual aircraft tracking to a probabilistic, flight-by-flight approach using Bayesian updating and probabilistic fatigue crack growth models. They demonstrated that DTs can reduce conservatism in maintenance planning, improve fatigue life predictions, and enable risk-based decision-making tailored to individual aircraft. Millwater \etal \cite{Millwater2020} presented probabilistic methods for risk assessment in DT frameworks, detailing computational strategies for estimating failure probabilities, remaining useful life, and inspection intervals, highlighting the importance of real-time data integration and surrogate modeling. Because fatigue cracking in aerospace materials is inherently a multi-scale problem, several researchers have employed multi-scale DT modeling strategies. As an example, Whelan and McDowell \cite{Whelan2020} model microstructure-sensitive fatigue behavior in Ti-6Al-4V. Their study highlights quantifying epistemic uncertainty from model form and parameters, using statistical volume elements (SVEs) and fatigue indicator parameters (FIPs) to assess fatigue resistance. Yeratapally \etal and Leser \etal present a two-part DT feasibility study that establishes a probabilistic, multi-scale framework for fatigue life prediction in Aluminum alloy 7075-T651. In Part I, Yeratapally \etal \cite{Yeratapally2020} develop a DT approach that couples microstructurally small crack modeling—using crystal plasticity finite element analysis of SVEs with microstructurally large crack (MLC) modeling via linear elastic fracture mechanics, calibrated through Monte Carlo and Markov Chain Monte Carlo methods. In Part II, Leser \etal \cite{Leser2020} integrate in-situ diagnostics using digital image correlation and Bayesian inference to iteratively calibrate the DT model as damage becomes observable, significantly reducing uncertainty in prognostic predictions. Within this application domain, past research has largely focused on Bayesian methods for updating DT model parameters, given new observations during service life, while investigations into the structural modifications to the DT model form remain open.

Our contribution herein is an engineered system that acts as a DT demonstrator, combining and extending derivative-informed surrogate modeling, dynamics updating, and modified sparse linear algebra solution technologies. 
Our interest in derivative-informed Gaussian processes (GPs) is sparked, in part, by recent advances in {\em in situ} automatic differentiation of finite element method (FEM) simulation fields \cite{VelasquezANMR25,NavarroVAMMR26,LogakannanABXZKMH26}. These FEM fields with their corresponding derivative information act directly or as proxies for the physical twin product \cite{Tao18062019,SunZZ24}. When this field and derivative information is available (e.g., from acquiring displacement, velocity, and acceleration data), either directly as part of on-board physical sensors or via these simulation science advances, we employ derivative-informed GPs constructed using a modified dynamic sparse Cholesky linear system solver
to leverage the dual advantages of high predictive accuracy and the ability to incorporate new data dynamically at minimal computational cost.
We present work that advances DT modeling by extending GP surrogates in three key directions. First, we incorporate derivative information into the GP formulation, up to and including fourth-order derivative information, enabling the surrogate to more effectively capture local sensitivities and improve predictive fidelity, particularly in high-dimensional and multi-physics contexts. Second, we utilize sparse Cholesky factorization methods to handle derivative-enhanced covariance structures, thereby preserving computational tractability even as model complexity increases. Prior studies, \cite{chen_sparse_2024, schafer_sparse_2021,chen_solving_2021,kang_correlation-based_2023}, have introduced related ideas to include derivatives within sparse Cholesky factorization. However, our contribution of an engineering system that serves as a DT demonstrator, combining various ideas and aligning them effectively in the DT system, is a meaningful step within the DT research community. This extension allows the surrogate to scale to large datasets while retaining the accuracy benefits of derivative augmentation. Finally, we develop a dynamic update algorithm tailored for the digital twin setting, in which new sensor data and operational measurements arrive sequentially. This algorithm ensures that the GP surrogate can be updated efficiently in real time without retraining from scratch, allowing the DT to evolve in tandem with its physical counterpart. Collectively, these contributions provide a principled and computationally efficient foundation for deploying adaptive, derivative-informed GP surrogates in digital twin applications, with particular relevance to aerospace and mechanical systems where both accuracy and scalability are critical.

The paper is organized as follows. In Section \ref{sec:math}, we present a mathematical description of GP modeling and how it can be updated to incorporate derivative information. We highlight the increased accuracy capabilities of our derivative-enhanced GP when used to model functions that are sufficiently smooth.
In Section \ref{sec:DT_algorithms}, we summarize a previously published sparse Cholesky algorithm with provable properties for solving a GP system \cite{chen_sparse_2024} and present our modifications of this algorithm to account for derivative-enhanced GP. Given our interest in using derivative-enhanced GPs in a DT context, we present in Section \ref{sec:dynamic_GP_for_DT} an extension of the sparse Cholesky algorithm applicable in the dynamic (streaming) context.  We present several variants of the algorithm and demonstrate which works best for DT problems. In Section \ref{sec:App_DT}, we combine our contributions and show how derivative-enhanced GPs solved using our dynamic sparse Cholesky algorithm can be used to solve a real-world DT application. We summarize our work and provide application-oriented observations concerning our DT engineered system solution in Section \ref{sec:summary}.

\section{GP Surrogate Modeling for DT Applications}
\label{sec:math}

We start by laying out the basic building blocks of Gaussian Process (GP) modeling that we will employ and providing a summary of well-known theoretical results that helped guide our work. In particular, we present in Section \ref{sec:kernel_methods_interpolation} a review of kernel methods. In Section \ref{sec:GPmodeling}, we then present a mathematical review of GPs as they are used for surrogate modeling. In Section \ref{sec:GPnoisederivatives}, the generalized case of GP modeling with noisy data is presented with our extension of GPs to utilize both function values and their derivatives up to any arbitrary derivative order $d$. In Section \ref{sec:GPderiv}, we present the specific case of modeling noise-free data within our contribution of derivative-enhanced GPs. Finally, in Section \ref{sec:GPverification}, we present numerical verification experiments to demonstrate the correctness of our algorithm and to highlight the superior convergence properties of derivative-enhanced GPs when applied to problems with sufficient smoothness.

\subsection{Kernel methods}
\label{sec:kernel_methods_interpolation}

We begin by defining the input domain as $\Omega \subset \mathbb{R}^p$. Let $\mathbf{x}^{(i)}$ be samples from $\Omega$, i.e. $\mathbf{x}^{(i)} \in \Omega \subset \mathbb{R}^p$. Assume that the data are given as $\mathcal{D}_{\text{train}} = (\mathbf{x}^{(i)}, f(\mathbf{x}^{(i)}))$ for i = 1, 2, ...., N, where $f(\mathbf{x}^{(i)})$ represents the value of the observed function at $\mathbf{x}^{(i)}$ and $N$ is the number of (unique) samples. We denote the training dataset as follows:
\[
\textbf{X} = \begin{bmatrix} \mathbf{x}^{(1)} \\ \mathbf{x}^{(2)} \\ \vdots \\ \mathbf{x}^{(N)} \end{bmatrix} \in \mathbb{R}^{N \times p}, \quad
\mathbf{f} = \begin{bmatrix} f(\mathbf{x}^{(1)}) \\ f(\mathbf{x}^{(2)}) \\ \vdots \\ f(\mathbf{x}^{(N)}) \end{bmatrix} \in \mathbb{R}^{N}.
\]

Using the notation of the training data defined above, the squared exponential (SE) kernel can be derived from the  Reproducing Kernel Hilbert Space (RKHS), where a positive definite kernel defines an inner product between functions. The SE kernel can be defined as:
\[
k(\mathbf{x}^{(i)},\mathbf{x}^{(j)})=\exp\left(-\frac{|\mathbf{x}^{(i)}-\mathbf{x}^{(j)}|^2}{2\delta^2}\right),
\]

\noindent where $\delta$ is the kernel length scale that controls the smoothness of the function. In this work, $\delta$ is set by an optimizer that finds the best $\delta$ values within the range $[0.001, 1000]$ based on the various hyperparameters of our scheme. Prior studies, \cite{chen_sparse_2024, schafer_sparse_2021,chen_solving_2021,kang_correlation-based_2023}, of sparse Cholesky methods mostly focus on the use of the Mat\'{e}rn kernels \cite{Dowling2021HidaMatrnKH} due to their reduced smoothness and favorable sparsity behaviors. However, we employ the SE kernel because higher-order derivatives (up to and including $4^{th}$-order in this work) require that the kernels be sufficiently differentiable. Extensions to the Mat\'{e}rn kernel with appropriate smoothness parameters remain an important direction for future work.

Using the kernel function $k(\,\cdot\,,\,\cdot\,)$ defined above, we can form the $(N \times N)$ covariance matrix, $\mathbf{K}_{[f,f]}$, computed using the training points, $\{ \mathbf{x}^{(i)} \}_{i=1}^N$. If we assume that our approximation is of the form:

\[
u(\mathbf{x}^{*}) = \sum_{i=1}^{N} \hat{u}_i \, \phi_i(\mathbf{x}^{*}),
\]

\noindent where $\hat{u}_i$ denotes the coefficients of our expansion and $\phi_i(\,\cdot\,)$ denotes our basis functions, then our coefficient vector is given by the expression $\hat{\mathbf{u}} = \mathbf{K}^{-1} \mathbf{f}$ and $\phi_i(\mathbf{x}^{*}) = k(\mathbf{x}^{*}, \mathbf{x}^{(i)})$, where $\mathbf{K}^{-1}$ is the inverse of the covariance matrix. For noisy observations, the posterior variance at a test point $\mathbf{x}^{*}$ is given by the following expression:
\[
\mathrm{Var}[u(\mathbf{x}^{*})] = k(\mathbf{x}^{*}, \mathbf{x}^{*}) - \mathbf{K}(\mathbf{x}^{*}, \textbf{X}) \mathbf{K}(\textbf{X}, \textbf{X})^{-1} \mathbf{K}(\textbf{X}, \mathbf{x}^{*}),
\]
\noindent where $\mathbf{K}(\mathbf{x}^{*}, \mathbf{X}) = [k(\mathbf{x}^{*}, \mathbf{x}^{(1)}), \dots, k(\mathbf{x}^{*}, \mathbf{x}^{(N)})]$. However, if the observations at the training points are noise-free, \(\mathrm{Var}[u(\mathbf{x}^{*})] = 0\), which is equivalent to the interpolation using radial basis functions (RBFs). All experiments in this work are conducted using the squared exponential kernel.

\subsection{Gaussian Process (GP) Modeling}
\label{sec:GPmodeling}

A GP is a stochastic process defined as a collection of random variables indexed by time, space, or a more general input domain, such that any finite collection of these variables follows a joint multivariate Gaussian distribution. This property makes GPs particularly powerful for modeling unknown functions in a nonparametric Bayesian framework. GPs can be modeled by placing a prior over the latent function, \textbf{f}, as:
\[
f(x) \sim \mathcal{GP}(\mu(\cdot), \mathrm{cov}(\cdot, \cdot)),
\]

\noindent where $\mu(\,\cdot\,)$ is the mean function of the process and $\mathrm{cov}(\,\cdot\,, \,\cdot\,)$ is its covariance function. In practice, one often sets $\mu(\,\cdot\,) = 0$ and uses the kernel function as the covariance function, i.e., $\mathrm{cov}(f(\mathbf{x}^{(i)}), f(\mathbf{x}^{(j)})) = k(\mathbf{x}^{(i)}, \mathbf{x}^{(j)})$. As discussed in \cite{fang_solving_2024}, under the GP prior, the function values at \textbf{f} follow a multi-variate Gaussian distribution, $\mathcal{P}$($\mathbf{f}$) = $\mathcal{N}(\mathbf{f}\,|\, \mathbf{0}, k(\mathbf{x}^{(i)}, \mathbf{x}^{(j)}))$. This is commonly referred to as GP projection. Let us assume that we want to compute the distribution of the function value at any input, $\mathbf{x}$, namely $\mathcal{P}(f(\mathbf{x})\,|\, \textbf{f})$. Since \textbf{f} and $f(\mathbf{x})$ are both assumed to follow a multivariate Gaussian distribution, we obtain a conditional Gaussian:
\[
\mathcal{P}(f(\mathbf{x})\,|\,\textbf{f}) = \mathcal{N}\big(f(\mathbf{x})\,|\,\mu(\mathbf{x}), \sigma^2(\mathbf{x})\big),
\]
\noindent where the conditional mean and variance, respectively, are given by
\[
\mu(\mathbf{x}) = \mathrm{cov}(f(\mathbf{x}), \textbf{f})\, \mathbf{K}^{-1}\textbf{f},
\]
and
\[
\sigma^2(\mathbf{x}) = \mathrm{cov}(f(\mathbf{x}), f(\mathbf{x})) - \mathrm{cov}(f(\mathbf{x}), \textbf{f})\,\mathbf{K}^{-1}\,\mathrm{cov}(\textbf{f}, f(\mathbf{x})).
\]
In the expression above, $\mathrm{cov}(f(\mathbf{x}), \textbf{f}) = k(\mathbf{x}, \textbf{X}) = [k(\mathbf{x}, \mathbf{x}^{(1))}, \ldots, k(\mathbf{x}, \mathbf{x}^{(N))}]$ and $\sigma(\cdot)$ denotes the standard deviation.

Since we are using a squared-exponential kernel in this work, the GP prior enforces smoothness and infinite differentiability on the latent function. In the case of traditional GPs, we assume noisy measurements represented as follows:
\[
y_{(i)} = f(\mathbf{x}^{(i)}) + \epsilon_i, \quad \epsilon_i \sim \mathcal{N}(0, \sigma_{n, i}^2)
\]
\noindent where each observation has its own (known or estimated) noise variance, $\sigma_{n, i}^2$, allowing for heteroscedastic noise.

We assume we are handling uncorrelated additive noise of the form:

\[
\textbf{y} = \textbf{f} + \boldsymbol{\epsilon}. 
\]

\noindent where  

\[
\textbf{y} = \begin{bmatrix} y_{(1)} \\ y_{(2)} \\ \vdots \\ y_{(N)} \end{bmatrix} \in \mathbb{R}^{N}, \quad
\boldsymbol{\epsilon} = \begin{bmatrix} \epsilon_{(1)} \\ \epsilon_{(2)} \\ \vdots \\ \epsilon_{(N)} \end{bmatrix} \in \mathbb{R}^{N}.
\]

Since both \textbf{f} and $\boldsymbol{\epsilon}$ are assumed to be Gaussian, \textbf{y} is also Gaussian. Thus, \textbf{y} can be denoted as:
\[
\textbf{y} \sim \mathcal{N}(0, \mathbf{K}_{ff} + \sigma_{n, i}^2 \textbf{I})
\]

\noindent where $\mathbf{K}_{ff}$ is the GP covariance matrix computed using the squared-exponential kernel and $\textbf{I}$ is the identity matrix. This expression can be further simplified to the following:

\[
\textbf{y} \sim \mathcal{N}(0, \mathbf{K}_{ff} + \mathbf{R})
\]

\noindent where, $\mathbf{R} = \sigma_{n, i}^2 \textbf{I}$, is the diagonal heteroscedastic noise matrix.

The mean of the GP posterior acts as the approximator of the latent function. In our case of GP with independent heteroscedastic noise, the posterior mean can be expressed as follows:

\begin{equation}
\bar{f}(\mathbf{x}) = \mathbf{K}(\mathbf{x}^{*}, \mathbf{X}) \big[\mathbf{K}(\mathbf{X}, \mathbf{X}) + \mathbf{R} \big]^{-1} \mathbf{y},
\label{eq:interpolation_eq}
\end{equation}

\noindent where $\mathbf{K}(\mathbf{x}^{*}, \mathbf{X})$ denotes the covariance vector between the test point $\mathbf{x}^{*}$ and the training data \textbf{X}. Note that upon comparison with the noise-free expression previously given, the matrix $\mathbf{K}(\mathbf{X}, \mathbf{X}) + \mathbf{R}$ contains the additional $\mathbf{R}$ term used to model the presence of noise (an observation relevant to our inversion and sparsification discussion below).

\subsection{Derivative-Enhanced Gaussian Process (GP) Surrogate Modeling for Noisy Data}
\label{sec:GPnoisederivatives}
Assuming that we have access to derivative information of $\mathbf{f}$ up to an arbitrary order $d$, the accuracy of the GP can be improved by incorporating these derivatives into the training procedure \cite{solak_derivative_nodate,wu_exploiting_2018,wang_explicit_2020,eriksson_scaling_2018,padidar_scaling_nodate,mukherjee_dgp-lvm_2025}. It is important to note that the kernel function used to model the GP should be sufficiently smooth and differentiable, which is consistent with our choice of the squared-exponential kernel. An alternative choice found in the literature is the family of Mat{\'e}rn kernels \cite{Dowling2021HidaMatrnKH}.

A GP that leverages derivative information can be formulated as $\mathbf{F}\sim \mathcal{GP}(0, \mathbf{K}_{[f,\nabla f, \nabla^2f,..., \nabla ^d f]} + \mathbf{R})$, where $\mathbf{F}$ is a vector that contains noisy observations of $\mathbf{f}$ and its derivatives, $\mathbf{K}_{[f,\nabla f, \nabla^2f,..., \nabla ^d f]}$ is the covariance matrix with derivatives, and $\mathbf{R}$ is the diagonal matrix containing the noise variances of each observed quantity. The formulation of $\mathcal{GP}$ is structured as follows:

\begin{equation}\label{eq:GP_der_noisy}
\begin{bmatrix} 
f \\ 
\nabla f \\ 
\nabla^{2} f \\ 
\vdots \\ 
\nabla^{d} f 
\end{bmatrix}  
\sim 
\mathcal{GP}\left(
0,
\begin{bmatrix} 
\mathbf{K}_{[f,f]} + \sigma_f^2 \mathbf{I} & \mathbf{K}_{[f,\nabla f]} & \mathbf{K}_{[f,\nabla^2 f]} & \cdots & \mathbf{K}_{[f,\nabla^d f]} \\
\mathbf{K}_{[\nabla f,f]} & \mathbf{K}_{[\nabla f,\nabla f]} + \sigma_{\nabla f}^2 \mathbf{I} & \mathbf{K}_{[\nabla f,\nabla^2 f]} & \cdots & \mathbf{K}_{[\nabla f,\nabla^d f]} \\
\vdots & \vdots & \vdots & \ddots & \vdots \\
\mathbf{K}_{[\nabla^d f,f]} & \mathbf{K}_{[\nabla^d f,\nabla f]} & \mathbf{K}_{[\nabla^d f,\nabla^2 f]} & \cdots & \mathbf{K}_{[\nabla^d f,\nabla^d f]} + \sigma_{\nabla^d f}^2 \mathbf{I}
\end{bmatrix}
\right)
\end{equation}

\noindent where $\nabla^df$ represents the $d^{th}$ derivative of function $\mathbf{f}$. The matrix $\mathbf{K}_{[\nabla^n f,\nabla^mf]}$ corresponds to the covariance of the $n^{th}$ and $m^{th}$ derivative observations; the elements are calculated using the derivatives of the kernel. Details regarding derivatives of the squared-exponential kernel can be found in \cite{solak_derivative_nodate,wu_exploiting_2018,mukherjee_dgp-lvm_2025}. For $d>1$, let $\nabla^df$ denote a column vector containing (unique) derivative terms. For example, in our notation, $\nabla^2f$ of a 2D function is written as:
\begin{equation}
    \nabla ^2f=\left[\frac{\partial^2f}{\partial x_1^2},\frac{\partial^2f}{\partial x_1 \partial x_2},\frac{\partial^2f}{\partial x_2 ^2 }\right]^T.
\end{equation}

\noindent The covariance matrix presented in \eq~\ref{eq:GP_der_noisy} can be interpreted as an exact GP covariance matrix perturbed by a diagonal regularizer (e.g., {\em nugget}) $\sigma_f^2 \mathbf{I}$, $\sigma_{\nabla f}^2 \mathbf{I}$, $\cdots$, $\sigma_{\nabla^d f}^2 \mathbf{I}$. 

To aid the reader, consider the case when one has both the field and first-order derivative information.  \eq~\ref{eq:GP_der_noisy} can subsequently be written as follows:

\begin{equation}\label{eq:first_order_GP_der_noisy}
\begin{bmatrix} 
f \\ 
\nabla f 
\end{bmatrix}  
\sim 
\mathcal{GP}\left(
0,
\begin{bmatrix} 
\mathbf{K}_{[f,f]} + \sigma_f^2 \mathbf{I} & \mathbf{K}_{[f,\nabla f]} &  \\
\mathbf{K}_{[\nabla f,f]} & \mathbf{K}_{[\nabla f,\nabla f]} + \sigma_{\nabla f}^2 \mathbf{I}
\end{bmatrix}
\right)
\end{equation}

\noindent i.e., the kernel matrix for relating both the primary correlation structure as well as that of the first derivatives is given by:

\begin{equation}\label{eq:first_order_kernel_GP_der_noisy}
\mathbf{K}_{\text{der}} =
\left[
\begin{matrix}
\mathbf{K}_{[f,f]} + \sigma_f^2 \mathbf{I}
& \mathbf{K}_{[f,\nabla f]} 
  \\

\mathbf{K}_{[\nabla f,f]} 
& \mathbf{K}_{[\nabla f,\nabla f]} + \sigma_{\nabla f}^2 \mathbf{I}
 \\

\end{matrix}
\right]
\end{equation}

\noindent For more details regarding the interpolatory structure of derivative-informed GPs, we refer the reader to Appendix B.3.

\subsection{Noise-Free Derivative-Enhanced Gaussian Process (GP) Surrogate Modeling}
\label{sec:GPderiv}

When there is no noise present in the data, i.e., when the noise variance terms in \eq~\ref{eq:GP_der_noisy} are set to zero, \eq~\ref{eq:GP_der_noisy} can be modified for the noise-free case as:

\begin{equation}\label{eq:GP_der}
\begin{bmatrix} 
f \\ 
\nabla f \\ 
\nabla^{2} f \\ 
\vdots \\ 
\nabla^{d} f 
\end{bmatrix}  
\sim 
\mathcal{GP}\left(
0,
\begin{bmatrix} 
\mathbf{K}_{[f,f]} & \mathbf{K}_{[f,\nabla f]} & \mathbf{K}_{[f,\nabla^2 f]} & \cdots & \mathbf{K}_{[f,\nabla^d f]} \\
\mathbf{K}_{[\nabla f,f]} & \mathbf{K}_{[\nabla f,\nabla f]} & \mathbf{K}_{[\nabla f,\nabla^2 f]} & \cdots & \mathbf{K}_{[\nabla f,\nabla^d f]} \\
\vdots & \vdots & \vdots & \ddots & \vdots \\
\mathbf{K}_{[\nabla^d f,f]} & \mathbf{K}_{[\nabla^d f,\nabla f]} & \mathbf{K}_{[\nabla^d f,\nabla^2 f]} & \cdots & \mathbf{K}_{[\nabla^d f,\nabla^d f]}
\end{bmatrix}
\right).
\end{equation}

For the case when we have both primary and first-order derivative information, 
\eq~\ref{eq:GP_der} can be written as follows:

\begin{equation}\label{eq:first_order_GP_der}
\begin{bmatrix} 
f \\ 
\nabla f 
\end{bmatrix}  
\sim 
\mathcal{GP}\left(
0,
\begin{bmatrix} 
\mathbf{K}_{[f,f]} & \mathbf{K}_{[f,\nabla f]} &  \\
\mathbf{K}_{[\nabla f,f]} & \mathbf{K}_{[\nabla f,\nabla f]}
\end{bmatrix}
\right)
\end{equation}

\noindent i.e., the kernel matrix for relating both the primary correlation structure as well as that of the first derivatives is given by:

\begin{equation}\label{eq:first_order_kernel_GP_der}
\mathbf{K}_{\text{der}} =
\left[
\begin{matrix}
\mathbf{K}_{[f,f]} 
& \mathbf{K}_{[f,\nabla f]} 
  \\

\mathbf{K}_{[\nabla f,f]} 
& \mathbf{K}_{[\nabla f,\nabla f]}
 \\

\end{matrix}
\right]
\end{equation}

The incorporation of derivative information into the GP model increases the size and the complexity of the covariance matrix. Since the numerical stability and conditioning of the GP system depend directly on the properties of the covariance matrix, it is important to establish the properties of the derivative-informed covariance matrix.  The following lemma guarantees the symmetry and positive definiteness of the covariance matrix (mentioned in \eq~\ref{eq:GP_der}) and provides a theoretical basis for the conditioning considerations discussed in Section \ref{sec:App_DT}. 

\vspace{1em} 
LEMMA 2.4.1. \textit {[Positive definiteness of the derivative-informed covariance matrix]\label{lem:B1}
Let $\mathbf{K}_{\text{der}}$ be the covariance matrix of a GP constructed from the function $f$ and the function $f$ derivatives $\nabla f, \dots, \nabla^d f$ up to order $d$, using a smooth, positive-definite kernel $k(\cdot, \cdot)$. Then $\mathbf{K}_{\text{der}}$ is symmetric and positive definite.}

\begin{proof}
See Appendix C.1.
\end{proof}

The incorporation of derivative observations is expected to improve the information available to the GP model. The following lemma formalizes this intuition by showing that the posterior uncertainty of the GP prediction cannot increase when higher-order derivative information is included. This result provides a theoretical basis for the improved predictive performance investigated in later parts of this work.

\vspace{1em}
LEMMA 2.4.2. \textit{[Reduction of posterior variance with derivative information]\label{lem:B2}
Let $\hat f^d$ denote the GP predictor constructed using derivative observations up to the order $d$. Then, the posterior variance at a test point, $x^*$, satisfies:
\[
\sigma_d^2(x^*)
\le
\sigma_{d-1}^2(x^*).
\]
Consequently,
\[
\mathbb E_{x^*}
\left[\sigma_d^2(x^*)\right]
\le
\mathbb E_{x^*}
\left[\sigma_{d-1}^2(x^*)\right].
\]
i.e., including higher-order derivatives can only reduce (or maintain) the posterior uncertainty of the GP prediction.}
\begin{proof}
See Appendix C.2.
\end{proof}

In addition to the availability of derivative information, the effectiveness of derivative-informed GP models also depends on how strongly this information influences predictions away from the observation locations. For the squared-exponential (SE) kernel, both the covariance terms and their derivative counterparts exhibit rapid spatial decay as the distance between points increases. The following lemma formalizes this property and provides the (necessary) theoretical insight into the locality of information propagation in derivative-informed GP models.

\vspace{1em}
LEMMA 2.4.3. \textit{[Exponential decay of the squared-exponential kernel derivatives]
\label{lem:2.3.3} Let $k:\mathbb{R}^p\times\mathbb{R}^p\to\mathbb{R}$ be the squared-exponential (SE) kernel defined as 
\[
    k(\mathbf{x}, \mathbf{y})=\sigma^2\exp\!\Big(-\frac{\|\mathbf{x}-\mathbf{y}\|^2}{2\delta^2}\Big),
\]
\noindent and let $d$ be the order of derivatives, then for any multi-indices  $\alpha,\beta$ with $|\alpha|,|\beta|\le d$, there exist constants $C_{\alpha,\beta}>0$ and $\gamma=\tfrac{1}{2\delta^2}>0$ such that for all $\mathbf{x},\mathbf{y}\in\mathbb{R}^p$
\[
\big| \partial_\mathbf{x}^\alpha \partial_\mathbf{y}^\beta k(\mathbf{x},\mathbf{y})\big| \le C_{\alpha,\beta}\, \exp\!\big(-\gamma\|\mathbf{x}-\mathbf{y}\|^2\big).
\]
Therefore, covariance entries and their derivative blocks decay exponentially with the squared inter-point distance.}

\begin{proof}
See Appendix C.3.
\end{proof}

In this work, the maximum number of derivatives is set to four, and the squared-exponential kernel was used to calculate the elements of the $\mathbf{K}_{[f,\nabla f, \nabla^2f,..., \nabla ^d f]}$ matrix.  

Under the noise-free regression assumption -- i.e., we observe the true function values without additive measurement noise -- the GP prior implies a joint Gaussian distribution over the training outputs $\mathbf{f}$ and the test outputs $\mathbf{f}^*$ at $M$ test points ($\mathbf{X}^*$):

\begin{equation}\label{eq:GP_Post}
\begin{bmatrix}\mathbf{f} \\ \mathbf{f}^*\end{bmatrix} = \mathcal{GP}\left (0,\begin{bmatrix} \mathbf{K}_{[f,f]} & \mathbf{K}_{[f,f^*]} \\\mathbf{K}_{[f^*,f]} & \mathbf{K}_{[f^*,f^*]}\end{bmatrix}\right),
\end{equation}

\noindent where $\mathbf{K}_{[f,f]}$ is the covariance matrix ($N \times N$) computed between the training points, $\mathbf{K}_{[f^*,f^*]}$ is the covariance matrix ($M \times M$) computed between the testing points and $\mathbf{K}_{[f,f^*]}$ (and its transpose, $\mathbf{K}_{[f^*,f]}$) are the cross-covariance matrices ($N \times M$) and ($M \times N$) respectively between the training and testing points. 

Conditioning on the observed training data yields the predictive posterior for the latent function at the test inputs:

\[
(\mathbf{f}^* \,\big|\, \mathbf{X}\, , \mathbf{f}\, , \mathbf{X}^*) \sim
\mathcal{N}\Bigl(\bar{\mathbf{f}}^*, \, \mathrm{cov}(\mathbf{f}^*)\Bigr),
\]

\noindent with

\[
\bar{\mathbf{f}}^* = \mathbf{K}_{f^*,f} \mathbf{K}_{f,f}^{-1} \mathbf{f}, 
\qquad
\mathrm{cov}(\mathbf{f}^*) = \mathbf{K}_{f^*,f^*} - \mathbf{K}_{f^*,f} \mathbf{K}_{f,f}^{-1} \mathbf{K}_{f,f^*}.
\]

Since we are in a noise-free setting, there is no observation noise variance, $\sigma_{n, i}^2$, added to the diagonal of $\mathbf{K}_{[f, f]}$, and the predictive mean $\bar{\mathbf{f}}^*$ exactly interpolates the training data at any training input $(\mathbf{x}^{(i)}), (\bar{\mathbf{f}}^*(\mathbf{x}^{(i)})=\mathbf{f}(\mathbf{x}^{(i)}))$. The predictive covariance collapses to zero at the training points, reflecting the certainty about the true function values at those locations.

We would like to remind the reader that the case of a noise-free derivative-enhanced GP is poorly conditioned: adding more data than the kernel's complexity makes the system numerically unstable, leading to challenges when attempting to invert it. Thus, noise (nugget) is added to make the system more stable, as is verified through numerical experiments in Section~\ref{sec:GPverification}.

This formulation makes GPs particularly appealing for deterministic function approximation (e.g., interpolating solutions of differential equations, modeling smooth physical phenomena without measurement noise), where the GP serves as a nonparametric interpolator with built-in uncertainty quantification away from the training data.

\subsection{Numerical Verification Experiments}
\label{sec:GPverification}
GPs augmented with derivative measurements up to $4^{th}$ order are verified on four numerical functions: 1D, 2D, 3D Griewank functions, and 3D Rosenbrock function. The primary reason for choosing the above functions is that they are challenging from an interpolation perspective while still maintaining smooth derivatives. Additionally, the dimensionality of the Griewank function can be increased without a significant change in the function form, which allows us to understand how the GP with derivatives scales with dimensionality. For all these functions, experiments are conducted by increasing the number of training points and the order of derivative information at each training point. The training dataset, $\mathcal{D}_{train}$, is generated by selecting equi-spaced points within the domain. The range of the input features is set at $[-\pi,\pi]$ for the Griewank functions (1D, 2D, and 3D) and $[-5,10]$ for the Rosenbrock function. Using the sampled $\mathbf{X}_{train}$ and $\mathbf{f}_{train}$, we compute the derivatives of $\mathbf{f}_{train}$ with respect to the input features up to the $4^{th}$ order analytically. The analytical expressions for the derivatives of the Griewank function can be found in \cite{Gri}. The trained GP is tested on a dataset that includes $1000$ randomly generated points, $\mathcal{D}_{test}=\{\mathbf{X}_{test},\mathbf{f}_{test}\}$, within the trained domain. During training, the kernel length scale, $\delta$, is optimized for lower prediction error on the test dataset. We note that $\mathbf{f}_{test}$ does not include any derivatives, and all the prediction errors are reported in mean squared error (MSE). The primary reason for reporting prediction errors in MSE rather than RMSE is that MSE shows differences in error magnitudes clearly, especially at smaller scales (e.g., for the order of $10^{-19}$). Since the goal of the plots is to compare the decrease in prediction error with the number of training points and derivative orders, the MSE provides a clearer separation of trends on the y-axis. In contrast, RMSE would compress these differences by taking a square root, making the patterns less visible. 

\begin{figure}[h!]
	\centering
	\includegraphics[width=1.\linewidth]{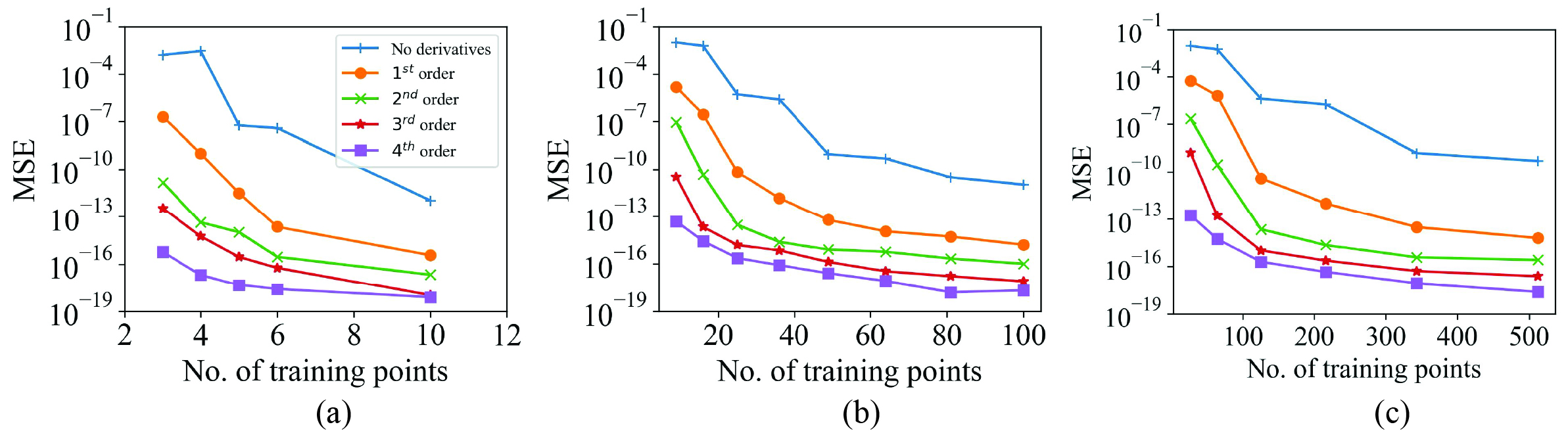}
	\caption{Results of numerical experiments for varying number of training points and derivative order (a) 1D Griewank function, (b) 2D Griewank function, and (c) 3D Griewank function.}
	\label{fig:E_GP_error}
\end{figure}

\begin{figure}[h!]
	\centering
	\includegraphics[width=0.5\linewidth]{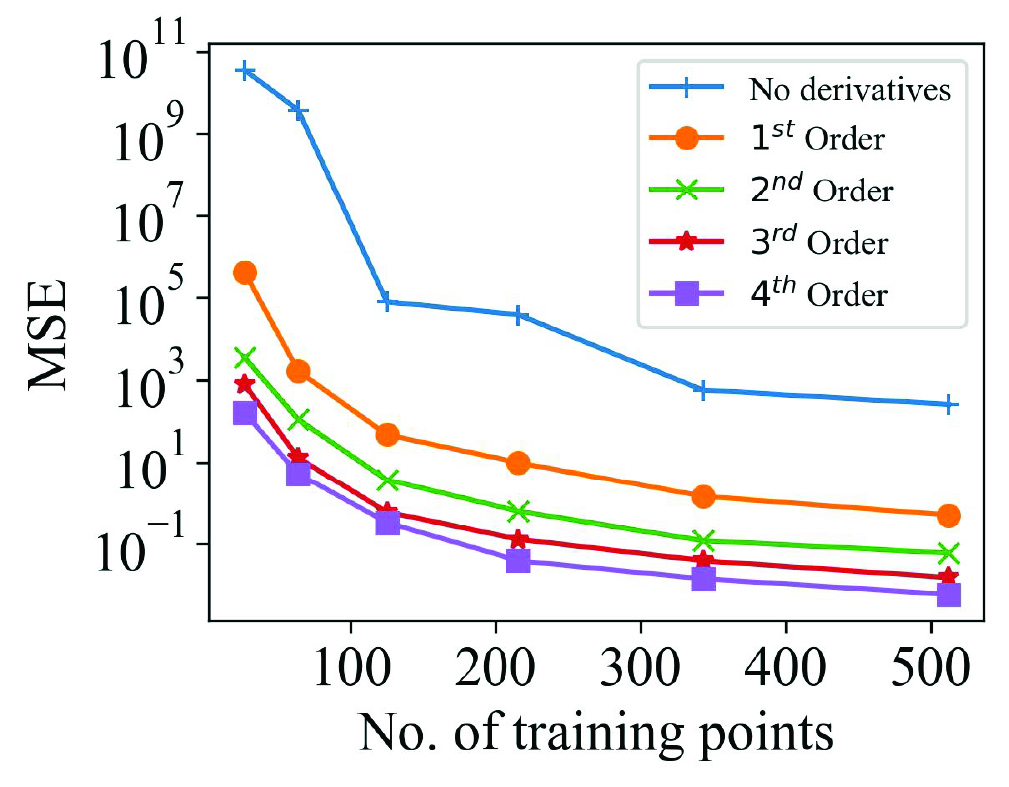}
	\caption{Results of numerical experiments on 3D Rosenbrock function for varying number of training points and order of derivatives. Note that the order of MSE errors is larger due to the steeper and higher magnitude nature of the Rosenbrock function. Although the magnitude of MSE errors is larger, the GP model performs equally well in relative terms when compared with the predictive performance of the Griewank function in \fig \ref{fig:E_GP_error}.}
	\label{fig:E_GP_error_Ro}
\end{figure}

Results of the experiments are shown in \fig \ref{fig:E_GP_error}. For all the studied functions, the prediction error reduced when the number of training points and the order of derivatives are increased. For the 1D Griewank function with three points, the prediction error reduced from approximately $10^{-3}$, without derivatives, to $10^{-15}$ when trained with $4^{th}$ order derivatives. As the number of training points increased, the prediction error reduced for GP with and without derivatives. However, the prediction error starts to plateau as the number of training points increases beyond six for the models trained with derivatives. This is due to the fact that the prediction error is already in the range of $10^{-17}$, and further increase in the number of training points would not improve the accuracy further. Similar observations can be noticed for 2D and 3D Griewank functions.  For the 3D Griewank function trained with 27 points, increasing order of derivatives reduced the prediction error from $10^{-2}$ (for no derivatives) to $10^{-13}$ (for $4^{th}$ order). For the Rosenbrock function with 27 training points in \fig \ref{fig:E_GP_error_Ro}, including derivative information reduced the prediction error from $10^{10}$ to less than $10^{2}$ when $4^{th}$ order derivatives are included in training. A similar observation can be made for higher $N$ values.

\begin{figure}[h!]
	\centering
	\includegraphics[width=1.0\linewidth]{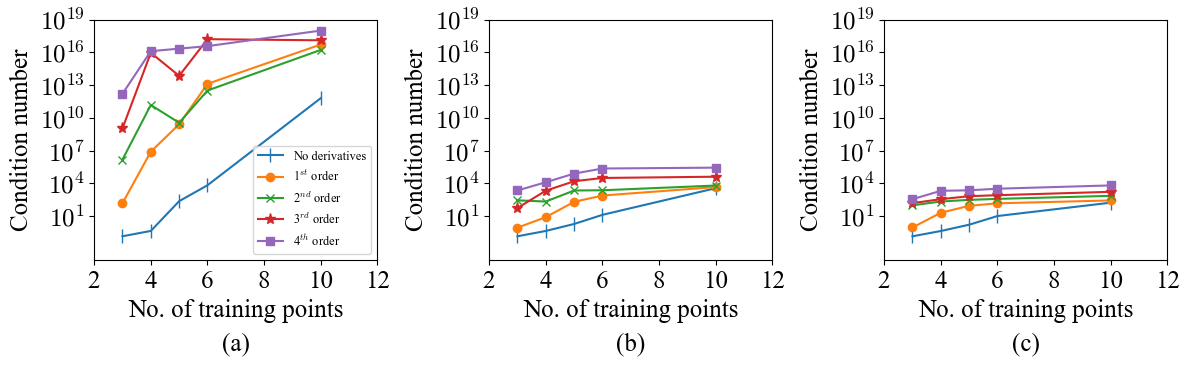}
	\caption{Results of numerical experiments for matrix conditioning on 1D Griewank function for varying number of training points, order of derivatives, and varying noise. (a) No noise, (b) Increasing noise, and (c) Decreasing noise.}
	\label{fig:1d_noisy_condition}
\end{figure}

Furthermore, to investigate the effect of adding noise to the GPs augmented with derivative measurements on numerical stability and predictive performance, we perform additional experiments. We examine the behavior of the matrix condition number (in $L_2$) of the noisy covariance matrix, $\mathbf{\Sigma} = \mathbf{K} + \mathbf{R}$, under varying nugget parameters, $\sigma^{2}$.

Figure \ref{fig:1d_noisy_condition} shows the results of the numerical experiments for the effect of varying noise levels on the condition numbers. In the case of increasing noise, an incremental 1\% increase in noise levels for each derivative order is added from 1\% noise for $0^{th}$ order derivative to 5\% noise for $4^{th}$ order derivative. 

Conversely, for the case of decreasing noise, we added 5\% noise for $0^{\text{th}}$ order derivatives, 4\% noise for $1^{\text{st}}$ order derivatives, 3\% noise for $2^{\text{nd}}$ order derivatives, 2\% noise for $3^{\text{rd}}$ order derivatives, and 1\% noise for $4^{\text{th}}$ order derivatives, i.e., there is a 1\% decrease in noise levels with an increase in the derivatives order.

In the case of decreasing noise, the covariance matrices exhibit small condition numbers, leading to higher predictive errors as shown in Figure \ref{fig:1d_noisy_mse}. In contrast, for increasing noise, the covariance matrices exhibit large condition numbers, resulting in lower predictive errors. This happens because the nugget and the scaling of the derivatives dominate the predictive error. In decreasing noise, the small nugget results in insufficient regularization of higher-order derivatives. However, in the case of increasing noise, the larger nugget for higher-order derivatives stabilizes the covariance matrices, which reduces the predictive error despite a higher condition number. Prior studies, \cite{chen_sparse_2024, schafer_sparse_2021,chen_solving_2021,kang_correlation-based_2023}, laid the theoretical foundations for the relationship between numerical stability and noise levels. We would like to remind the reader that \cite{chen_sparse_2024, schafer_sparse_2021,chen_solving_2021,kang_correlation-based_2023} provide the theoretical foundations of nugget adaptivity.  The results reported herein are consistent with those works. The MSE begins to plateau in Figure \ref{fig:1d_noisy_mse} because the statistical component of noise dominates the overall error, and further regularization does not yield any improvement in accuracy. It should be noted that, unless otherwise stated, no artificial noise is added to the function values or derivative observations used in the numerical experiments presented in this work. The noise terms introduced in this section (and throughout this paper) correspond only to the assumed nugget terms incorporated into the covariance matrix to improve numerical conditioning and stabilizing the matrix inversion.

\begin{figure}[h!]
	\centering
	\includegraphics[width=1.0\linewidth]{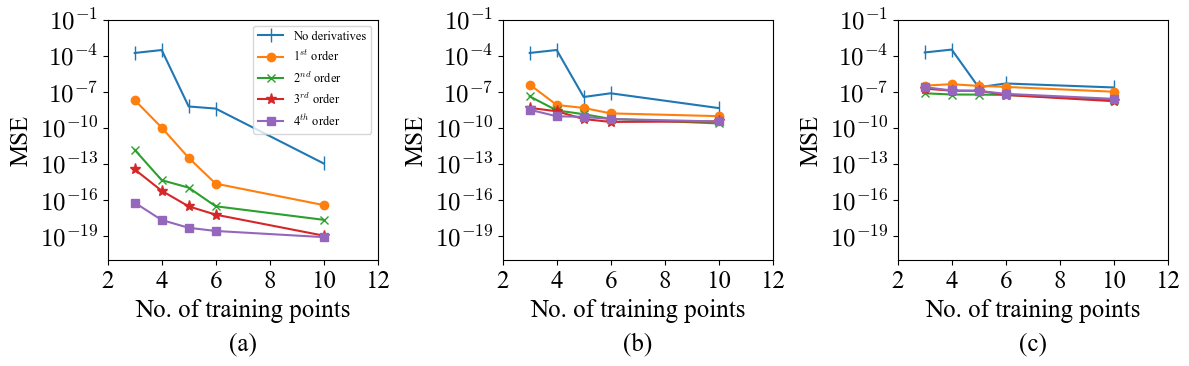}
	\caption{Results of numerical experiments for the predictive performance (measured by MSE) on 1D Griewank function for varying number of training points, order of derivatives, and varying noise. (a) No noise, (b) Increasing noise, and (c) Decreasing noise.}
	\label{fig:1d_noisy_mse}
\end{figure}


%
\section{Sparse Cholesky for Derivative-Enhanced GP}
\label{sec:DT_algorithms}

In Section \ref{sec:math}, we presented both traditional GP modeling and our enhancement using derivative information. As previously discussed, the incorporation of derivatives improves the accuracy of our GP approximation, given sufficient smoothness of the underlying function we are approximating. However, derivative-enhanced GPs come at a cost: the inclusion of derivative terms leads both to a significant increase in the size of the covariance matrix and an increase in its conditioning as a function of the number of points $N$, dimension $p$, and derivative order $d$. 

To quantify the growth in computational complexity associated with derivative-informed GP models, it is useful to determine how the number of covariance terms scales with the problem dimension and derivative order. The following lemma establishes this scaling relationship and motivates the sparsification methods discussed later in this work. The number of derivative terms, $N_d$, to be incorporated in the covariance matrix is given by the following.

LEMMA 3.0.1. \textit{[Covariance matrix size and sparsity]\label{lem:B4}
Given $N$ training points drawn from a $p$-dimensional space and with derivatives up to the order $d$, the number of covariance terms is given by the following:}
\begin{equation}
N_d = \binom{p+d-1}{d} N,
\label{eq:4}
\end{equation}
\textit{
\noindent where the resulting (updated) covariance matrix ${\bf K}_{\text{der}} \in \mathbb{R}^{N_d \times N_d}$ is symmetric positive-definite (s.p.d.) with a block structure corresponding to the derivatives.}

\begin{proof}
See Appendix C.4.
\end{proof}

While derivative observations improve the approximation properties of the GP model, they also introduce significant numerical challenges. In particular, derivative-informed covariance matrices exhibit deteriorating conditioning as the number of derivative terms increases, leading to instability in matrix inversion and a degradation of numerical accuracy in GP inference. This effect becomes more pronounced as a function of derivative order $d$ and the kernel length scale $l$, both of which influence the spectral properties of the covariance matrix. To characterize these numerical effects, it is important to quantify the role of model parameters in conditioning the derivative-informed covariance matrix. The following lemma provides a formal definition of the condition number for the derivative-augmented covariance matrix and highlights its dependence on the derivative order and the kernel length scale. This lemma will later be referenced in this work to motivate the need for nugget regularization and improved numerical stabilization.

LEMMA 3.0.2. \textit{[Conditioning of derivative-informed covariance matrix]\label{lem:B3}
For a derivative-informed covariance matrix ${\bf K}_{\text{der}}$ with order $d$ derivatives and kernel length scale $l$, the condition number is given by
\[
\kappa({\bf K}_{\text{der}}) = \|{\bf K}_{\text{der}}\|_2 \cdot \|{\bf K}_{\text{der}}^{-1}\|_2,
\]
which increases with $d$ and decreases with increasing $l$. Given that the covariance matrix is s.p.d., this expression can be rewritten as:
\[
\kappa({\bf K}_{\text{der}}) = \frac{\lambda_{max}}{\lambda_{min}} ,
\]
where $\lambda_{max}$ and $\lambda_{min}$ denote the maximum and minimum eigenvalues, respectively, of the covariance matrix.
}
\begin{proof}
See Appendix C.5.
\end{proof}

Assuming all the derivative terms are available, computation of ${\bf K}_{\text{der}}^{-1}$ while solving the exact GP scales as $O((N \times N_d)^3)$ in time and as $O((N\times N_d)^2)$ in memory. (We use `exact' to denote the covariance matrix prior to any sparsification approximations as presented below and assuming Gaussian elimination). Clearly, efficient computation of the action of ${\bf K}_{\text{der}}^{-1}$ is desired. As discussed earlier, the size of the covariance matrix with derivatives increases significantly with respect to $N$, $d$, and $p$.  Solving for such large matrix systems can be computationally intractable \cite{liu_when_2020}, and hence there is a need to efficiently approximate the covariance matrix to help alleviate these scalability issues. Solving these large matrix systems can be approximated through several methods \cite{heaton_case_2019}, such as low rank approximation methods \cite{musco_recursive_2017,chen_randomly_2025}, sparse approximation \cite{vecchia_estimation_1988,furrer_covariance_2006,datta_hierarchical_2016,datta_nonseparable_2016,smola_sparse,cao_variational_2023}, and others \cite{rahimi_random_nodate,quinonero-candela05a}. Out of the several studied methods, work related to the Vecchia approximation has gained popularity in the world of GPs applied to computational and data science problems \cite{vecchia_estimation_1988,stein_approximating_2004,datta_hierarchical_2016, sun_statistically_2016,katzfuss_scaled_2021,huan_sparse_2025}. Within this body of work, one of the critical aspects upon which these approximation methods are built is the choice of the ordering and conditioning set \cite{guinness_permutation_2018, kang_correlation-based_2023, katzfuss_general_2021,schafer_sparse_2024}. In this work, we utilize sparse Cholesky factorization of the exact GP \cite{schafer_sparse_2021,chen_solving_2021}, which has proven to be an efficient approximation method with near-linear computational complexity. \cite{chen_sparse_2024, schafer_sparse_2021,chen_solving_2021,kang_correlation-based_2023} have previously made attempts to incorporate derivatives into the sparse Cholesky formulations. However, there are some differences between our approach of extending sparse Cholesky factorization to include derivative observations of arbitrary order $d$. \cite{chen_sparse_2024, schafer_sparse_2021,chen_solving_2021,kang_correlation-based_2023} formulate the derivatives as constraints in the Hilbert space, and they define the derivatives as embedded inside the functionals, while describing the derivatives as only specific combinations (such as the Laplacian operator), which is not necessary in our formulation. Additionally, our formulations describe the latent function, $\mathbf{f}$, to be random and describe derivatives as random processes (which are encoded within the covariance matrix). We would like to acknowledge that we did not entirely introduce the concept of derivative-informed GPs. However, our contribution is a DT-specific engineering system (see Section~\ref{sec:App_DT}), which is built by aligning all the relevant mathematical building blocks. In this section, we provide a brief overview of the sparse Cholesky factorization algorithm \cite{schafer_sparse_2021} and how we extend it to include derivative observations. We then show verification results using our updated approach -- highlighting the impact of, and interplay between, the various parameters of the algorithm such as the sparsification factor, number of training points, etc.


\subsection{Sparse Cholesky Factorization of the GP's Precision Matrix}
\label{sec:Sparse_GP_review}

The objective of the sparse Cholesky algorithm is to build a sparse Cholesky factor of the precision matrix. The precision matrix, in the context of computational methods, refers to the inverse of the covariance matrix obtained as part of our GP model. In this section, we provide a brief review of the steps involved in building a sparse Cholesky factor of the precision matrix of the GP. For additional details, readers are directed to the original work upon which our work is based \cite{schafer_sparse_2021,chen_solving_2021}. 

The first step in building a sparse Cholesky factor of the precision matrix involves ordering the set of training points, $\{\mathbf{x}^{(i)}, i \, \epsilon \, I\}_{i=1}^N$, using maximum-minimum distance (MMD) ordering \cite{schafer_sparse_2021,chen_solving_2021,kang_correlation-based_2023}. In MMD ordering, the sequence of the points is chosen based on the maximum of the minimum distances from the set of unordered points, and the index of the ordered points is stored as a vector $\mathbf{P}$, as shown below:
\begin{equation}\label{eq:P_mmd}
    \mathbf{P}(q+1)=arg \ max_{i \, \epsilon \,I \setminus \{1,..,q\} } \, dist (\mathbf{x}^{(i)},\{\mathbf{x}^{(1)},...,\mathbf{x}^{(q)}\})
\end{equation} 
\noindent The length scale of the ordered points is stored in a vector $\mathbf{l}$, given as:
\begin{equation}\label{eq:l_mmd}
    \mathbf{l}^{(i)}=dist(\mathbf{x}^{\mathbf{P}(i)},\{\mathbf{x}^{\mathbf{P}(1)},...,\mathbf{x}^{\mathbf{P}(i-1)}\}.
\end{equation}

\noindent The intuition behind this step is that the ordering is performed by selecting the points furthest from the previous point; thus, the reordered sequence contains points that are ``not too close to each other." The aforementioned ordering has been shown to produce Cholesky factors with near sparsity, the proof of which can be found in \cite{schafer_sparse_2021,chen_solving_2021}. The primary reason to use a distance-based metric for ordering (such as MMD ordering), instead of a correlation-based metric is to exploit the spectral decay (screening effect) of the kernel as was suggested in the work by \cite{chen_sparse_2024, schafer_sparse_2021,chen_solving_2021,kang_correlation-based_2023}.

Once the ordering is determined, the sparsity set $(S)$ is determined by:
\begin{equation}
    S_{\mathbf{P},\mathbf{l},\rho}=\{(i,j)\, \rotatebox[origin=c]{270}{$\cup$} \, I \times I: i \leq j, \,dist(\mathbf{x}^{\mathbf{P}(i)},\mathbf{x}^{\mathbf{P}(j)}) \leq \rho \mathbf{l}^{(j)} \},
\end{equation}
\noindent where $\rho$ influences the size of the sparsity set. The sparsity set we obtain can be aggregated into groups based on both the ordering and geometric location, denoted by $S_{\mathbf{P},\mathbf{l},\rho,\lambda}$, where $\lambda$ is set as $1.5$ as suggested by \cite{schafer_sparse_2021}. These aggregated groups are termed supernodes, $\mathcal{SN}$. Each $\mathcal{SN}$ consists of a list of parent and child indices, where the term parent refers to the index of the columns in the matrix and the term child denotes the indices of non-zero entries in the column. The concept of supernodes is particularly useful in this work, as they allow us to reuse the computed Cholesky factors for a set of rows and columns within the matrix. This offers a significant computational advantage in our digital twin framework, which will be discussed in detail later.

Based on the determined ordering and sparsity pattern (with aggregation), the sparse matrix is obtained by KL minimization, given by the following expression:

\begin{equation} 
\mathbf{U} = \operatorname*{arg\,min}_{\mathbf{\hat U} \in S_{\mathbf{P},\mathbf{l},\rho,\lambda}}\mathcal{D}_{KL}\left( \mathcal{N}(0,\mathbf{K}) \middle\| \mathcal{N}\left(0,(\mathbf{\hat U}\mathbf{\hat U}^T)^{-1}\right)\right).
\end{equation}

The above equation has a closed-form solution which can be found in \cite{schafer_sparse_2021,chen_solving_2021}. In this work, we follow the steps mentioned in this section to generate a sparse approximation, ${\bf U}$, of the precision matrix. The algorithm is implemented in Python (\url{https://github.com/Shridhar0305/derivative_enhanced_GPs_DTs}), and the code results are compared with the original work's results by \cite{chen_solving_2021}. We primarily focus on noise-free data in this work (and leave the case of noisy data for future work).

\subsection{Derivative-Enhanced Sparse Cholesky GP for Noise-Free Data}
\label{sec:Sparse_GP_deriv}

In this section, the sparse Cholesky GP algorithm presented in Section \ref{sec:Sparse_GP_review} is extended to include derivative measurements up to an arbitrary derivative order, $d$. As discussed in Section \ref{sec:math}, incorporating derivatives does not violate the symmetric positive-definite property of the covariance matrix; thus, extending the idea of sparse Cholesky GP with Cholesky factorization to incorporate derivatives is valid. A critical point to be addressed when incorporating the derivative observations into the sparse Cholesky GP approximation is how the derivatives are incorporated into the ordering, $\mathbf{P}$, and its corresponding sparsity set. In an exact GP scenario, the way derivatives are placed in the matrix formation does not have any effect on the mathematical accuracy, as matrix ordering does not change the spectrum of the operator; however, that is not the case for sparse Cholesky GP. The following sections address how derivatives are included in building the sparse Cholesky GP. The DT application of derivative-enhanced sparse Cholesky GP for noise-free data is detailed in Section \ref{sec:App_DT}.

\subsubsection{Ordering Derivatives and Supernodes with Derivatives}
\label{sec:Supernodes_with_derivatives}
Given a set of functional and derivative measurements, $\mathbf{F}$, up to an arbitrary derivative order, $d$, at all training points, we studied four different methods of ordering with derivatives. In this section, we discuss only one method in detail, which we call ``point-wise ordering algorithm 1". The readers are referred to Appendix A.1 for the rest of the three methods, namely ``point-wise ordering algorithm 2", ``measurement-wise ordering algorithm 1", and ``measurement-wise ordering algorithm 2". In Appendix A.1, we provide a detailed comparison between all four methods. 

In ``point-wise ordering algorithm 1", the derivative measurements are grouped with the points in an array-of-structures format. For better understanding, we illustrate the structure of the covariance matrix that has the function values $f(x)$ and its corresponding first-order derivative measurement $\nabla f(x)$ in \fig~\ref{fig:K_withder}. Assuming $N$=10, the plot shows the structure of the $\mathbf{K}_{der}$ matrix when $f$ and $\nabla f$ are grouped by points. Here, $\mathbf{F}$ is ordered as the following:
\[
[f^{(\mathbf{P}(1))},\nabla f^{(\mathbf{P}(1))},f^{(\mathbf{P}(2))},\nabla f^{(\mathbf{P}(2))}, ....,f^{(\mathbf{P}(N))},\nabla f^{(\mathbf{P}(N))}].
\]

\begin{figure}[h!]
    \centering
    \begin{minipage}{0.35\linewidth}
        \centering
        \includegraphics[width=\linewidth]{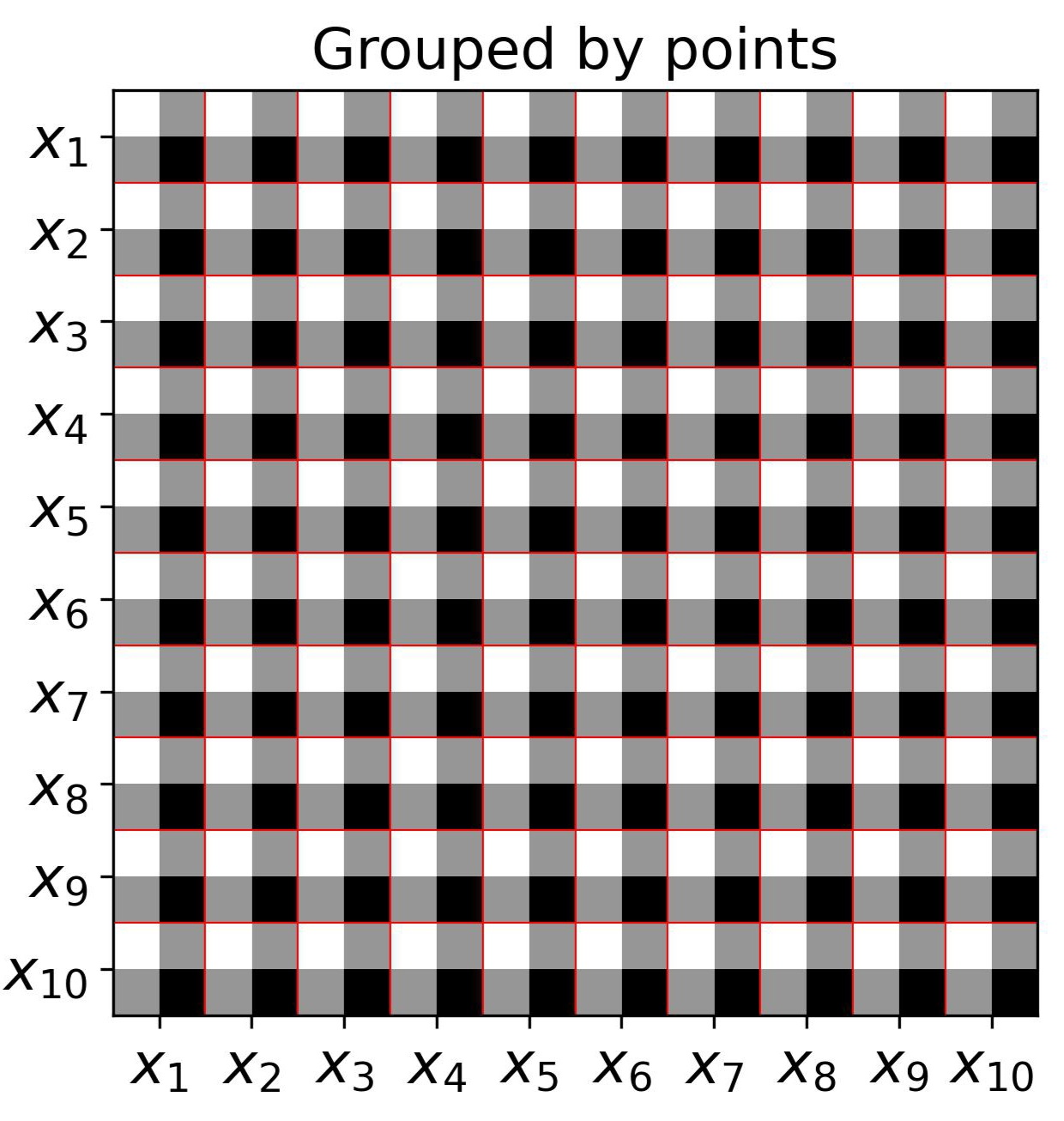}
    \end{minipage}
    \hspace{2mm} 
    \begin{minipage}{0.15\linewidth}
        \centering
        \includegraphics[width=\linewidth]{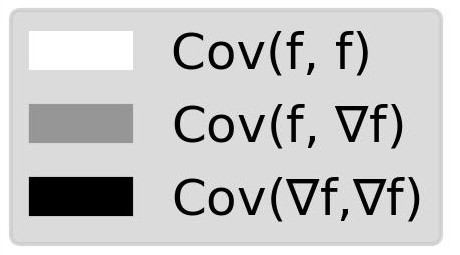}
    \end{minipage}
    \caption{The figure illustrates the point-wise ordering algorithm 1 of incorporating derivative measurements in the formation of the $\mathbf{K}_{der}$ matrix. The derivative measurements are grouped by points as they are placed next to the functional observations. Note that these functional observations are ordered according to $\mathbf{P}$.}
    \label{fig:K_withder}
\end{figure}

The initial ordering $\mathbf{P}$ was obtained using Eq. \ref{eq:P_mmd} without any derivative measurements, and then we extend  $\mathbf{P}$ to incorporate derivative measurements to obtain $\mathbf{P}^d$. A subscript $po-1$ is added to $\mathbf{P}^d$ to refer to the ordering grouped by the point-wise ordering algorithm 1. The algorithm used to obtain $\mathbf{P}^d_{po-1}$ is shown in Algorithm \ref{alg:p_point_1}. The $\mathbf{P}^d_{po-1}$ is obtained by iterating through $\mathbf{P}$ and adding the index of each derivative measurement immediately after the point measurement in the sequence. 

\begin{center}
\begin{minipage}{0.46\textwidth}
\begin{algorithm}[H]
	\caption{Constructing the $\mathbf{P}^d_{po-1}$ array}
	\label{alg:p_point_1}
	\begin{algorithmic}[1]
            \State \textbf{Input:} $\mathbf{P}$ \text{from MMD ordering} 
            \State \textbf{Output:} $\mathbf{P}^d_{po-1}$
            \Statex
		\State $td \gets \lfloor N_d / N \rfloor$
		\For{$i \gets 1$ to $N$}
		\For{$j \gets 1$ to $td$}
		\State $k \gets i \cdot td + j$
		\State $\mathbf{P}^d_{po-1}[k] \gets \mathbf{P}[i] + N \cdot j$
		\EndFor
		\EndFor
	\end{algorithmic}
\end{algorithm}
\end{minipage}
\end{center}

The supernodes, $\mathcal{SN}$, are originally obtained without any derivative measurement, through the procedure described in Section \ref{sec:Sparse_GP_review}. The existing supernodes, $\mathcal{SN}$, are then updated to include the derivative measurements to obtain $\mathcal{SN}^d_{po-1}$. In $\mathcal{SN}$, the set of parents and children consists of the index of the elements in $\mathbf{P}$. To obtain $\mathcal{SN}^d_{po-1}$, indices in each set of both parent and child are expanded to include the derivative measurements of the corresponding indices. Note that $\mathcal{SN}^d_{po-1}$ is a list of multiple supernodes that are used to build the sparse matrix. 

\subsection{Numerical Verification Experiments: Derivative-Enhanced Sparse Cholesky GP with Noise-Free Data}

We verify our derivative-enhanced sparse Cholesky GP algorithm on the 1D, 2D, and 3D Griewank functions for varying numbers of training points and orders of derivatives. We assume no noise is present in the training data (either in the function values or their derivatives). Additionally, we also study the accuracy of the sparse Cholesky GP with different $\rho$ values to understand the influence of the sparsity of the matrix on the prediction accuracy. 

We present in Figure~\ref{fig:S_GP_error} the results of our experiments for the Griewank functions with the sparsification factor set to $\rho =10$. The training and testing data are the same as those used for training the exact GP reported in Section \ref{sec:math}. For the reported hyperparameters, prediction from the sparse Cholesky GP is equal to the exact GP (\fig~\ref{fig:E_GP_error}), and the observations reported from the results of the exact GP apply to the sparse Cholesky GP as well. As the number of training points and the order of the derivative increase, the prediction error reduces significantly. Note that the grouping method has no appreciable effect on the prediction error, as the sparsity of the matrix is relatively small for the reported hyperparameters. Additionally, the higher dimensionality of the input function leads to a slightly higher predictive error, as shown in Figure~\ref{fig:S_GP_error}.

\begin{figure}[H]
	\centering
	\includegraphics[width=1.\linewidth]{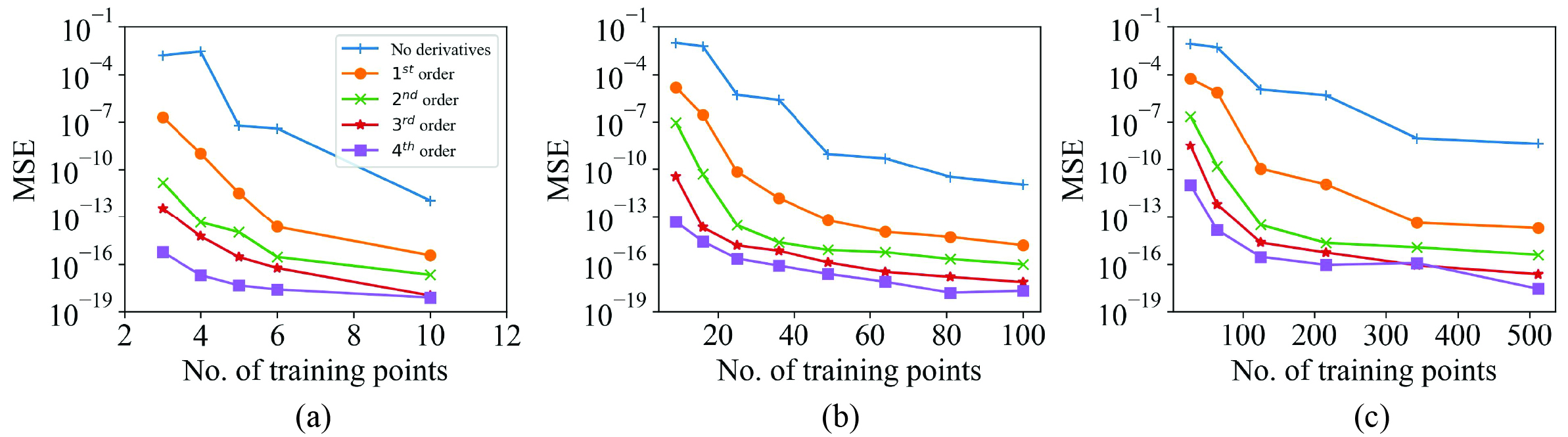}
	\caption{Prediction error from sparse Cholesky GP for 1D (a), 2D (b), and 3D (c) Griewank functions}
	\label{fig:S_GP_error}
\end{figure}

We know that the ordering of the derivatives within the matrix plays no role in the formal accuracy of the method due to spectral equivalency under permutations; however, this property does not hold when the matrix is made sparse. We perform additional experiments by varying $\rho$ values to compare the predictive performance for the point-wise ordering algorithm 1 and the measurement-wise ordering algorithm 1. Similarly, we perform experiments by varying the number of training points, $N$, to compare the predictive performance of the point-wise ordering algorithm 1 and the measurement-wise ordering algorithm 1. These additional results are reported in Appendix A.2.

From the experiments given in Appendix A.1, it is evident that the point-wise ordering algorithm 1 and the measurement-wise ordering algorithm 2 show better predictive performance compared to the other proposed methods. However, in the case of sparse Cholesky GP for our DT application, point-wise ordering algorithm 1 offers a computational advantage, as it allows adding additional functional and derivative observations by including additional columns in the matrix. Thus, point-wise ordering algorithm 1 is chosen as the main method in this work. The details of our dynamic algorithm are presented in Section \ref{sec:dynamic_GP_for_DT}.

%
\section{Dynamic Sparse Cholesky GP For Use in Digital Twin Application}
\label{sec:dynamic_GP_for_DT}

In a DT system, surrogate models, like GPs, are employed to predict the target property of interest for the given state of the physical system. These surrogates are often trained using an initial set of data representing the physical system and then deployed in service as a DT. The growth in the number of derivative-enhanced observations follows Lemma 3.0.1, which shows that the covariance matrix size increases combinatorially with dimension $p$ and derivative order $d$. This motivates the need for sparse and scalable methodologies. One of the critical aspects of DT is that the surrogate model should have the ability to be updated in a smooth fashion to accommodate any changes in the physical system. In other words, the surrogates need to be dynamically updated without the need for extensive re-training. Additionally, the state of the physical system can change significantly during the operation, and the trained surrogate may not perform well enough even with regular dynamic updates. Inevitably, the surrogate model may need to be retrained. The DT system should have the ability to detect when to trigger re-training. Considering the above-discussed requirements for the surrogates in mind, we propose a dynamic sparse Cholesky GP algorithm, which has the ability to be dynamically updated or retrained when new information is available. 

\subsection{Dynamic Derivative-Enhanced Sparse Cholesky GP} 
\label{sec:dynamic_S_GP}

As discussed in Section~\ref{sec:DT_algorithms}, a major consideration when implementing GPs for DT applications is the choice of linear solver used for ``inverting'' the kernel matrix system.  Earlier we introduced the idea of using Cholesky as an appropriate solver method given the symmetric positive definite properties of the kernel matrix, even when additional derivative information is added (see Lemma 2.4.1 in Sec~\ref{sec:math}). In addition to matrix properties, variants on Cholesky has been used in several other areas \cite{Polok2013IncrementalBC, Hecht_Updated_sparse_cholesky} where dynamic updates of GP systems were required. In addition to the traditional benefits of Cholesky, the sparse Cholesky algorithm given in Section~\ref{sec:DT_algorithms} yields additional benefits because of the connection it draws between groupings of points in space and collections of columns within the kernel matrix. This is where grouping by points offers a computational advantage because when new data with derivative information is available, new data can be included in the matrix by adding an additional column without re-computing the whole matrix. However, the major questions to be addressed are where the new column should be placed in the matrix and the sparsity of the new column. In order to address these questions, we rely on the idea of a dynamic supernode,  which allows us only to re-evaluate the Cholesky factors of a small number of columns when building the matrix $\mathbf{U}$, thus eliminating the need for re-evaluating the complete sparse matrix. 

Note that the dynamic supernode in this work should not be confused with the dynamic supernode concept used in matrix update and downdate \cite{davis_dynamic_2009}. Additionally, incremental factorization methods have been extensively studied in other domains, such as \cite{Polok2013IncrementalBC, Hecht_Updated_sparse_cholesky}. These methods share the common objective of avoiding complete re-factorization when new information becomes available. Our work is related to this computational philosophy; however, our focus is on developing a dynamic update framework for derivative-enhanced GP models within a DT setting (Section \ref{sec:App_DT}). In this section, we discuss two different approaches to generating and updating our dynamic supernodes.
For the sake of simplicity, the approaches below are described for derivative-free measurements. However, they can be extended to include any arbitrary order of derivatives. We primarily describe the approaches below for noise-free measurements. Additionally, dynamic sparse Cholesky GP is only applied for derivatives grouped by points, as it reduces the number of supernodes that need to be re-evaluated. Since the dynamic sparse Cholesky GP methods discussed in this section will be applied for the DT application problem in Section \ref{sec:App_DT}, incoming data will be added to the existing data, which will allow the DT to get better aligned with the PT, i.e., discarding previous acquired data may lead to the loss of information about the underlying system behavior. Lastly, we would like to remind the reader that although the squared-exponential kernel shows rapid spectral decay (which makes it suitable for low-rank approximations), our focus in this work (especially in Section~\ref{sec:dynamic_GP_for_DT}) is on a scalable sparse formulation that leads to streaming updates in an effective manner. The sparse Cholesky method provides us with a natural way for localized streaming updates through dynamic supernodes, which is not possible in standard low-rank GP formulations.

\subsubsection{Supernode Update: Approach 1 (SU-Approach1)} 

We begin our discussion of the supernodes updating by assuming that we start with an initial set of points from which we have generated our kernel matrix and our initial permutation $\mathbf{P}$. Initially, we obtain $\mathbf{P}$ from the MMD ordering mentioned in Sections \ref{sec:Sparse_GP_review} and \ref{sec:Sparse_GP_deriv}. We generate supernodes, $\mathcal{SN}$, from the initial set of training points, $\mathcal{D}_{train}$. The supernodes, $\mathcal{SN}$, are generated by first constructing the sparsity pattern associated with the reordered kernel matrix according to the permutation $\mathbf{P}$. During the MMD ordering, each point has a length scale, $\mathbf{l}^{(i)}$, associated with it (\eq~\ref{eq:l_mmd}), which is used along with $\rho$ to form supernodes (see Section \ref{sec:Sparse_GP_review}). Specifically, for each reordered point, $i$, the neighbor set $\mathcal{N}(i)$ is defined as
\[
\mathcal{N}(i) 
= \left\{\, j < i \;:\; 
\| \mathbf{x}_{P(i)} - \mathbf{x}_{P(j)} \|
\leq \rho \, \mathbf{l}^{(i)} 
\right\},
\]
where $\rho$ is the sparsification parameter and $\mathbf{l}^{(i)}$ is the insertion radius associated with the $i^{\text{th}}$ pivot in the MMD ordering. This neighbor structure defines the nonzero pattern of the sparse Cholesky factor. Supernodes are then constructed by grouping consecutive columns that share identical sparsity structures below the diagonal. Precisely, a supernode, $\mathcal{S}_k \in \mathcal{SN}$, is defined as a maximal contiguous set of column indices such that 
\[
\mathrm{pattern}(i) \setminus \{i\}
=
\mathrm{pattern}(i+1) \setminus \{i+1\},
\]
for all consecutive indices $i$ within the supernode. 

In order to do the dynamic update, we split $\mathbf{P}$ into two arrays, $\mathbf{P}_{\text{fix}}$ and $\mathbf{P}_{\text{dyn}}$. The set $\mathbf{P}_{dyn}$ is obtained by taking $20\%$ of elements from the tail end of $\mathbf{P}$. In other words, the size, $M$, of $\mathbf{P}_{dyn}$ is about $20\%$ of $N$. After obtaining the fixed and dynamic sets, the supernodes with the parent set from $\mathbf{P}_{fix}$ are considered fixed supernodes, $\mathcal{SN}_{fix}$. Alternatively, the supernodes with parents from $\mathbf{P}_{dyn}$ are considered dynamic. In this approach, the parents and children of all the dynamic supernodes are merged together to form a single dynamic supernode, $\mathcal{SN}_{dyn}$. Doing so, we eliminate the need to pick the supernodes to which the new point will be added. 

In our setting, a parent node refers to the pivot element (precisely, a column index) assigned to each training point under the MMD ordering. After applying the MMD ordering, each point, $\mathbf{x}_{P(i)}$, is associated with a position $i$ in the elimination order. The parent of the node $i$ is the ``identifier" used in the elimination tree induced by the sparse Cholesky factorization of the reordered matrix. Therefore, the parent node gives the ordering-dependent representative that determines the way in which fill-in propagates in the elimination tree. Therefore, the parent node is the index in the elimination tree (induced by $\mathbf{P}$) that governs the sparsity propagation and the supernode(s) grouping. As described above,  the fixed/dynamic split is defined by partitioning $\mathbf{P}$. Hence, the ``parent set" (of nodes) of a supernode corresponds to the subset of $\mathbf{P}$ that defines its elimination ordering (and, as a consequence, its dependency structure in the Cholesky factorization).

Once the fixed and dynamic sets are determined, the fixed supernodes, $\mathcal{SN}_{fix}$, would not be disturbed when an additional data point is available; only $\mathcal{SN}_{dyn}$ is re-evaluated. $\mathcal{SN}_{\mathrm{dyn}}$ is re-evaluated by updating only the local sparsity structure associated with the new data points and the previously identified dynamic region (of data points). Let $\mathbf{x}_{\mathrm{new}}$ denote a new data point. We first determine its interaction set with respect to the existing training configuration by constructing its neighbor set using the same sparsification rule as in the initial factorization,
\[
\mathcal{N}(\mathrm{new}) 
= \left\{\, j \in \mathcal{D}_{\mathrm{train}} \;:\;
\| \mathbf{x}_{\mathrm{new}} - \mathbf{x}_j \|
\leq \rho \, \mathbf{l}^{(j)} \right\}
\cup \mathrm{kNN}(\mathbf{x}_{\mathrm{new}}).
\]

The indices corresponding to the affected dynamic supernodes are identified via the union of supernodes associated with  $\mathbf{P}_{\mathrm{dyn}}$, and supernodes whose column indices are reachable from the neighbors of $\mathbf{x}_{\mathrm{new}}$ under the current elimination structure. This defines a restricted index set $\mathcal{R}$, which determines the portion of the factorization which must be updated. 

The sparsity pattern of the updated system is updated locally by appending $\mathbf{x}_{\mathrm{new}}$ and recomputing only the affected column patterns:
\[
\mathrm{pattern}(\mathbf{x}_{\mathrm{new}}) 
= \mathcal{N}(\mathrm{new}),
\]
and propagating fill-in updates to all ancestors of nodes in $\mathcal{R}$ according to the elimination tree structure induced by $\mathbf{P}$. As mentioned earlier, no modification is made to $\mathcal{SN}_{\mathrm{fix}}$.

Finally, the dynamic supernodes are reconstructed by performing a local re-grouping within $\mathcal{SN}_{\mathrm{dyn}}$. Columns whose updated sparsity patterns remain identical after insertion are merged into existing supernodes, while columns whose patterns differ from their predecessors are split into new supernodes. Illustration of the above-discussed steps of creating fixed and dynamic dataset supernodes is shown in \fig~\ref{fig:dynamic_update}. On the left, we can see the $N \times N$ sparse matrix obtained using the initial ordering $\mathbf{P}$ and $\mathcal{SN}$. The gray and orange columns in the middle figure show columns from the fixed and dynamic supernodes, respectively. Finally, all the supernodes with dynamic parents are merged to form $\mathcal{SN}_{dyn}$, shown as orange in the right-most figure in  \fig~\ref{fig:dynamic_update}. If the derivative information is available as part of training dataset, $\mathcal{D}_{train}$, it can be added to $\mathbf{P}_{fix}$ and $\mathbf{P}_{dyn}$, and subsequently $\mathcal{SN}_{fix}$ and $\mathcal{SN}_{dyn}$ are updated using the procedure described in Section \ref{sec:Supernodes_with_derivatives}.

\begin{figure}[h!]
	\centering
	\includegraphics[trim=0 9.25cm 0 0, clip,width=0.9\linewidth]{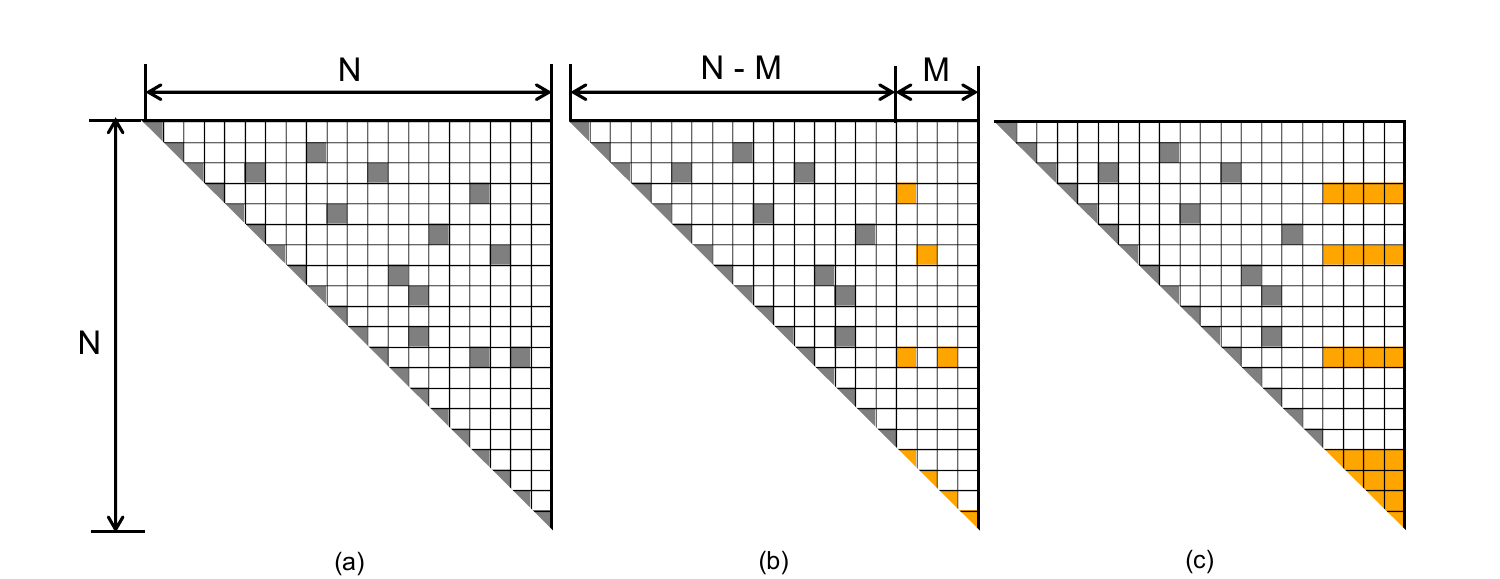}
	\caption{Illustration of generating fixed and dynamic supernodes to generate the sparse matrix. The diagonal entries shown in gray and orange correspond to parents of fixed supernodes, $\mathcal{SN}_{fix}$, and dynamic supernodes, $\mathcal{SN}_{dyn}$, respectively.}
	\label{fig:dynamic_update}
\end{figure}

When a new measurement is available, the generated dynamic supernode is updated in the following manner, and an illustration of the update process is shown in \fig~ \ref{fig:dynamic_update_2}. When new data is available, it is added to the dynamic dataset $\mathbf{P}_{dyn}$, and the set is then reordered based on the MMD ordering scheme, and the parent of the $\mathcal{SN}_{dyn}$ is updated based on the new order. The Cholesky factors of the updated $\mathcal{SN}_{dyn}$ are evaluated and used to build the sparse matrix $\mathbf{U}$. This process is repeated until the model needs to be retrained. The criteria for re-training will be discussed in Section \ref{sec:fast_gp}. When the model gets retrained, a new dynamic set, $\mathbf{P}_{dyn}$, is created as all the available data points, including the fixed set, are subjected to re-ordering. 

\begin{figure}[h!]
	\centering
	\includegraphics[trim=0 9.cm 0 0, clip,width=1.\linewidth]{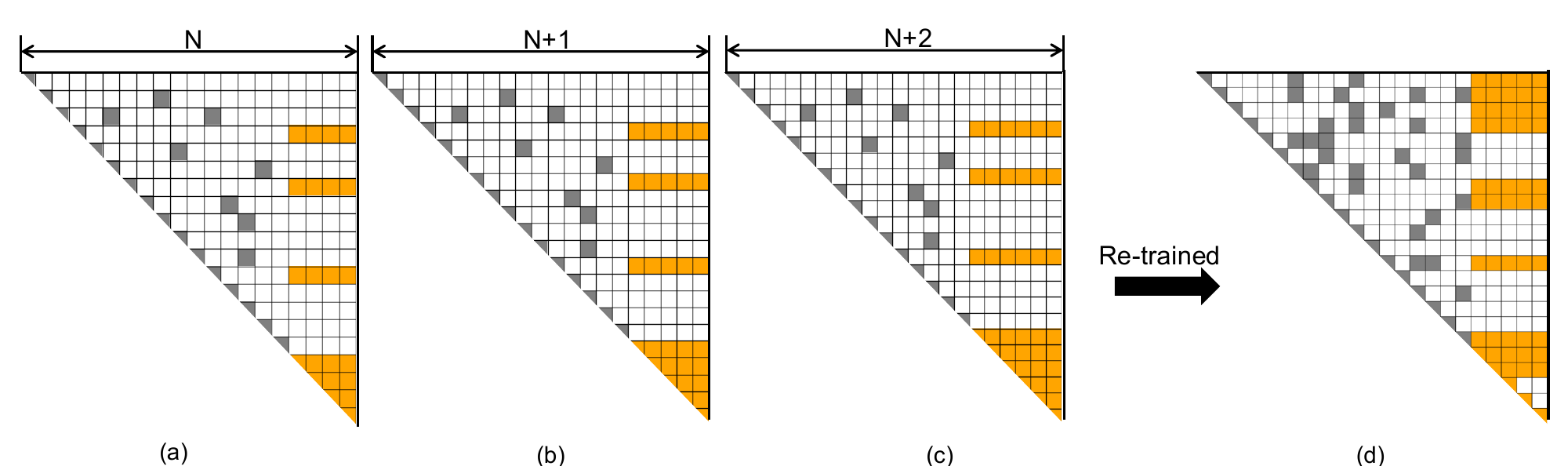}
	\caption{Illustration of the dynamic update of sparse Cholesky GP using SU-Approach1 with re-training. A $N \times N $ sparse matrix with fixed and dynamic set is shown in (a), new points are included in the dynamic supernodes (b and c), and finally, the model is re-trained, during which a new set of fixed and dynamic sets is created, shown in (d).}
	\label{fig:dynamic_update_2}
\end{figure}

\begin{figure}[H]
	\centering
	\includegraphics[width=0.4\linewidth]{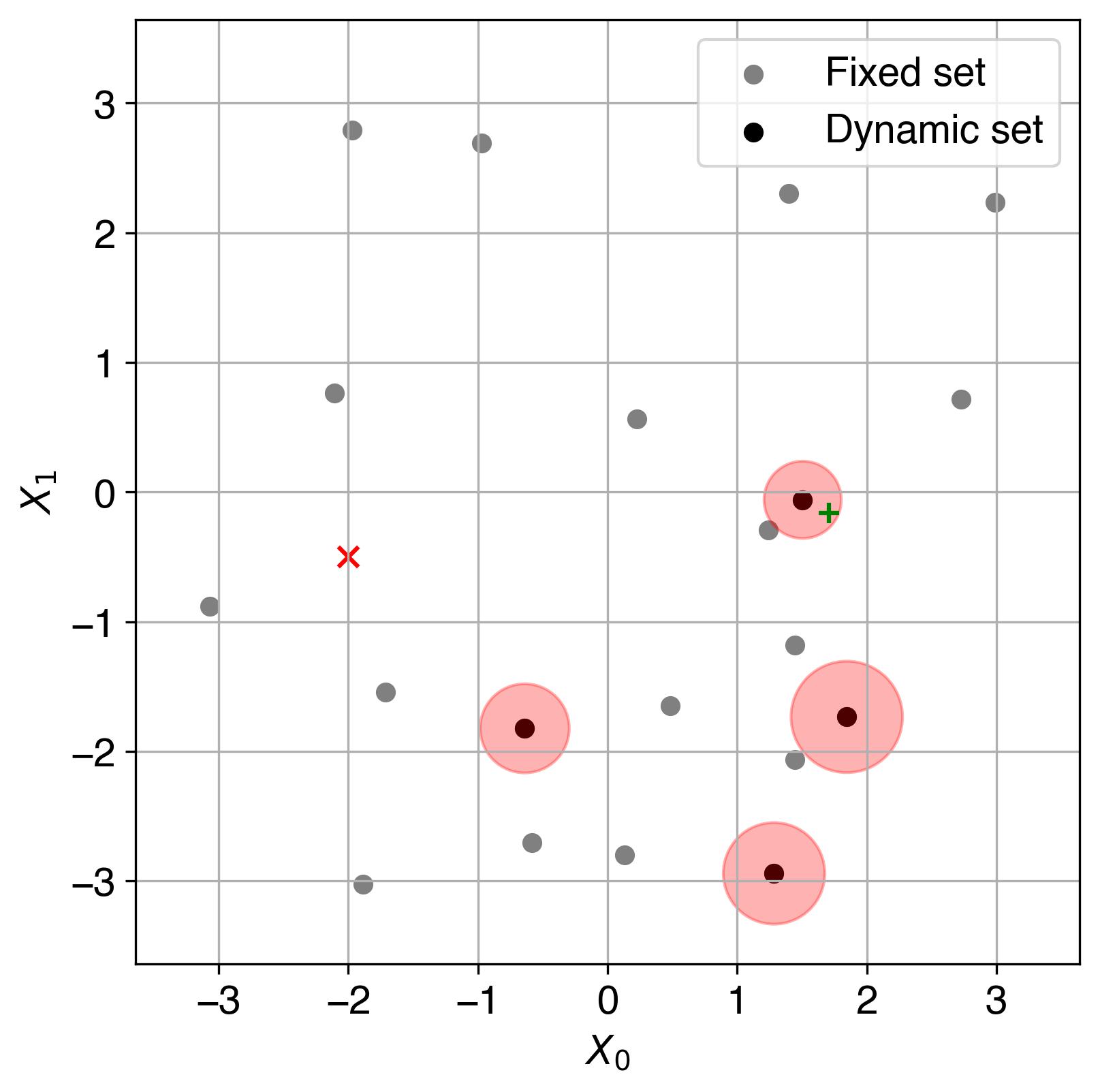}
	\caption{The figure shows the distribution of fixed (gray points) and dynamic set (black points) in a random dataset. When a new point (green point) falls within the radius ($\rho \cdot \mathbf{l}^{(i)}$), an existing supernode is updated. On the contrary, if any of the new point (red point) does not fall within the radius, a new supernode is created.}
	\label{fig:geo_point}
\end{figure}

\begin{figure}[H]
	\centering
	\includegraphics[clip,width=0.75\linewidth]{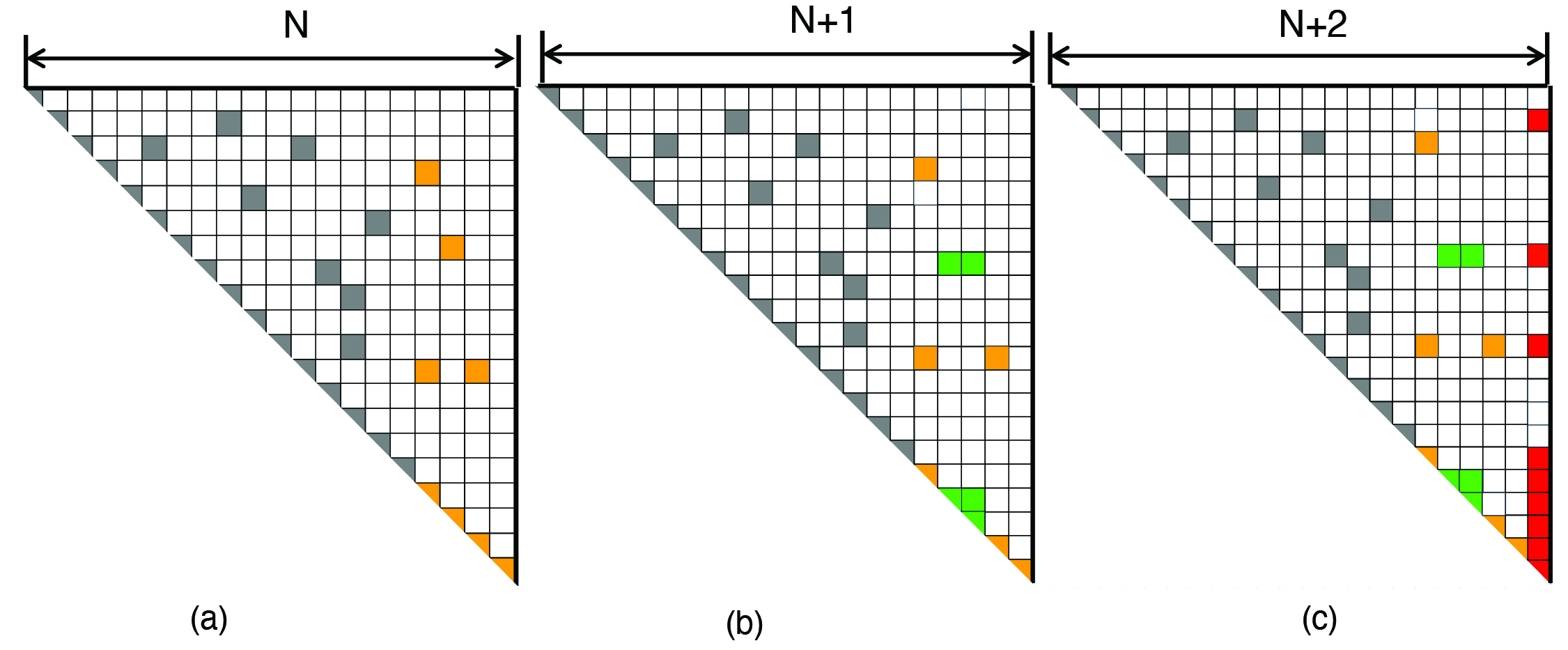}
	\caption{Illustration of the dynamic update of the GP using SU-Approach2. On the left (a), the initial $N \times N$ sparse matrix with fixed and dynamic set highlighted in gray and orange, respectively. (b) Since the green point falls within the set radius, the dynamic supernode is updated, as shown in green. and (c) a supernode is created to accommodate a new point that does not fall within the set radius of any points in the dynamic set. }
	\label{fig:dynamic_update_geo}
\end{figure}

\subsubsection{Supernode Update: Approach 2 (SU-Approach2)}

In this approach, we obtain $\mathbf{P}_{fix}$ and $\mathbf{P}_{dyn}$ through $N-M$ and $M$ split of $\mathbf{P}$, the same way as the previous approach. Once obtained, the supernodes with elements from $\mathbf{P}_{fix}$ and $\mathbf{P}_{dyn}$ are considered  $\mathcal{SN}_{fix}$ and $\mathcal{SN}_{dyn}$ supernodes, respectively. Unlike the previous approach, the dynamic supernodes are not merged to form one dynamic supernode; in other words, multiple $\mathcal{SN}_{dyn}$ can exist in this approach. Once the dynamic supernodes are formed during the initial training, the newly available points are added to one of the dynamic supernodes, or a new supernode is created, which is decided based on the geometric location of the new point with respect to the dynamic set. During the MMD ordering, each point has a length scale, $\mathbf{l}^{(i)}$, associated with it (\eq~\ref{eq:l_mmd}), which is used along with $\rho$ to form supernodes (see Section \ref{sec:Sparse_GP_review}). When the newly available point falls within the radius $(\rho \cdot \mathbf{l}^{(i)})$ of any of the points in the dynamic set, then the new point is added to the $\mathcal{SN}_{dyn}$ and the Cholesky factor is re-evaluated. If the new point does not fall within the radius of any of the points in the dynamic set, a new supernode is created. Figure~\ref{fig:geo_point} illustrates fixed points (gray), dynamic points (black), and a red circle shows the radius of the circle ($\rho \cdot \mathbf{l}^{(i)}$) for each of the dynamic points. If a newly available point (green) lies within the radius of any of the dynamic set, then the corresponding supernode is updated, and factors from other dynamic supernodes (and fixed ) are reused. Visualization of this is provided in \fig~\ref{fig:dynamic_update_geo} (b).  If the new point (red) does not fall within the specified radius of any of the points in the dynamic set, then a new supernode is created, red column in \fig~\ref{fig:dynamic_update_geo}.

\subsection{Fast Derivative-Enhanced Sparse Cholesky GP Updating and Re-training} \label{sec:fast_gp}
\alg  \ref{alg:dynamic_GP} shows our fast derivative-enhanced sparse Cholesky GP updating and re-training algorithm. The criteria for re-training are dependent on several factors. In this work, we choose three different criteria for re-training. First, when the state of the physical twin changes, the bounds of the input features may fall beyond the range of the previously trained DT surrogate. Therefore, it is crucial to identify the outliers when the new information is available from the physical twin. Secondly, when the sparse Cholesky GP is dynamically updated, it may not result in a better surrogate than the deployed surrogate, i.e, the prediction error of the dynamically updated surrogate is more than the existing prediction error on a standard test dataset. In that case, the updated surrogate is not deployed, and the new data is stored and utilized while re-training. Additionally, a fixed budget can be set for the amount of unused new data, and re-training can be triggered when the budget is reached. Third, the deployed surrogate can be assumed to be diverging when the prediction error of the dynamically updated surrogate increases continuously for a set of newly added points. Therefore, a fixed number of continuous divergences is used to trigger re-training of the model.  We now present the details of our algorithm.

Based on the initial set of data, $\mathcal{D}_{train}$, we train a surrogate model, $\mathcal{M}(\theta, \mathcal{D}_{train})$, which is a sparse Cholesky GP with or without derivative information. The hyperparameters, $\theta$, are optimized to reduce the prediction error, $L(\mathcal{M})$, on the test dataset,  $\mathcal{D}_{test}$. The surrogate, $\mathcal{M}$, with optimized hyperparameters, is deployed as the digital twin. The prediction error from the deployed surrogate is set to $L_{best}$. The additional information obtained from the physical twin is added to a dynamic dataset, $\mathcal{D}_{stream}$, which will be used to dynamically update or re-train the deployed digital twin, $\mathcal{M}$. For every new point, $\mathcal{D}_{new}$, from the dynamic set, the algorithm checks whether one of the following criteria is met to trigger retraining. 1) Every $x_{new}$ will be checked if it is an outlier compared to the existing dataset using the outlier detection algorithm (\alg B.1 mentioned in Appendix B). The algorithm calculates the distance between the points used in the current state of GP using the k-nearest neighbor (k-NN) method. We employ a distance-based metric for outlier detection to determine whether a new data point is outside the nominal training domain. Additionally, since we want our work to exploit the screening effect \cite{chen_sparse_2024} of the kernel, a distance-based metric is the most useful. Based on the calculated distance, the threshold for outlier detection, $\tau$, is determined using the hyperparameter, $\eta_{out}$. The new point, $x_{new}$, is classified as an outlier when the distance between the new point and existing points is greater than the calculated threshold, $\tau$, 2). If the number of unused data points, $\eta_{unused}$, is greater than the preset budget, $\eta_{budget}$, 3). If the number of continuous divergence, $\eta_{div}$, is greater than the preset limit, $\eta_{div\_th}$. If one of the above three criteria is met, then a complete re-training of the model is performed with a new $X_{train}$, which is a concatenation of the existing $X_{train}$ and $X_{stream}$. Note that $\mathcal{D}_{stream}$ is created for the sake of numerical experiments; in practice, whenever new data,  $\mathcal{D}_{new}$, is available, it will be immediately used to update the model. 

\begin{algorithm}[H]
	\caption{Fast sparse Cholesky GP update with Outlier-based Retraining}
	\label{alg:dynamic_GP}
	\begin{algorithmic}[1]
		\State \textbf{Input:} Initial training data $\mathcal{D}_{train}=\{X_{train}, f_{train}\}$, test data $\mathcal{D}_{test}$, stream of new data points $X_{stream}$
		\State \textbf{Notation:} $\mathcal{M}(\theta, \mathcal{D})$ is a GP model with hyperparameters $\theta$. $L(\mathcal{M})$ is the model's mean squared error on $\mathcal{D}_{test}$.
		
		\Statex
		\State \textbf{Initialize Model:}
		
		\State $\mathcal{M} \gets \text{Train GP with } \theta^* \text{ and } \mathcal{D}_{train}$\Comment{Where $\theta^*$ is optimized  hyperparameters}
		\State $L_{best} \gets L(\mathcal{M}),\ \eta_{div}=0, \eta_{unused}=0$

		\Statex
		\For{$i = 1$ to $|\mathcal{D}_{stream}|$}
		\State $x_{new} \gets X_{stream}[i]$
		\If{\textsc{IsOutlier}($x_{new}, X_{train}$) \text{or} \textsc{ $\eta_{unused}> \eta_{budget}$}} \text{or} \textsc{ $\eta_{div}> \eta_{div\_th}$} 
		\State $X_{add} \gets \{x \in X_{stream}[1 \dots i] \mid x \notin X_{train} \}$
		\State $X_{train} \gets X_{train} \cup X_{add}$ and update $f_{train}$
		\State $\mathcal{M}(\theta_{new}^*) \gets \Call{FullRetrain}{X_{train},f_{train}}$
			\If {$L(\mathcal{M}(\theta_{new}^*)) <  L_{best}$}
			\State $\theta^* \gets \theta_{new}^*$
			\Else
		\Repeat
		\State $\rho+=1$
		\State $\mathcal{M}(\theta_{new}^*) \gets \Call{FullRetrain}{X_{train},f_{train}}$
		\State $\theta_{new}^* \gets \arg\min_{\theta} L(\mathcal{M}(\theta, \mathcal{D}_{train}))$
		\Until{$L(\mathcal{M}(\theta_{new}^*)) <  L_{best}$ or Sparsity of $\mathcal{M}(\theta_{new}^*) > \text {Lower-bound sparsity}$}
		\EndIf
		\State $\theta^* \gets \theta_{new}^*$
		\State $\mathcal{M} \gets \text{Train GP with } \theta^* \text{ and } \mathcal{D}_{train}, , L_{best} \gets L(\mathcal{M})$

		\Else
		\State $\mathcal{M}_{cand} \gets \textsc{FastUpdate}(\mathcal{M}, x_{new})$ \Comment{Dynamic update}
		\State $L_{cand} \gets L(\mathcal{M}_{cand})$
		\If{$L_{cand} < L_{best}$}
		\State $\mathcal{M} \gets \mathcal{M}_{cand}, L_{best} \gets L_{cand}$ \Comment{Accept the update}
		\State $X_{train} \gets X_{train} \cup \{x_{new}\}$ and update $Y_{train}$
		\Else
		\State $\eta_{unused}+=1$ \Comment{If update is not accepted, $\mathcal{M}$ and $L_{best}$ are unchanged.}
		\If {$L_{cand} / L_{best}>1$} \Comment{Check for continuous divergence}
		\State $\eta_{div}+=1$
		\Else
		\State $\eta_{div}=0$
		\EndIf
		\EndIf
		\EndIf
		\EndFor
	\end{algorithmic}
\end{algorithm}

During retraining, including additional data points does not guarantee an improved surrogate if the $\rho$ is fixed. For a fixed $\rho$, the sparsity of the matrix increases as the number of training points increases. If the prediction error from the retrained model is worse than the $L_{best}$, the $\rho$ value is increased until the retrained model performs better or the sparsity of the retrained model does not fall below a set limit. If the re-training is not triggered, then the sparse Cholesky GP is dynamically updated using one of the approaches discussed in Section \ref{sec:dynamic_S_GP}. We have provided the computational cost analysis for the dynamic sparse GP in Appendix B.2.

\subsection{Numerical Verification Experiments}
The fast update algorithm with dynamic update using SU-Approach1 and SU-Approach2 is verified through numerical experiments (2D Griewank function) and is reported in \fig~\ref{fig:dynamic_results}. In the numerical experiments, we train the initial model using $\mathcal{D}_{train}$, and it is updated using $\mathcal{D}_{stream}$ through dynamic update or re-training based on criteria set in \alg \ref{alg:dynamic_GP}. The size of $\mathcal{D}_{train}$ and $\mathcal{D}_{stream}$ are set at $25$ and $10$ points, respectively. Both datasets are randomly generated, and derivatives up to $4^{th}$-order are included in the dataset. The number of points added to the model from $\mathcal{D}_{stream}$ is shown on the horizontal axis of  \fig~\ref{fig:dynamic_results}, and the prediction error from the initially trained model is shown at $0$ point. During training and update, the model is tested with the same dataset $\mathcal{D}_{test}$. 

Figure~\ref{fig:dynamic_results} (a) shows the results of the experiment where the dynamic update is performed using SU-Approach1. Incorporating derivatives improved the prediction accuracy of the model. As additional data is included, the prediction error of the sparse Cholesky GP is reduced noticeably, irrespective of the order of derivatives included in the training. Similar observations can be made from the results of SU-Approach2, as shown in \fig~\ref{fig:dynamic_results} (b). Upon comparing the errors between the approaches, SU-Approach1 showed a lower prediction error than the error from SU-Approach2. This is due to the fact that for similar hyperparameters, SU-Approach1 exhibits lower sparsity than SU-Approach2, which is due to the formation of one big dynamic supernode by combining all smaller dynamic supernodes. Doing so increases the number of non-zero off-diagonal elements, leading to lower sparsity in SU-Approach1.

\begin{figure}[H]
    \centering
    \includegraphics[width=1\linewidth]{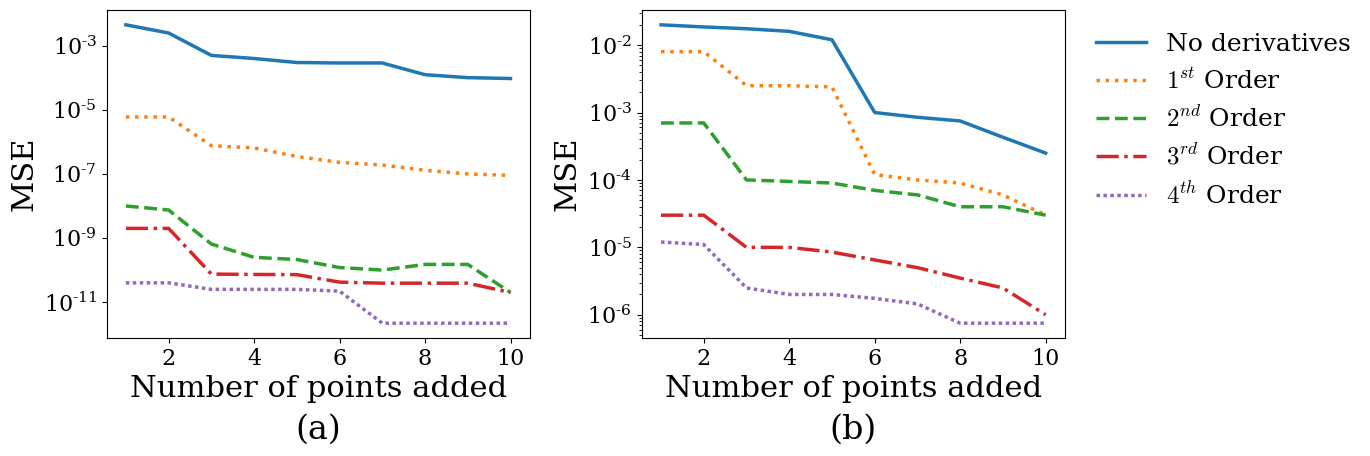}
    \caption{Results from dynamic update of the sparse Cholesky GP using SU-Approach1 (a) and SU-Approach2 (b) tested on a 2D Griewank function with initial $\rho=10$.}
    \label{fig:dynamic_results}
\end{figure}
\section{Application to Digital Twin}
\label{sec:App_DT}

To demonstrate the use of our new dynamic update algorithm containing derivative-enhanced sparse Cholesky modifications, we apply it within a real-world DT framework for predicting fatigue crack growth in an aircraft structure. The work of \cite{LogakannanABXZKMH26} demonstrated that higher-order derivatives are an effective way to train SciML models on an elasticity problem in solid mechanics. Motivated by the work of \cite{LogakannanABXZKMH26}, we applied GPs with higher-order derivatives within the context of DTs for fatigue crack growth in an aircraft structure. In practice, fatigue cracks can evolve into arbitrary and complex shapes, but in many scenarios, they can be represented as a semi-elliptical surface crack in a thin plate, where a cyclic load applied normal to the crack faces drives fatigue crack growth. This geometry is illustrated in \fig \ref{fig:crack_dim}, where $a$ is the crack depth, $c$ is the crack length on the surface of the plate, and $t$ is the plate thickness. The rate at which the fatigue crack grows on the surface, $\frac{dc}{dN}$, is dependent on the material, applied cyclic loading, and the state of the crack, defined by $a$ and $c$. The objective is to update the DT model using periodic inspections of the crack state. The updated DT model is then used to make improved predictions, relative to the initial model, about the future crack states. The localized effect of inspection updates is consistent with Lemma 2.4.3, which establishes that for the squared-exponential kernel, both covariance and derivative-based covariance terms decay exponentially with inter-point distance. Consequently, derivative observations exert their strongest influence in regions that are close to the corresponding measurement locations, with rapidly diminishing impact as the distance increases. Superscript $PT$ and $DT$ will be used to distinguish between measured values (from the physical twin) and the values modeled (from the digital twin) by the sparse Cholesky of GP, respectively.

\begin{figure}[H]
    \centering
    \includegraphics[width=0.35\linewidth]{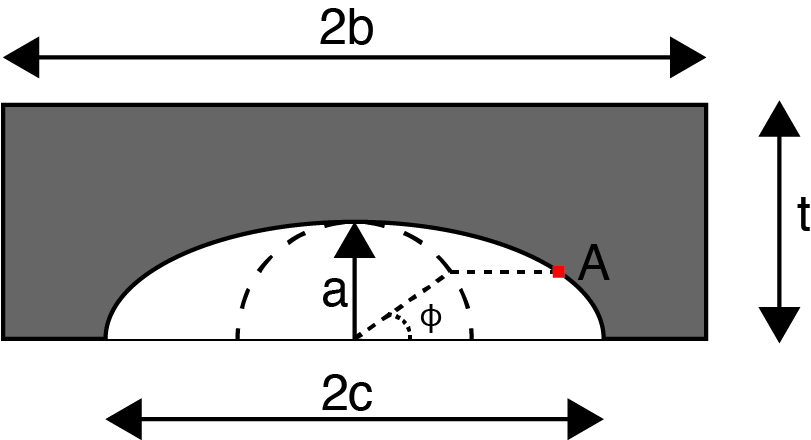}
    \caption{Semi-elliptical surface crack geometry in a finite plate, used to represent a thin aerospace component.}
    \label{fig:crack_dim}
\end{figure}

\subsection{Digital Twin (DT) Workflow} \label{sec:dt_workflow}

The corresponding DT workflow in Figure \ref{fig:DT_workflow} is a detailed version, specific to this application, of the general (abstract) workflow presented in Figure \ref{fig:DT_workflow_generic}. Consistent with the general DT workflow, this application-specific workflow consists of three phases: initialization of the DT model (black outline), a dynamic update given PT observations (red outline), and a real-time tethering between the PT and DT for real-time prediction (green outline) of the crack state. We use subscript $i$ for the dynamic update stage and $j$ for the prognosis stage. In other words, $i$ refers to inspection measurements of the crack state, which are used to update the DT model, while $j$ refers to the prediction of future crack states, given observations and model updates made at $i$.

The process involved in the initialization phase is marked in black dotted lines in \fig \ref{fig:DT_workflow}. This phase involves initial data acquisition and generating the initial DT, which is a sparse Cholesky GP surrogate. The data used to train the DT model can come from lab-scale experiments (\eg nominal material information) or from similar PTs that are already in service (\eg other aircraft in a fleet). The obtained dataset, $\mathcal{D}$, includes the values of $a$, $c$, $\frac{dc}{dN}$, and derivatives of $\frac{dc}{dN}$ with respect to $a$ and $c$ up to an arbitrary order, $d$. Note that $\frac{dc}{dN}$ is referred to as $f$ in Section \ref{sec:math}.

$\mathcal{D}$ is split into training data, $\mathcal{D}_{train}$, and testing data, $\mathcal{D}_{test}$, and used to develop the DT. For the model trained with $d^{th}$ order, $\mathcal{D}_{train}$ includes all the derivative terms up to order $d$, including mixed partial terms. See Section \ref{sec:dt_setup} for details on how these derivatives were obtained. The derivative values on the test dataset, $\mathcal{D}_{test}$, are not used, \ie the trained model is only tested for $\frac{dc}{dN}$ (\ie $f$). Upon completion of this initialization stage, it is important to note that the DT model is nominal in the sense that it does not yet capture any specific details of a particular PT. Once the initial DT model trained and validated, it enters the service alongside the PT where it is updated to include PT-specific details. 

During the inspection points, $i$, a measurement of the crack state ($a$ and $c$) is obtained, which is used to update the DT model (either dynamically or via re-training) by correcting for any discrepancies between observed PT crack growth rate and corresponding DT model predictions. When the DT model is updated, a smooth transition is desired to avoid discontinuities, especially in cases where derivative information is included. It is important to note that in previous work in this area, Bayesian updating methods were used, which was represented by a fixed model form and parameter re-calibration \cite{Yeratapally2020,leser_digital_2020}. In those prior cases, smoothness can be readily maintained. However, dynamically updating or retraining the model, as is done here, requires additional considerations for assessing smoothness. Immediately after initialization, differences between the PT observations and the DT model predictions stem from the fact that the DT model is, at this point, a nominal model that is not yet specific to any particular PT. Initially, DT model updates would be adjusting for as-manufactured differences between a specific PT and nominal case, for example. As flights (service) continue, updates would continue to update for additional PT-specific details that can include specific environments, material behavior, of changes in mechanisms. The newly available data from inspection, $\mathcal{D}_{new}$, is used to dynamically update or retrain the sparse Cholesky GP surrogate using the algorithm \ref{alg:dynamic_GP}, as outlined in red in \fig \ref{fig:DT_workflow}. Similar to $\mathcal{D}_{train}$, $\mathcal{D}_{new}$ includes the values of $a$, $c$, $\frac{dc}{dN}$, and derivatives of $\frac{dc}{dN}$ with respect to $a$ and $c$ up to an arbitrary order, $d$, including all the mixed partial terms for $d>1$.

During service, a real-time tethering between the PT and its DT is enabled for on-the-fly prognosis of the PT, outlined in green in \fig \ref{fig:DT_workflow}. The objective of this step is to predict crack growth in real-time, based on the most up-to-date DT model. At any real-time point, $j$, the DT model can be queried to predict the crack growth increment ($\Delta c$) based on the current state of the crack, expected load ($\Delta \sigma$), and expected number of cycles ($\Delta N$). In this demonstration, the applied load is assumed to be the same throughout each PT flight. However, applied loading represents a DT model input variable that can be measured in real-time, $j$, (\ie unlike crack state, load measurement does not require an inspection at $i$) and, if measured, can be used to make improved DT predictions in this prognosis stage. Finally, to align this demonstration with practice, the DT prognosis steps ($j$) assume no new observations of crack state are made. Consequently, predictions are made using the updated DT model along with an assumption of self-similar crack growth: $\frac{a}{c}$ remains constant. This ratio, however, is updated at the inspection points, $i$, when the PT crack state is measured.

\begin{figure}[H]
    \centering
    \includegraphics[trim=0 0.5cm 0 0cm, clip,width=.85\linewidth]{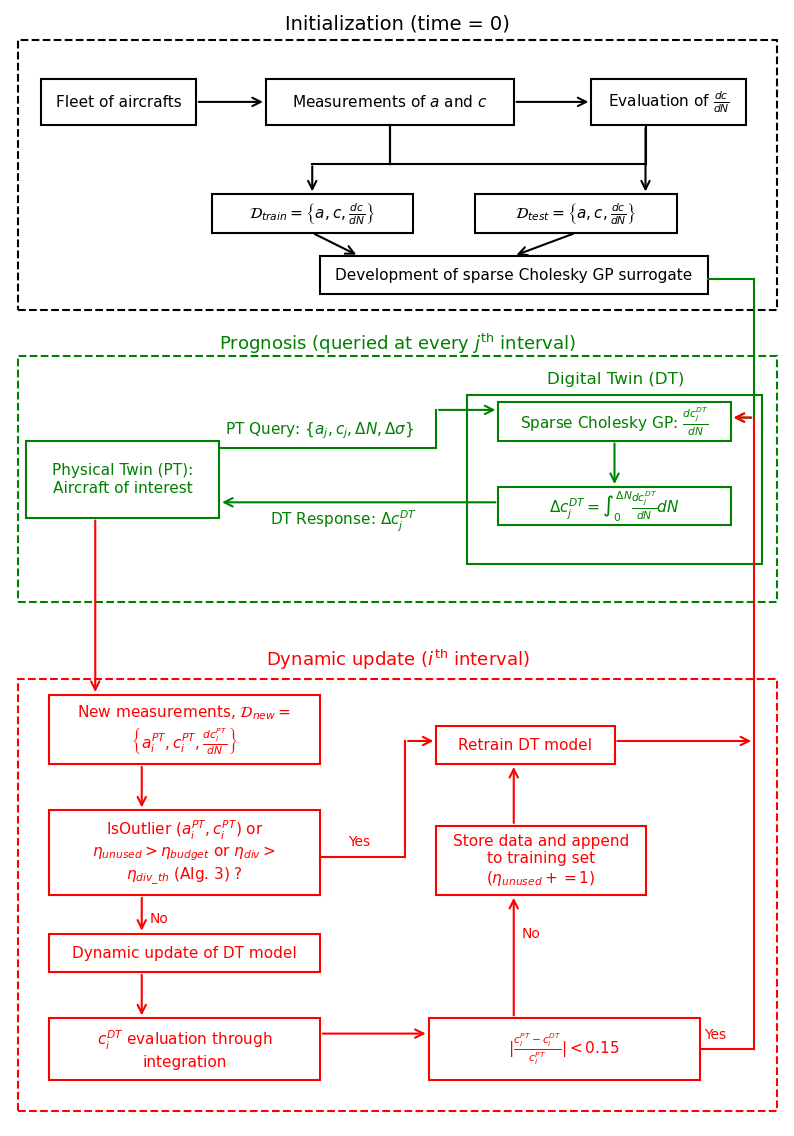}
    \caption{The employed DT workflow with initialization outlined in black, dynamic update of the DT model outlined in red, and real-time prognosis outlined in green. $a^{PT}$ and $c^{PT}$ correspond to the crack state obtained from the PT inspection. $a^{DT}$ and $c^{DT}$ represents the crack state prediction from the DT model.}
    \label{fig:DT_workflow}
\end{figure}

\subsection{Physical Twin (PT) Simulation}
For this demonstration, no physical experiments were completed. Instead, a mechanics-based model was used to simulate PT flights. Assuming constant amplitude cyclic loading, the rate of fatigue crack growth is governed by the crack state ($a$ and $c$), applied loading ($\Delta \sigma = \sigma_{max}-\sigma_{min}$), load ratio ($R= \frac{\sigma_{min}}{\sigma_{max}}$), and material properties ($C$ and $m$). A load ratio, $R=0$, was selected, implying $\sigma_{min} = 0$ and $\Delta \sigma = \sigma_{max}$, or simply $\sigma$ (the subscript `max' is dropped hereafter). Furthermore, the rate at which the fatigue crack grows (\eg $\frac{dc}{dN}$ or $\frac{da}{dN}$) is assumed to follow the Paris law model:

\begin{equation}\label{eq:paris_law}
	\frac{dc}{dN}=C(\Delta K)^m,
\end{equation}

\noindent where $\Delta K$ is cyclic stress intensity factor (described next), and for $R=0$, $\Delta K = K_{max}$. As with $\sigma$, we drop the subscript `max' and refer to $K$ hereafter. The stress intensity factor, $K$, defines the driving force for crack propagation and is a function of load and geometry variable (and not dependent on material properties). Stress intensity factors are generally obtained using high-fidelity computational fracture mechanics, see Ingraffea \cite{Ingraffea2004}, or surrogate models that are most recently developed using a variety of machine learning approaches \cite{GAUTAM2025111387}. Here, we employ the stress intensity factor surrogate model, obtained via symbolic regression, by Merrell \etal \cite{MERRELL2024110432}: 

\begin{equation}
	K=\sigma \cdot \sqrt{\frac{\pi l}{Q}}\cdot f_w(\frac{a}{t},\frac{c}{b}) \cdot M(\frac{a}{t},\frac{a}{c}) \cdot g(\frac{a}{t},\frac{a}{c},\phi)
    \label{eq:sif}
\end{equation}

\noindent where $f_w$ is a finite width correction factor, $M$ accounts for the aspect ratio (\ie $\frac{a}{c}$) of the crack, $g$ accounts for the free surface effects, $l$ is used to measure the perpendicular distance from the point of interest to the closest axis, and $Q$ is the square of the complete elliptic integral of the second kind. Geometrical features $a$, $c$, $t$, $b$, and $\phi$ are defined in Figure \ref{fig:crack_dim}. The equations defining $f_w$, $M$, and $g$ can be found in Equations 18, 19, and 22 of Merrell \etal \cite{MERRELL2024110432}. 

To simulate the PT crack evolution, an initial $a = 0.024$ and $c = 0.012$ $(\frac{a}{c} = 2)$ was selected and inserted into a plate with dimensions $t = 0.1$ in and $b = 0.72$ in. Material properties $C = 5.25 \times 10^{-21}$ and $m = 3.97$ were sampled from a distribution mimicking aluminum 7075-T6 alloy. We refer to these $C$ and $m$ as $C_2$ and $m_2$ to distinguish from the nominal model obtained during the initialization step, which are referred to as $C_1$ and $m_1$. Loading of $\sigma=8500$ psi was then applied.  With this PT, Equation \ref{eq:sif} was used to simulate the stress intensity factor near the surface, $\phi=5^\circ $, and depth, $\phi=90^\circ$. With the $K$ at each point (surface and depth), Equation \ref{eq:paris_law} was then used to simulate the fatigue crack growth rate, which was integrated over $N$ cycles to obtain the simulated PT crack state evolution, $a^{PT}$ and $c^{PT}$. The crack state in the PT is inspected every $\Delta N=5 \times 10^{4}$ cycles, which is then used to update the DT model. The simulated service life of the PT is defined as $N=7.5 \times 10^{5}$ cycles.

\subsection{Digital Twin (DT) Model Setup, No Noise} \label{sec:dt_setup}

The initial DT model was trained within $0.001 \le a \le 0.08$ inches and $0.2 \le \frac{a}{c} \le 2$ inches. This domain is defined such that the ranges of $\frac{a}{c}$, $\frac{a}{t}$, and $\frac{c}{b}$ remain within the valid bounds for $K$, per \eq \ref{eq:sif}. Pairwise $a$ and $c$ data were then obtained by sampling 10 training points and 500 testing points from uniform distributions. The nominal DT model was then defined to consist of a nominal Aluminum alloy with $C_1 = 5.52 \times 10^{-21}$ and $m_1=4$ \cite{HUDSON1969429}. Equation \ref{eq:paris_law} was then used to compute the corresponding $\frac{dc}{dN}$ for each $(a,c)$ pair to form $\mathcal{D}_{train}$ (size 10) and $\mathcal{D}_{test}$ (size 500), \ie representing a nominal prior dataset. Effectively, the differences of $C_2 < C_1$ and $m_2 < m_1$ result in PT crack growth that is significantly slower than the nominal case. 

Obtaining the derivatives of $\frac{dc}{dN}$ with respect to $a$ and $c$, to be included in $\mathcal{D}_{train}$ (see Section \ref{sec:dt_workflow}), could be obtained using numerical differentiation of the existing $(a, c, \frac{dc}{dN})$ data. This approach would likely be most representative of in-practice case, wherein the $(a, c, \frac{dc}{dN})$ data would be measured, while higher-order derivatives would likely not be possible to measure directly. However, numerical differentiation would introduce (well-understood) error into this process and potentially obfuscate the desired assessment of the algorithm's efficacy. Consequently, in this first no-noise study, the derivatives are obtained analytically by taking all derivatives of Equation \ref{eq:paris_law} with respect to $a$ and $c$. The effect of noise is introduced later, at the end of Section \ref{sec:DT_results}.

Using the nominal datasets, a Cholesky factorization of the GP model (no sparsification, see Section \ref{sec:math}), and a dynamic sparse Cholesky of the GP model with an initial $\rho$ of 20 was trained, providing a baseline comparison for the sparse Cholesky GP's performance. For all DT models (whether sparsified or not), a small value of jitter is added along the diagonal to improve the conditioning of the matrix. As established in Lemma~\ref{lem:B1}.1, although the derivative-informed covariance matrix remains symmetric positive-definite, the inclusion of derivative observations can adversely affect its conditioning. This motivates the need for careful tuning of the nugget term and kernel hyperparameters in all subsequent experiments. Please note that special attention must be given in optimizing the regularization term (e.g., nugget term) and the length scale parameters. This behavior is consistent with Lemma 3.0.2, which characterizes the condition number of the derivative-informed covariance matrix and its dependence on the kernel length scale and derivative order. Once trained, we use the DT model to predict the crack growth rate of the crack in the PT. At every $i$ (inspection) step, the sparse Cholesky GP model is either re-trained or dynamically updated depending on the re-training criteria in \alg \ref{alg:dynamic_GP}, and the SU-Approach1 for the dynamic update is used due to the improved prediction accuracy observed in the numerical experiments. 

Finally, the objective of the DT model updating is to provide individualized predictions of the PT, which implies evolving away from the initial nominal DT model, as necessary. To quantitatively assess this objective, we complete a parallel study in which the initialized Cholesky factorized GP (without sparsification) model (DT) remains fixed during the simulated service life (\ie the initial DT model is not updated with new crack state observations). Then, during prognosis steps, DT model predictions (prognosis) are made from the observed crack state ($a^{PT}$ and $c^{PT}$). In doing so, we are able to report on efficacy of the DT model updating and retraining, specifically, while keeping all other variables fixed.

\subsection{Results} \label{sec:DT_results}
Figure \ref{fig:results_DT_1} shows the predicted crack size, $c$, (left column) and the relative percent difference, $\eta$, between the predicted (DT) and actual (PT) crack growth rates (right column) given by:

\begin{equation}
    \eta=\left|\frac{\frac{dc^{PT}}{dN}-\frac{dc^{DT}}{dN}}{\frac{dc^{PT}}{dN}}\right|\cdot 100. 
\end{equation}

Each plot in Figure \ref{fig:results_DT_1} illustrates results for the DT model trained with orders of derivatives ranging from zeroth to fourth. The yellow and blue points (and lines) indicate the crack size observed in the initial nominal DT model and in the PT inspections, respectively. The vertical dotted lines represent the inspection steps, $i$, where the second vertical dashed line corresponds to the first inspection and update point, $i = 1$. At the first vertical dashed line corresponding to $i = 0$ (at $N = 0$), quality control inspection data could be acquired after manufacturing, but before the PT service life. At this point, crack state data ($a^{PT}$ and $c^{PT}$) could be obtained, but PT-specific crack size data would not be available until $i = 1$. Consequently, since the DT model has not yet been updated to account for any specific PT at $i = 0$, the corresponding initial DT model prediction will be that of a nominal crack size (yellow points and lines) until the first update, $i = 1$. Training of the initial nominal model benefited significantly from including derivatives in the training data, which reduced the DT model error from more than 14\% to less than 2\% for the DT model with $1^{st}$-order derivatives and higher as shown in Figure \ref{fig:results_DT_1}(c).

Figure \ref{fig:results_DT_1}(a) plots the results of the baseline case in which the initial Cholesky factorization GP (no sparsification) model was not updated throughout the DT model service life. In this baseline case, it was expected that the DT model will track the nominal (initial) data and not evolve towards the PT data. As expected, plotting $\eta$ over the service life illustrates no improvement and eventual divergence in accuracy with respect to the PT. In other words, the initial nominal model becomes increasingly inaccurate as the PT service life progresses. Additionally, DT models that were trained using at least first-order derivatives more accurately model the nominal crack size. As shown in the left column of Figure \ref{fig:results_DT_1}(a), there is a significant under-prediction of the fatigue crack size for the DT model trained without derivative data.

When the DT model is updated at each $i$ step, as in Figures \ref{fig:results_DT_1}(b) and \ref{fig:results_DT_1}(c), the prediction from the DT model becomes increasingly accurate as the PT service life progresses. This is evident by comparing $\eta$ values along the right side of Figure \ref{fig:results_DT_1}. For the baseline case of Cholesky factorization GP without sparsification, Figure \ref{fig:results_DT_1}(b), $\eta$ at $i=1$ was 7.20\%, 4.90\%, 3.10\%, 1.80\%, and 0.82\% when trained with $0^{th}$-order, $1^{st}$-order, $2^{nd}$-order, $3^{rd}$-order, and $4^{th}$-order, respectively. For the case of the dynamic sparse Cholesky GP model, as in Figure \ref{fig:results_DT_1}(c), the updates gave the $\eta$ accuracy at $i=1$ of 13.40\%, 9.60\%, 6.00\%, 4.10\%, and 2.33\% for $0^{th}$-order, $1^{st}$-order, $2^{nd}$-order, $3^{rd}$-order, and $4^{th}$-order, respectively. The clear advantage of our dynamic sparse Cholesky GP model lies in the relatively cheap computational cost of $\mathcal{O}(Ms^2)$ as described in Appendix B.2 (compared to the $\mathcal{O}(N^{3})$ computational cost of retraining Cholesky factorization GP without sparsification), while giving $\eta$ accuracy errors that are similar to the case of Cholesky factorization GP without sparsification.

For both the Cholesky factorized GP (without sparsification) with re-training, \ref{fig:results_DT_1}(b) and the dynamic sparse Cholesky GP model, Figure \ref{fig:results_DT_1}(c), the $\eta$ values decrease significantly with fatigue cycles. There is a clear trend: employing higher-order derivatives achieves better accuracy than lower-order derivatives. This trend aligns with Lemma~\ref{lem:B2}.2, which predicts a monotonic reduction in posterior uncertainty as additional derivative information is incorporated. For the $\eta$ plot of Figure \ref{fig:results_DT_1}(b), at inspection point, $i=9$ at fatigue cycle $4.5 \times 10^5$, the $\eta$ accuracies were 1.80\%, 1.20\%, 0.80\%, 0.40\%, and 0.20\% for $0^{th}$-order, $1^{st}$-order, $2^{nd}$-order, $3^{rd}$-order, and $4^{th}$-order, respectively.  For the $\eta$ plot of Figure \ref{fig:results_DT_1}(c), at inspection point, $i=9$ at fatigue cycle $4.5 \times 10^5$, the $\eta$ accuracies were 2.10\%, 1.40\%, 0.80\%, 0.50\%, and 0.49\% for $0^{th}$-order, $1^{st}$-order, $2^{nd}$-order, $3^{rd}$-order, and $4^{th}$-order, respectively. This clearly demonstrates the effectiveness of including higher-order derivatives in reducing prediction errors. The observed reduction in prediction error with increasing derivative order is consistent with Lemma~\ref{lem:B2}.2, which shows that incorporating higher-order derivative observations reduces the posterior variance of the GP model. In addition to including higher-order derivatives, the dynamic sparse Cholesky GP method (from Figure \ref{fig:results_DT_1}(c)) gave comparable results to the Cholesky factorized GP model (without sparsification) (Figure \ref{fig:results_DT_1}(b)), but with a computationally cheaper complexity of $\mathcal{O}(Ms^2)$ (instead of the $\mathcal{O}(N^{3})$ computational cost of retraining Cholesky factorized GP model without sparsification, or the $\mathcal{O}(Ns^2)$ computational cost of sparse Cholesky GP). The memory cost of dynamic sparse Cholesky GP is $\mathcal{O}(N s)$ (compared to the memory cost of $\mathcal{O}(N s)$ for sparse Cholesky GP and $\mathcal{O}(N^2)$ for retraining Cholesky factorized GP without sparsification). Note that $M \ll N$ and $s$ denotes the average number of nonzero entries per column in the sparse Cholesky factor. The computational cost analysis for the dynamic sparse Cholesky GP is mentioned in Section \ref{sec:dynamic_GP_for_DT} (referred from Appendix B.2). The runtime to obtain all the data (for all derivative orders) presented in Figure \ref{fig:results_DT_1}(b) for Cholesky factorization GP (without sparsification) with retraining is approximately 25 minutes on 1 CPU core of Intel XeonSP Cascadelake at the Center of High Performance Computing at the University of Utah. The runtime is approximately 3 minutes to obtain all the data (for all derivative orders) presented in Figure \ref{fig:results_DT_1}(c) for dynamic sparse Cholesky GP (with 10 inspection points) on the same hardware resources. These time estimates are done assuming that the nugget term and length scale were already optimized before clocking the wall time.

We also conducted additional experiments to quantify the effect of the dynamic model update frequency. We (roughly) doubled the update frequency (compared to Figure \ref{fig:results_DT_1}(c)), while keeping the total lifespan of the aircraft to be the same. The left column of Figure \ref{fig:results_DT_1}(d) shows the crack size w.r.t. the fatigue cycles, while the right column shows the $\eta$ accuracy w.r.t. the fatigue cycles. The $\eta$ accuracy plot of Figure \ref{fig:results_DT_1}(d) demonstrates that doubling the update frequency increased the accuracy of the predictions compared to the prediction accuracies plotted in Figure \ref{fig:results_DT_1}(c). In Figure \ref{fig:results_DT_1}(d), the $\eta$ errors at the last fatigue cycle ($N = 4.5 \times 10^5$) are 0.92\%, 0.56\%, 0.31\%, 0.14\%, and 0.048\% for $0^{th}$-order, $1^{st}$-order, $2^{nd}$-order, $3^{rd}$-order, and $4^{th}$-order, respectively. Thus, there is an approximate decrease of 60\%-70\% (on average) in the $\eta$ errors when the inspection points are doubled. The decrease in error at the last fatigue cycle for the $4^{th}$-order derivative is approximately 90\%. Figure \ref{fig:double_update_vs_single_update} shows the $\eta$ errors for $0^{th}$-order, $2^{nd}$-order, and $4^{th}$-order derivatives for both 10 inspection points (dashed lines) and 19 inspection points (solid lines). We only plotted $0^{th}$-order, $2^{nd}$-order, and $4^{th}$-order derivatives in Figure \ref{fig:double_update_vs_single_update} for plotting clarity. Clearly, the DT model accuracy improves with increased model update frequency.

Results to this point demonstrate a significant decrease in prediction error when increasing derivative order or increasing inspection frequency, as would be expected. Considering the crack length, $c$, results in Figure \ref{fig:results_DT_1} it is also seen that the true (PT) crack length exceeds that of the predicted (DT) crack length. Aircraft components that are governed by damage tolerance concepts are retired at a specified crack length that is deemed to be unsafe (risk of critical fracture is high). These results demonstrate model prediction that is not conservative: the actual crack length would be (significantly) larger than the predicted crack length. However, using the developed algorithm, both options of increasing derivative order or increasing inspection frequency demonstrate a smooth convergence to the PT behavior which mitigates the non-conservatism.

\begin{figure}[H]
\centering

\begin{subfigure}{0.48\textwidth}
\centering
\includegraphics[width=\linewidth]{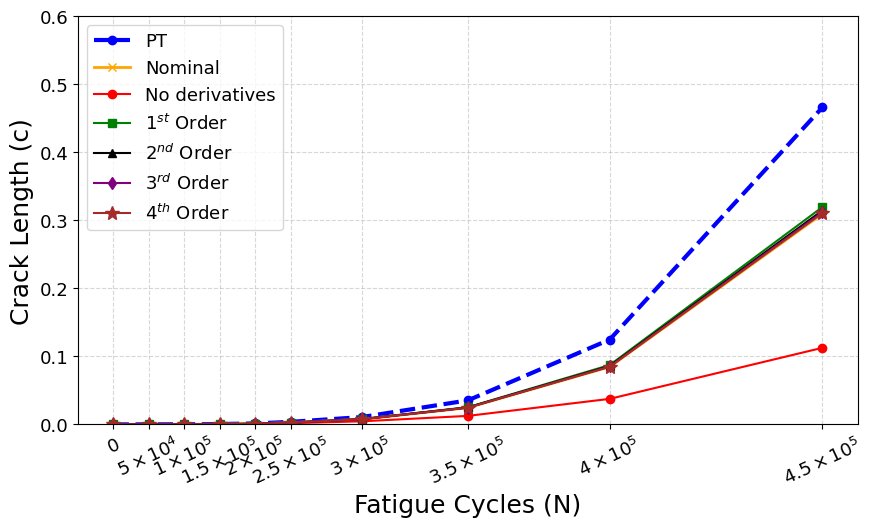}
\end{subfigure}
\hfill
\begin{subfigure}{0.48\textwidth}
\centering
\includegraphics[width=\linewidth]{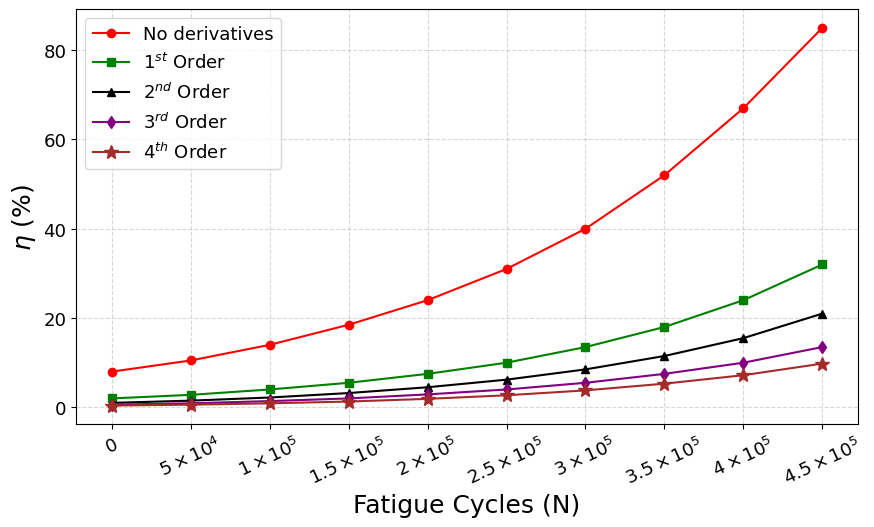}
\end{subfigure}

(a)

\vspace{0.8em}

\begin{subfigure}{0.48\textwidth}
\centering
\includegraphics[width=\linewidth]{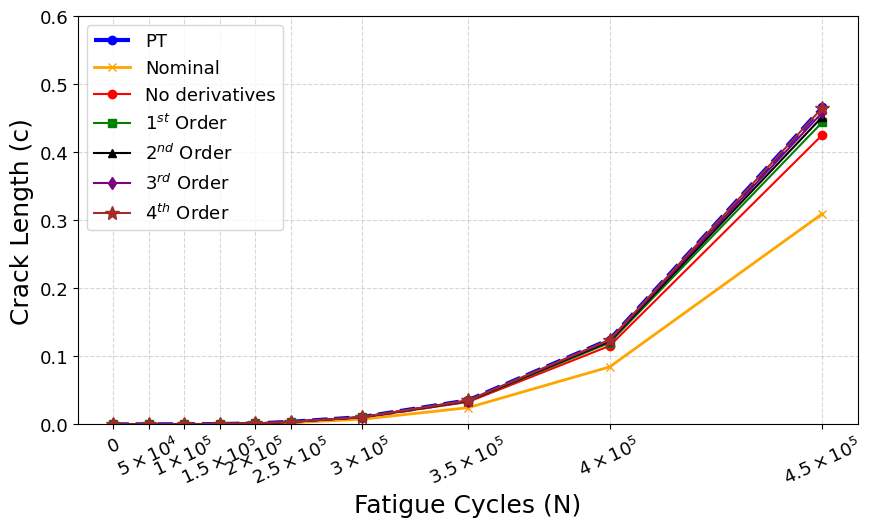}
\end{subfigure}
\hfill
\begin{subfigure}{0.48\textwidth}
\centering
\includegraphics[width=\linewidth]{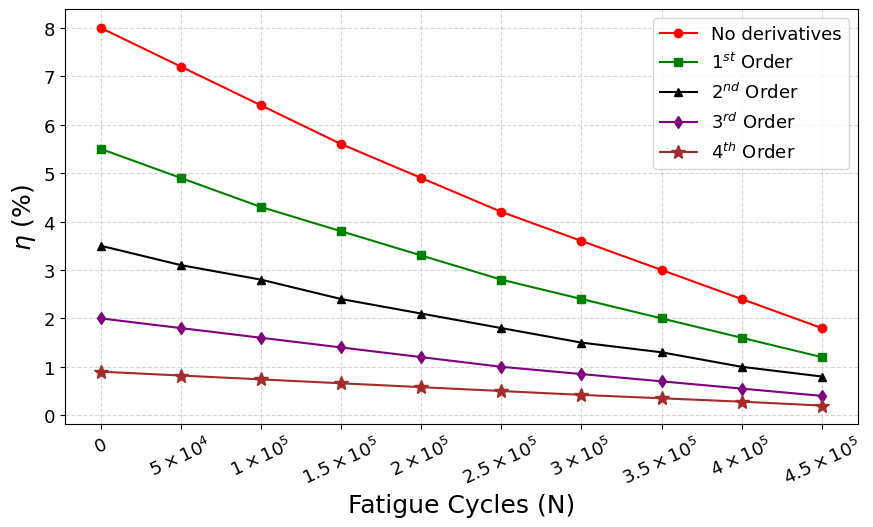}
\end{subfigure}

(b)

\vspace{0.8em}

\begin{subfigure}{0.48\textwidth}
\centering
\includegraphics[width=\linewidth]{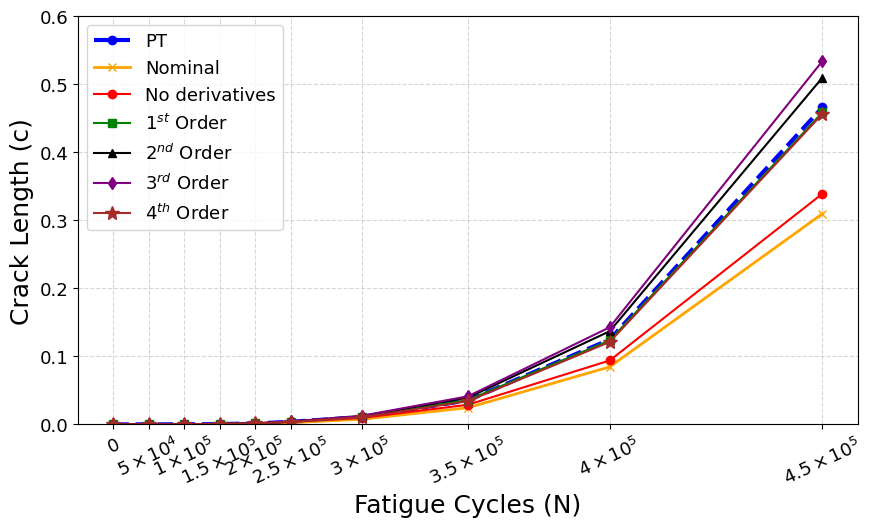}
\end{subfigure}
\hfill
\begin{subfigure}{0.48\textwidth}
\centering
\includegraphics[width=\linewidth]{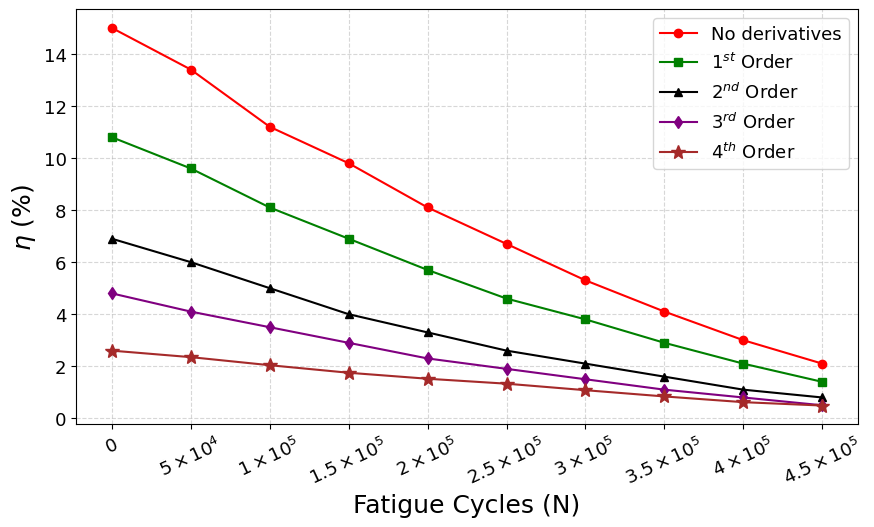}
\end{subfigure}

(c)

\vspace{0.8em}

\begin{subfigure}{0.48\textwidth}
\centering
\includegraphics[width=\linewidth]{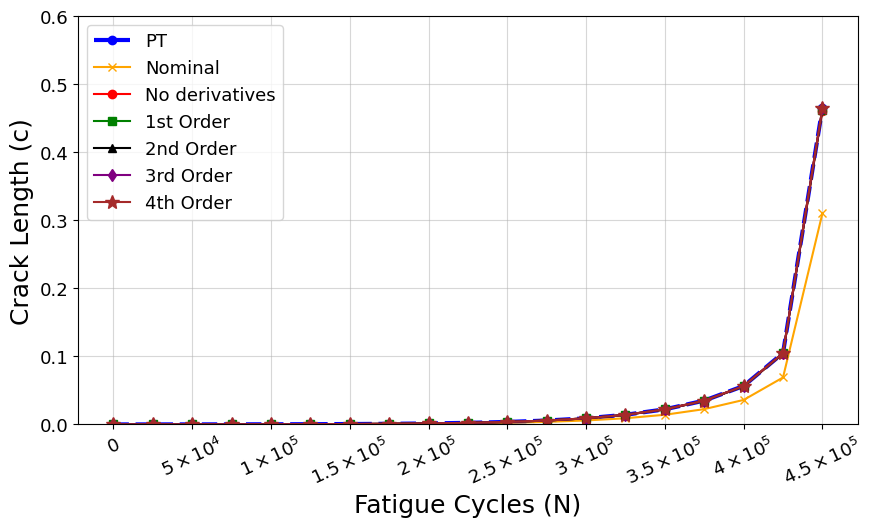}
\end{subfigure}
\hfill
\begin{subfigure}{0.48\textwidth}
\centering
\includegraphics[width=\linewidth]{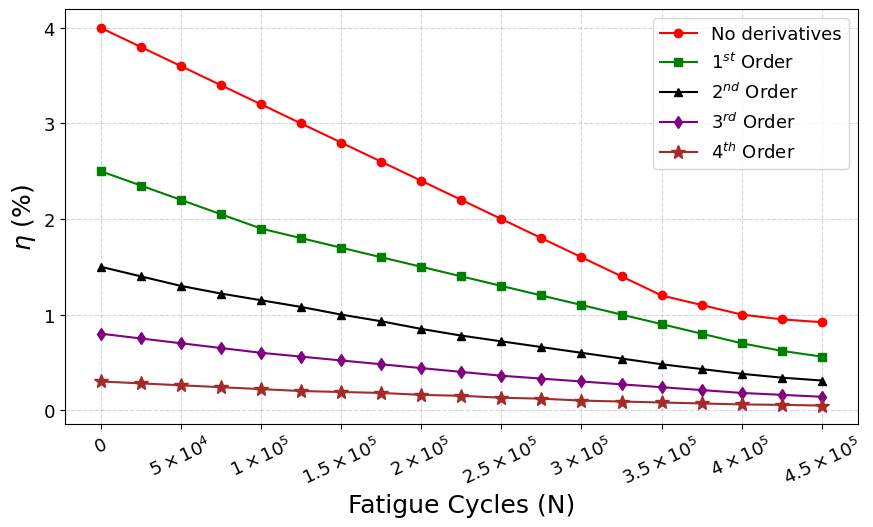}
\end{subfigure}

(d)

\caption{Left column shows the crack size predicted and the right column shows the $\eta$ accuracies w.r.t. the PT. (a) nominal DT, (b) Cholesky factorization GP model (without sparsification) with re-training at every $i$ step, and (c) dynamic sparse Cholesky GP model with initial $\rho$ of 20. (d) dynamic sparse Cholesky GP with initial $\rho$ of 20 and (roughly) double the inspection points as (c).}
\label{fig:results_DT_1}

\end{figure}

\begin{figure}[H]
    \centering
    \includegraphics[width=0.85\linewidth]{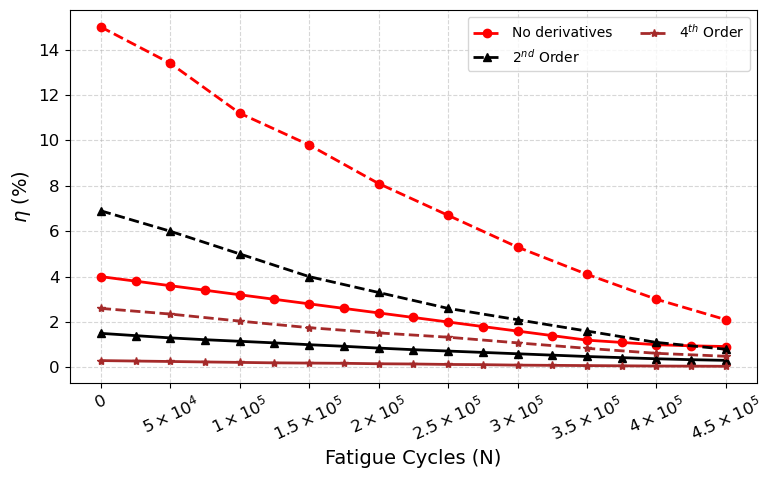}
    \caption{$\eta$ errors before and after (roughly) doubling the update frequencies for dynamic sparse Cholesky GP. Dashed lines correspond to 10 inspection points. Solid lines represent to 19 inspection points.
    }
    \label{fig:double_update_vs_single_update}
\end{figure}

In practice, there will be intrinsic and extrinsic sources of uncertainty for any DT application. Consequently, we conduct additional experiments to demonstrate the robustness of our dynamic sparse Cholesky GP algorithm within the DT application for noisy observations. Specifically, we introduce intrinsic (aleatoric) noise by treating $C$ and $m$ as random variables. This modification captures the more realistic case of inherent behavior variability, which stems from random material and manufacturing variations. To model this case, the PT fatigue crack growth data was generated using a Monte Carlo simulation based on Paris law (Eq.\ref{eq:paris_law}). $C$ was modeled using a lognormal distribution,

\begin{equation}\label{eq:lognormal_C}
	C \sim \mathrm{LogNormal}
    \left(
    \ln \left(5.25\times10^{-21}\right), 0.25^{2}
    \right),
\end{equation}

\noindent which ensures that samples are positive-valued while representing the multiplicative variability commonly observed in fatigue crack growth data. Simultaneously, $m$ was modeled using a normal distribution,

\begin{equation}\label{eq:normal_m}
    m \sim \mathcal{N}
    \left(
    3.2,
    0.15^{2}
    \right).
\end{equation}

For each Monte-Carlo simulation, independent samples of $C$ and $m$ were drawn from their respective distributions and substituted into the Paris law (Eq.\ref{eq:paris_law}). Numerical integration of the resulting crack growth model produced a corresponding crack length history over the component's service life. Repeating this process produced an ensemble of crack growth trajectories that represented the uncertainty associated with fatigue crack propagation. 

Figure \ref{fig:noisy_DT_0_order} shows the $\eta$ accuracy plot. The $\eta$ bounds at $i=0$ ($N=0$) are 4.29\% – 30.82\%, 2.93\% – 21.01\%, 1.95\% – 14.01\%, 1.17\% – 8.40\%, and 0.59\% – 4.20\% for $0^{th}$-order, $1^{st}$-order, $2^{nd}$-order, $3^{rd}$-order, and $4^{th}$-order, respectively. The bounds reduced to 0.45\% – 3.25\%, 0.31\% – 2.21\%, 0.21\% – 1.48\%, 0.12\% – 0.89\%, and 0.06\% – 0.44\% for $0^{th}$-order, $1^{st}$-order, $2^{nd}$-order, $3^{rd}$-order, and $4^{th}$-order, respectively at $i=9$ ($N=4.5 \times 10^{5}$). 

\begin{figure}[H]
    \centering
    \includegraphics[width=1.0\linewidth]{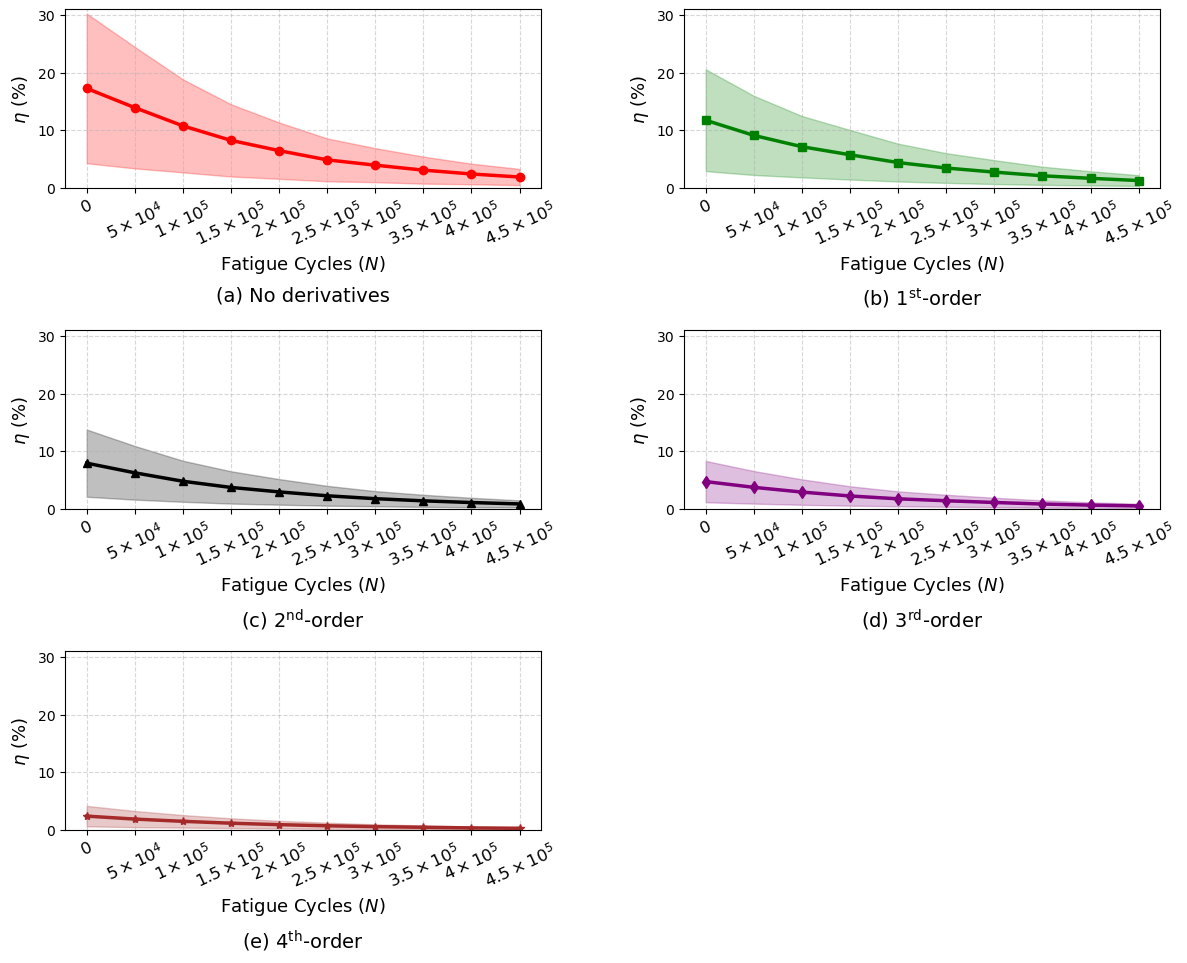}
    \caption{$\eta$ accuracy plots (for each derivative order) for DT simulations using $C$ and $m$ drawn from
    $C \sim \mathrm{LogNormal}\!\left(\ln\!\left(5.25\times10^{-21}\right),\,0.25^{2}\right)$
    and
    $m \sim \mathcal{N}\!\left(3.2,\,0.15^{2}\right)$, respectively.}
    \label{fig:noisy_DT_0_order}
\end{figure}

Additionally, we conducted experiments to demonstrate the effectiveness of our dynamic sparse Cholesky GP method under partial derivatives observations. Figure \ref{fig:partial_observations_DT} shows the $\eta$ accuracy plot for the results of the experiments when a fixed 25\%, 50\%, 50\%, 50\%, and 50\% of the $0^{th}$-order, $1^{st}$-order, $2^{nd}$-order, $3^{rd}$-order, and $4^{th}$-order derivatives are missing, respectively. The figure demonstrates that the $\eta$ error decreases with increasing fatigue cycles, and higher-order derivatives have a lower $\eta$ error than the lower-order derivatives. 

\begin{figure}[H]
    \centering
    \includegraphics[width=0.85\linewidth]{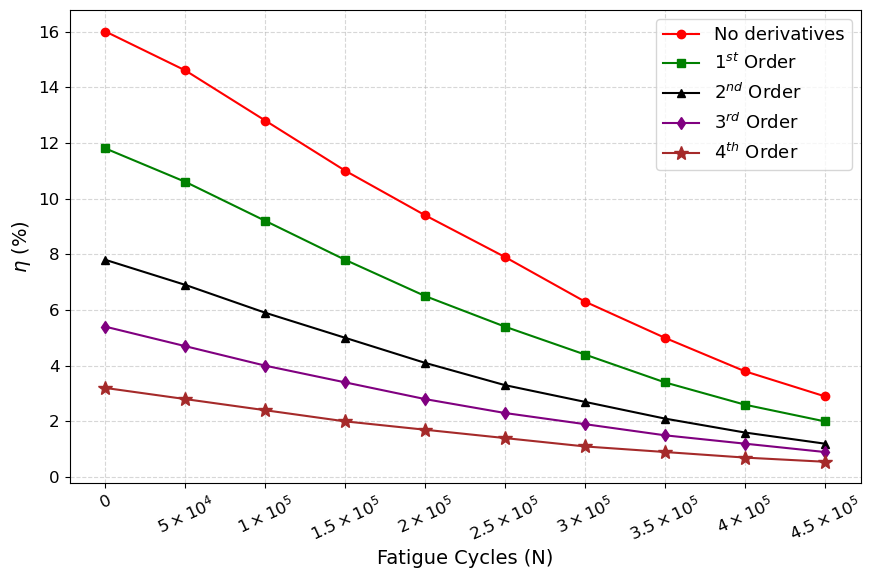}
    \caption{$\eta$ accuracy plot under partial derivatives observations.}
    \label{fig:partial_observations_DT}
\end{figure}

\section{Summary and Conclusions}
\label{sec:summary}
In this work, we demonstrate an end-to-end digital twin (DT) framework to model fatigue crack growth in an aerospace vehicle. Our contribution is an engineered system that acts as a DT demonstrator, combining and extending derivative-informed surrogate modeling, dynamics updating, and modified sparse linear algebra solution technologies. 
We present a dynamically updated sparse Cholesky algorithm to utilize the updated data and corresponding derivatives from a PT. The results of this study are summarized below:  
\begin{itemize}
    \item Numerical experiments on GPs with derivatives showed a significant reduction in prediction error. For a 3D Griewank function with 27 points, GP models trained with $4^{th}$-order derivatives showed prediction error on the order of $10^{-13}$ when compared to $10^{-3}$ for GP models trained without derivatives. Although incorporating derivatives improves the model prediction accuracy, it comes at the expense of an increased size of the covariance matrix.  
    \item We leverage a sparse Cholesky algorithm that utilizes maximum-minimum ordering and an aggregated sparsity set. This algorithm is then generalized to incorporate derivatives in the training data. The numerical experiments showed that increasing the order of derivative improved the prediction accuracy of the sparse Cholesky factor of the precision model. For sufficiently larger sparsification factor $\rho$, the sparse Cholesky factor of the precision matrix has shown similar prediction accuracy to that of the exact GP. 
    \item We develop and present two different dynamic update algorithms, which enable new data to be added to the sparse Cholesky factor of the precision matrix without requiring complete retraining. Whenever new information is available, it is added to the dynamic supernode and only the Cholesky factors of the dynamic supernode are reevaluated; remaining factors are reused. Such a dynamic update offers a significant computational advantage as it eliminates the need for full matrix evaluation. Numerical experiments showed that the prediction accuracy of the sparse Cholesky-based GP model improved when new data were added to the sparse Cholesky factor of the precision matrix. 
    \item Finally, we apply the developed derivative-informed dynamic sparse Cholesky algorithm to a fatigue crack growth DT problem. Similar to the simpler numerical experiments, incorporating derivatives was observed to significantly increase the DT model prediction accuracy. Additionally, the dynamic update of the sparse Cholesky factor of the precision matrix (DT model) throughout the simulated service life demonstrated the ability to individualize initially nominal predictions to that of a specific PT. Without such updates, predictions from the initial nominal DT model diverged from the PT behavior, while increasing the DT model update frequency continually improved the DT model accuracy. 
\end{itemize}

We conclude by provide a few application-oriented observations concerning our DT engineered system solution that might be of help to fellow practitioners. Some of these observations are well-known within their subarea (e.g., GP modeling or computational linear algebra), but may not rise to the forefront when debugging a DT system.  First, consistent with what was observed in \cite{LogakannanABXZKMH26}, one must be mindful of the accuracy of the derivatives provided by the PT -- in particular, whether the error and/or noise associated with the derivatives is different than the primary field variables. Second, the conditioning of the linear system generated when using derivative information requires focused attention, with a regularization term (e.g., nugget term) requiring optimization being often being required. Third, special attention should be paid not only to the units, but also to the order of magnitude, of the primary variables and their derivative quantities.  When combined with the earlier conditioning comment, one should be mindful that extensive (and expensive) length scale and nugget parameter optimization may be necessarily.  Lastly, the frequency of inspections should be determined based upon not only computational costs and DT constraints but also based on the inherent time scales of critical events in the PT.  

 The algorithm formulation assumes independent noise structure for increasing derivative orders. In practice, this is not always the case, e.g., numerical differentiation might be used, which would enforce a clear dependence. Consequently, a limitation of the presented algorithm is that independent physical attributes should be directly measured by independent sensor types (e.g., measuring displacement, velocity, and acceleration directly). Nevertheless, the application study (aircraft structures) with intrinsic noise sources demonstrates a robustness to noise.  For future work, a study on the effect of dependent noise structures across derivative orders, with algorithm extensions, is recommended.

\section*{CRediT authorship contribution statement}
$\textbf{Shridhar Vashishtha:}$ Writing -- review \& editing, Writing -- original draft, Validation, Methodology, Formal analysis, Data curation, Conceptualization, Visualization, Investigation, Software; $\textbf{Krishna Prasath Logakannan:}$ Writing -- original draft, Software, Methodology, Formal analysis; $\textbf{Jacob Hochhalter:}$ Writing -- review \& editing, Writing -- original draft, Validation, Supervision, Resources, Methodology, Funding acquisition, Conceptualization; $\textbf{Shandian Zhe:}$ {Writing -- review \& editing, Writing -- original draft, Supervision, Methodology, Funding acquisition, Formal analysis, Conceptualization; $\textbf{Robert M. Kirby:}$ Writing -- review \& editing, Writing -- original draft, Supervision, Project administration, Methodology, Funding acquisition, Conceptualization.

\section*{Data availability}
The code is available at \url{https://github.com/Shridhar0305/derivative_enhanced_GPs_DTs}. Data will be made available upon request.

\section*{Declaration of competing interest}
The authors declare that they have no known competing financial interests or personal relationships that could have appeared to influence the work reported in this paper.

\section*{Acknowledgements}
This work has been supported by the AFOSR MURI grant FA9550-20-1-0358 and the NSF DMS 2529112. 
%
\clearpage
\appendix

\section*{\MakeUppercase{Appendix A. Ordering Algorithms}}
\label{appA}
\setcounter{algorithm}{0}
\setcounter{figure}{0}
\renewcommand{\thealgorithm}{A\arabic{algorithm}}

\noindent \textbf{A.1.} \textbf{\textit{Additional Details Regarding the Ordering Algorithms Discussed in Section 3}}

This section provides details on the ``point-wise ordering algorithm 2", the ``measurement-wise ordering algorithm 1", and the ``measurement-wise ordering algorithm 2". Additionally, a detailed comparison between the predictive performance of these three methods is provided in contrast to the predictive performance of ``point-wise ordering algorithm 1" presented in Section 3 of this work. As a note to the reader: the work of \cite{chen_sparse_2024, schafer_sparse_2021,chen_solving_2021,kang_correlation-based_2023} presents an ordering similar to our ``measurement-wise ordering algorithm 1".

In measurement-wise ordering algorithm 1, the derivative measurements are grouped separately. Therefore, in measurement-wise ordering algorithm 1, all the $0^{\text{th}}$ order derivatives are grouped together, followed by all the $1^{\text{st}}$ order derivatives, then all the $2^{\text{nd}}$ order derivatives are grouped together, followed by the $3^{\text{rd}}$ order derivatives, and finally all the $4^{\text{th}}$ order derivatives are grouped together. For better understanding, we illustrate the structure of the covariance matrix that has function value $f(x)$ and its corresponding first-order derivative measurement $\nabla f(x)$ in \fig \ref{fig:K_with_der_measurement_1}. The plot shows the structure of the covariance matrix when all the derivative-free measurements, $f$, are ordered first, followed by $\nabla f$, i.e., $\mathbf{F} =[f^{\mathbf{P}(1):\mathbf{P}(N)},\nabla f^{\mathbf{P}(1):\mathbf{P}(N)}]$. For Cholesky factorized GP (no sparsification), the ordering should not have any effect on the distribution, as this method simply permutes the same elements within the matrix.

\begin{figure}[h!]
	\centering
        \renewcommand{\thefigure}{A.1.1}
	\includegraphics[width=0.4\linewidth]{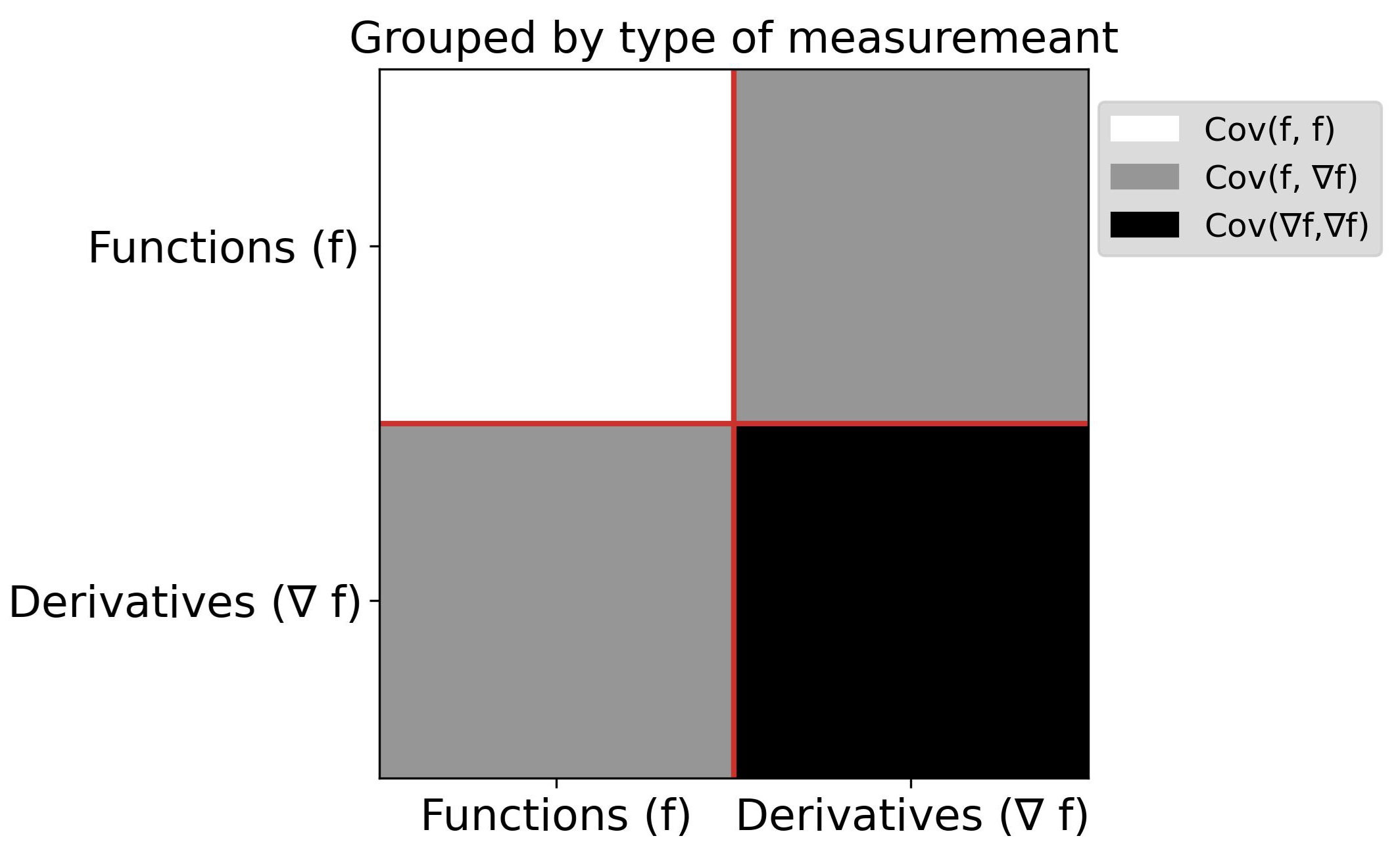}
	\caption{The figure illustrates the measurement-wise ordering algorithm 1 of incorporating derivative measurement in the formation of the $\mathbf{K}_{der}$ matrix. The derivative measurements are placed after the functional observation, following the same ordering. }
	\label{fig:K_with_der_measurement_1}
\end{figure}

For this method, we obtain the initial ordering $\mathbf{P}$ using \eq \ref{eq:P_mmd} without any derivative measurements, and then we extend $\mathbf{P}$ to incorporate derivative measurements to obtain $\mathbf{P}^{d}$. A subscript $me$–1 is added to $\mathbf{P}^{d}$ to refer to the ordering grouped by the measurement-wise ordering algorithm 1. The algorithm to obtain $\mathbf{P}^{d}_{me-1}$ is shown in Algorithm \ref{alg:p_meas}.

\begin{center}
\begin{minipage}{0.46\textwidth}
\begin{algorithm}[H]
    \caption{Constructing the $\mathbf{P}^{d}_{me-1}$ array}
    \label{alg:p_meas}
    \begin{algorithmic}[1]
        \State \textbf{Input:} $\mathbf{P}$ \text{ from MMD ordering}
        \State \textbf{Output:} $\mathbf{P}^{d}_{me-1}$
        \Statex
        \State $td \gets \lfloor N_d / N \rfloor$
        \For{$b \gets 1$ to $td$}
            \State \text{offset} $\gets b \cdot N$
            \For{$k \gets 1$ to $N$}
                \State $\mathbf{P}^{d}_{me-1}[\text{offset} + k] \gets \mathbf{P}[k] + \text{offset}$
            \EndFor
        \EndFor
    \end{algorithmic}
\end{algorithm}
\end{minipage}
\end{center}

To obtain the $\mathbf{P}^d_{me-1}$, the functional observations are first placed as per the MMD ordering, followed by all the derivative observations, following the same ordering. In other words, each derivative measurement is ordered the same as $\mathbf{P}$, and is stacked to $\mathbf{P}^d_{me-1}$. 

$\mathcal{SN}$ are originally obtained without any derivative measurement, through the procedure described in Section \ref{sec:Sparse_GP_review}, and then they are generated to include derivative measurements in them. This is done by generating a new set of parents and children that corresponds to the indices of the derivative measurements, and they are added to the existing supernode, $\mathcal{SN}$, to obtain $\mathcal{SN}^d_{me-1}$. $\mathcal{SN}^d_{me-1}$ is a list of multiple supernodes that are used to build the sparse Cholesky factor of the precision matrix.

We perform experiments by varying $\rho$ values to compare the prediction errors between the measurement-wise ordering algorithm 1 and the point-wise ordering algorithm 1. The results are reported in \fig \ref{fig:S_GP_error_2D_rho}. For both the groupings, increasing the order of derivatives reduces the prediction error when the $\rho$ is sufficiently large, let us call it the saturation point $\rho_s$, which is dependent on the number of training points. At $\rho _s$, the sparsity of the matrix $\mathbf{U}$ reaches the lower bound, and any further increase in $\rho$ is not expected to have a significant effect on prediction.  When the number of training points is 16, the prediction error plateaus after $\rho=4$, and any further increase in $\rho$ does not result in improved prediction accuracy. When $N$ is 36 and 64, the value of $\rho _{s}$ is 5 and 8, respectively.  When the $\rho < \rho _{s}$, the prediction error increases almost linearly in log scale with a decrease in $\rho$. This is because when the matrix becomes sparse, the information of specific point measurements is lost; thus, the model is expected to have a higher prediction error. Interestingly, the increase in error is much more significant when the derivatives are included in training, as noticed by differences in slope for different orders of derivatives. When the matrix becomes increasingly sparse, we lose the information of points along with their derivatives. We know that the model prediction error is reduced significantly when derivatives are included. On the contrary, we are expected to lose accuracy significantly when some derivative information is lost in the sparse. Similar observations can be made for the results shown in \fig \ref{fig:S_GP_error_1D_rho} of the 1D Griewank function as well.

\begin{figure}[H]
	\centering
        \renewcommand{\thefigure}{A.1.2}
	\includegraphics[width=1\linewidth]{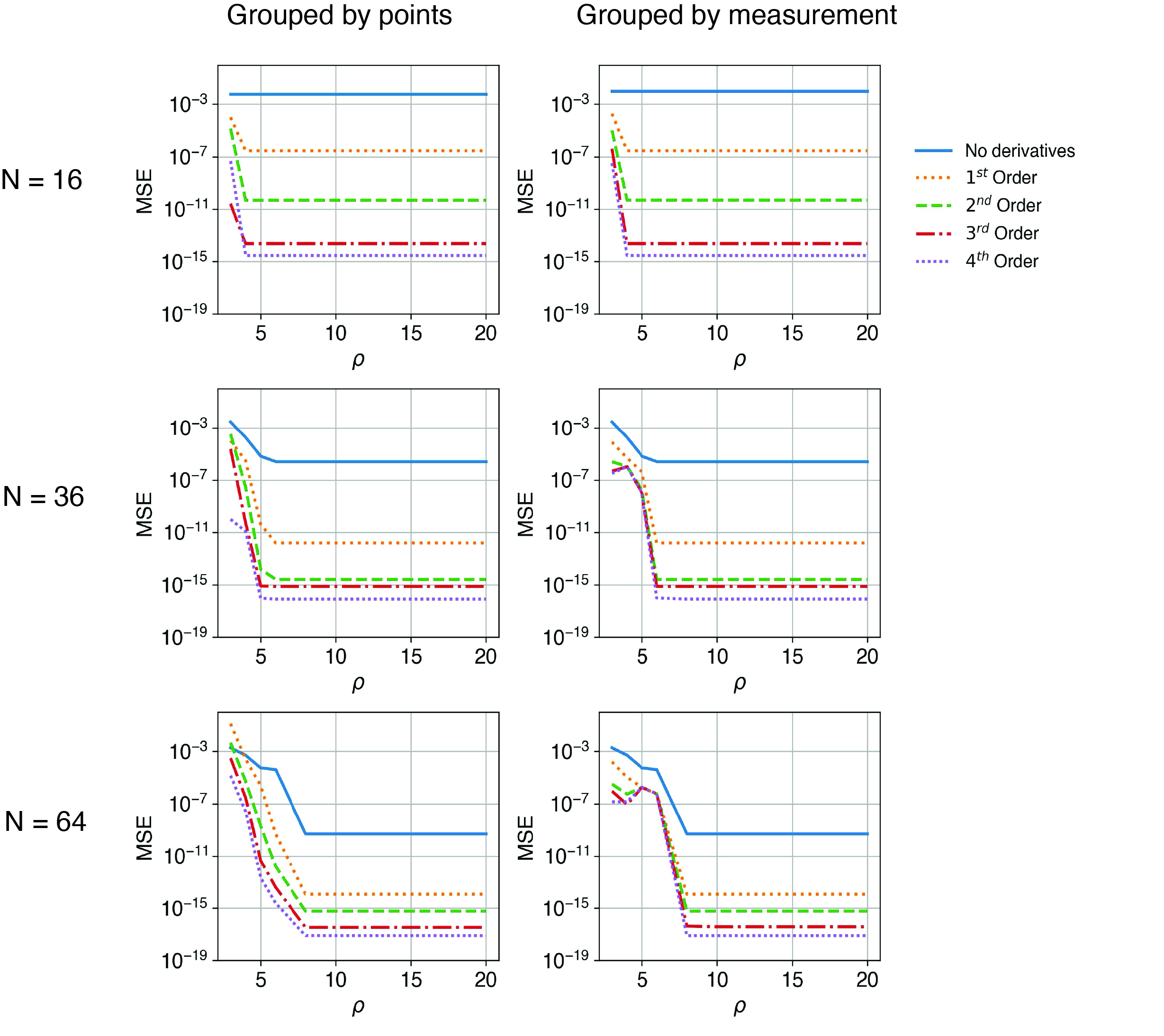}
	\caption{Results of numerical experiments 2D Griewank function from sparse Cholesky GP for different $\rho$ values and order of derivatives. Figures in the left and right columns show the results of the sparse Cholesky factor of the precision matrix when the derivatives are grouped by point-wise ordering algorithm 1 and measurement-wise ordering algorithm 1, respectively.}
	\label{fig:S_GP_error_2D_rho}
\end{figure}

\begin{figure}[H]
	\centering
        \renewcommand{\thefigure}{A.1.3}
	\includegraphics[width=1.2\linewidth]{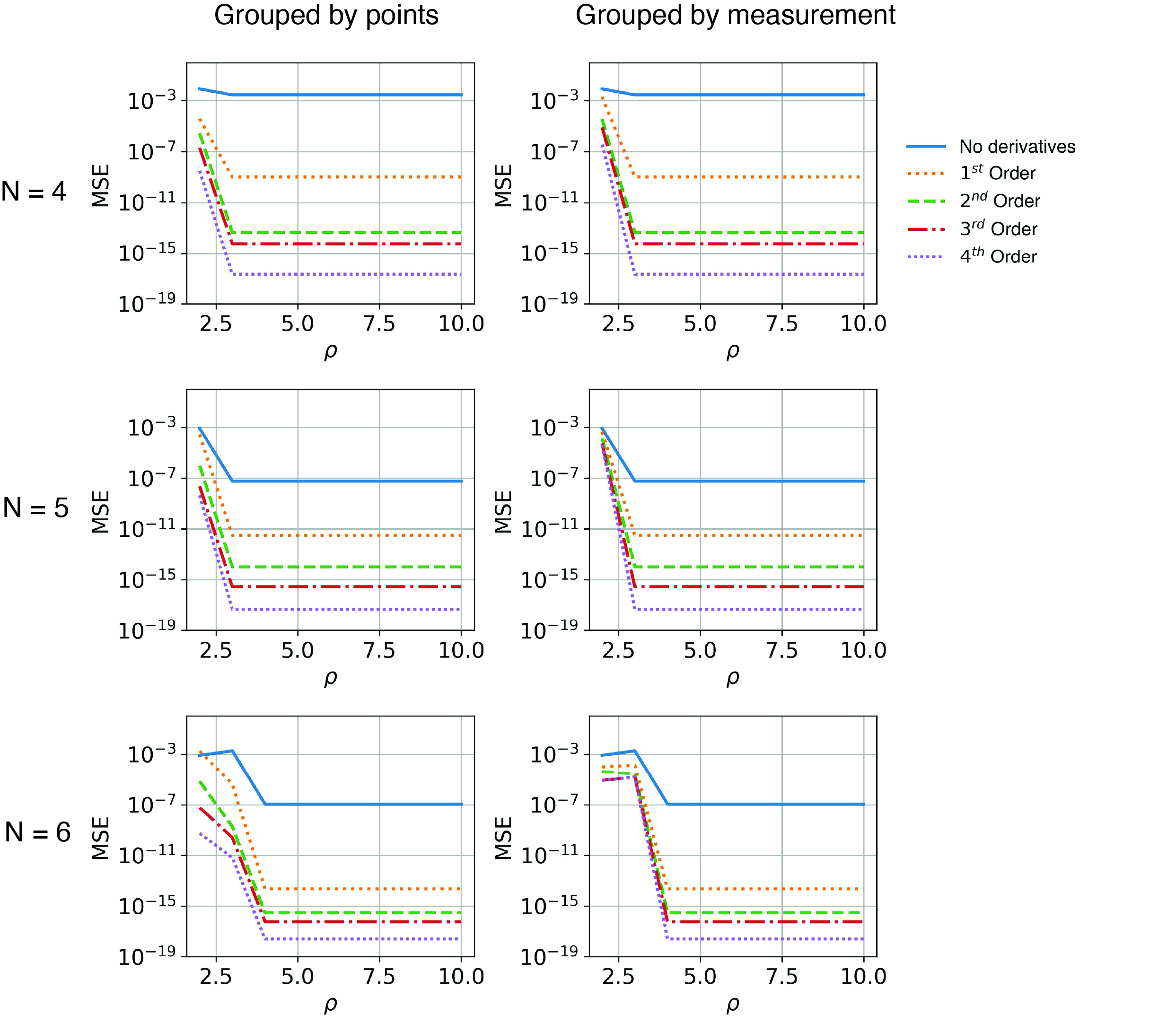}
	\caption{Results of numerical experiments 1D Griewank function from sparse Cholesky GP for different $\rho$ values and order of derivatives.}
	\label{fig:S_GP_error_1D_rho}
\end{figure}

For $\rho < \rho_{s}$, the accuracy of the model is affected by the type of grouping used to include derivative measurements utilized in building the sparse matrix. When the derivatives are grouped by measurement-wise ordering algorithm 1, the prediction error of the model is higher compared to the model when derivatives are grouped by point-wise ordering algorithm 1. For example, \fig \ref{fig:rho_comp} shows the prediction error from measurement grouped by points and measurements for $N=36$ and $\rho=5$. Note that when the derivatives are grouped by the point-wise ordering algorithm 1, the prediction error reduces with an increase in the order of derivatives; however, the reduction in error is smaller when the derivatives are grouped by the measurement-wise ordering algorithm 1. Similar observations can be made for other $N$ and $\rho$ values. Note that the trend of the error is unclear when the matrix is really sparse, for example, $\rho=3$. This suggests that there exists a lower bound of $\rho_L$ below which adding derivatives does not guarantee an improvement in accuracy.

\begin{figure}[H]
	\centering
        \renewcommand{\thefigure}{A.1.4}
	\includegraphics[width=0.8\linewidth]{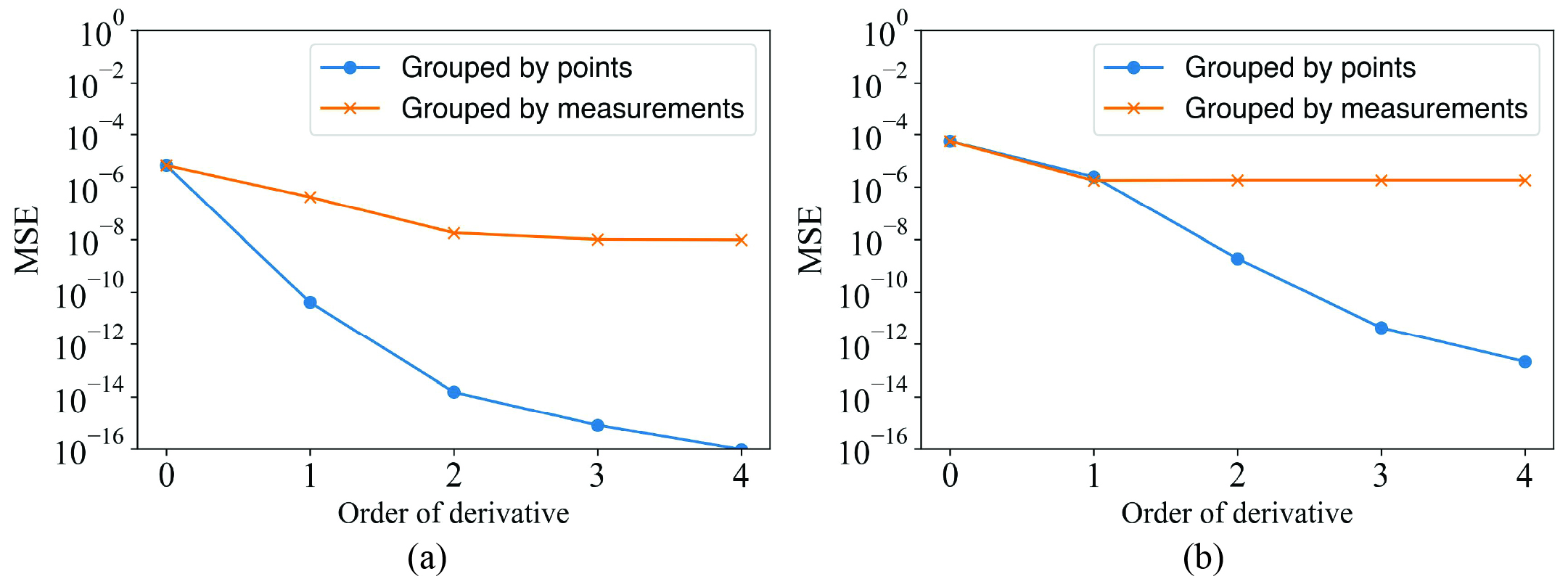}
	\caption{Predictive error of sparse Cholesky GP for point-wise ordering algorithm 1 and measurement-wise ordering algorithm 1 (a) $N=36$ and $\rho =5$, and (b) $N=64$ and $\rho =5$. }
	\label{fig:rho_comp}
\end{figure}

Now, we will turn our attention to two other ordering algorithms. We refer to them as ``point-wise ordering algorithm 2" and ``measurement-wise ordering algorithm 2". In point-wise ordering algorithm 2, we obtain the ordering $\mathbf{P}^d_{po-2}$ by including derivative measurements within Eq.~\ref{eq:P_mmd}. Therefore, point-wise ordering algorithm 2 involves including derivative observations while creating $\mathbf{P}^d_{po-2}$. In contrast, point-wise ordering algorithm 1 involved forcing all the derivative observations of a point to be grouped with that particular point. More precisely, $\mathbf{P}^d_{po-2}$ of point-wise ordering algorithm 2 is created using the following equation:

\[
\mathbf{P}^d_{po-2}(q+1)
=
\arg\max_{i \in I \setminus \{1,\ldots,q\}}
\;
\mathrm{dist}
\Big(
\{\nabla^k f(\mathbf{x}^{(i)})\}_{k=0}^{4},
\{\nabla^k f(\mathbf{x}^{(j)})\}_{j=1,\ldots,q}^{k=0,\ldots,4}
\Big)
\]

\noindent The length scale of the ordered points (for the case of point-wise ordering algorithm 2) is then given by the following:

\[
\mathbf{l}^{(i)}
=
\mathrm{dist}
\Big(
\{\nabla^k f(\mathbf{x}^{\mathbf{P}^d_{po-2}(i)})\}_{k=0}^{4},
\{\nabla^k f(\mathbf{x}^{\mathbf{P}^d_{po-2}(j)})\}_{j=1,\ldots,i-1;\,k=0,\ldots,4}
\Big).
\]

\noindent Once $\mathbf{P}^d_{po-2}$ is created, the supernodes, $\mathcal{SN}^d_{po-2}$, are obtained by the procedure described in Section~\ref{sec:Sparse_GP_review} using $\mathbf{P}^d_{po-2}$. 

For the case of measurement-wise ordering algorithm 2, we obtain $\mathbf{P}^d_{me-2}$ by including derivative measurements at the end of $\mathbf{P}$ (obtained from Eq.~\ref{eq:P_mmd}), i.e., $\mathbf{P}^d_{me-2}$ is obtained by appending derivative measurements, grouped by each derivative type (all $0^{\text{th}}$ order derivatives grouped together, all $1^{\text{st}}$ order derivatives grouped together, etc.), to $\mathbf{P}$ (obtained from Eq.~\ref{eq:P_mmd}). Once $\mathbf{P}^d_{me-2}$ is obtained, $\mathcal{SN}^d_{me-2}$ can be created using $\mathbf{P}^d_{me-2}$ using the procedure described in Section~\ref{sec:Sparse_GP_review}. Both $\mathcal{SN}^d_{po-2}$ and $\mathcal{SN}^d_{me-2}$ are lists of multiple supernodes used to build the sparse Cholesky factor of the precision matrix.

\fig \ref{fig:method_comp_four} shows a detailed comparison between the four algorithms. The figure compares the predictive performances of point-wise ordering algorithm 1, measurement-wise ordering algorithm 1, point-wise ordering algorithm 2, and measurement-wise ordering algorithm 2 with respect to the order of derivatives for varying numbers of training points, N, and varying $\rho$ values. 

\begin{figure}[H]
	\centering
        \renewcommand{\thefigure}{A.1.5}
	\includegraphics[width=1.\linewidth]{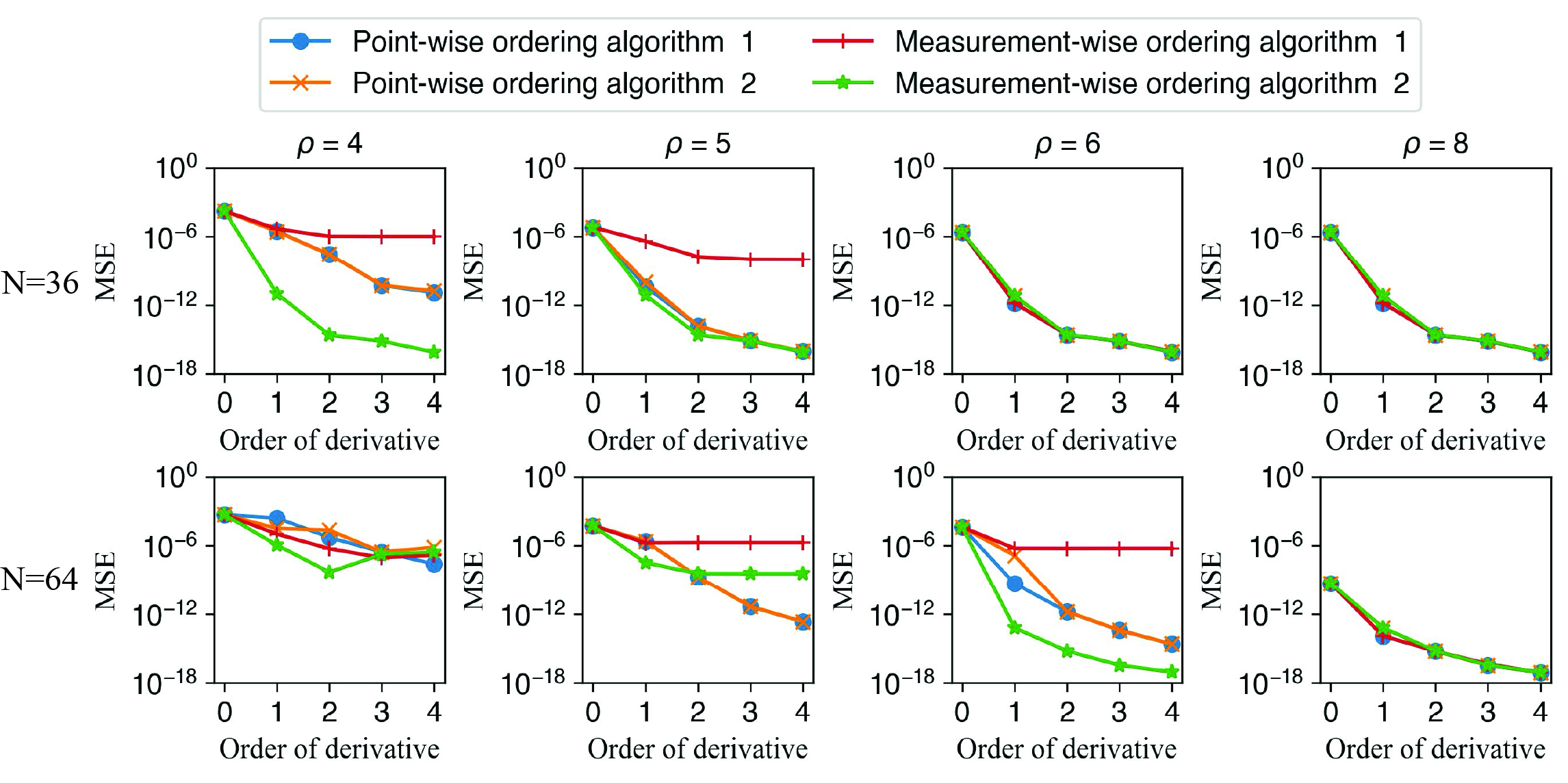}
	\caption{MSE comparison for the four different methods}
	\label{fig:method_comp_four}
\end{figure}

\vspace{1em}
\textbf{A.2.} \textbf{\textit{Additional Results}}

This appendix entails additional experimental results performed by varying $\rho$ values and the number of training points. \fig \ref{fig:S_GP_error_3D_rho} shows the results of the experiments performed by varying $\rho$ for the 3D Griewank function from the sparse Cholesky GP. The experiments are performed for point-wise ordering algorithm 1 and measurement-wise ordering algorithm 1 for a fixed number of training points, N, with varying $\rho$. As the $\rho$ values and the order of derivatives increase, the prediction error reduces for a varying number of training points, N.

\begin{figure}[H]
	\centering
        \renewcommand{\thefigure}{A.2.1}
	\includegraphics[width=0.95\linewidth]{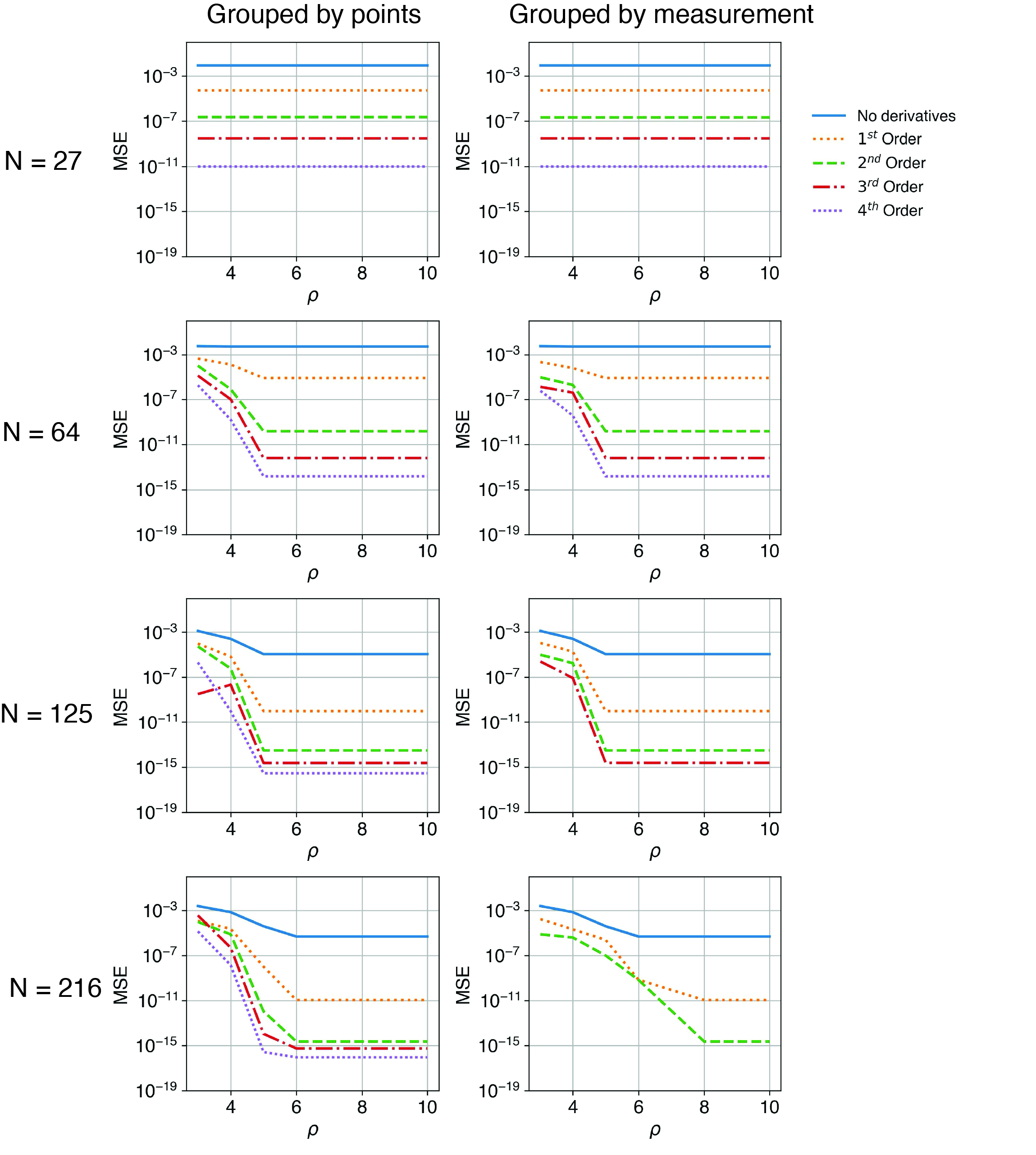}
	\caption{Results of numerical experiments 3D Griewank function from sparse Cholesky GP for different $\rho$ values and order of derivatives}
	\label{fig:S_GP_error_3D_rho}
\end{figure}

 \fig \ref{fig:S_GP_error_2D_N} shows the results of the experiments performed by varying the number of training points, N, for the 2D Griewank function from the sparse Cholesky GP. The experiments are performed for point-wise ordering algorithm 1 and measurement-wise ordering algorithm 1 for a fixed $\rho$ value and varying number of training points. As the number of training points and the order of derivatives increase, the prediction error increases until a threshold of $\rho = 8$ is attained, after which the predictive error shows a decreasing trend. In order for the prediction error to decrease with an increase in the number of training points, $\rho$ needs to satisfy the lower bound, $\rho \gtrapprox \log\left(\frac{N}{\epsilon}\right)$, mentioned in \cite{schafer_sparse_2024}. In \fig \ref{fig:S_GP_error_2D_N}, since we increase the number of training points while keeping $\rho$ fixed, we get an increase in predictive error until the minimum threshold of $\rho = 8$ is reached.

\begin{figure}[H]
	\centering
        \renewcommand{\thefigure}{A.2.2}
	\includegraphics[width=0.92\linewidth]{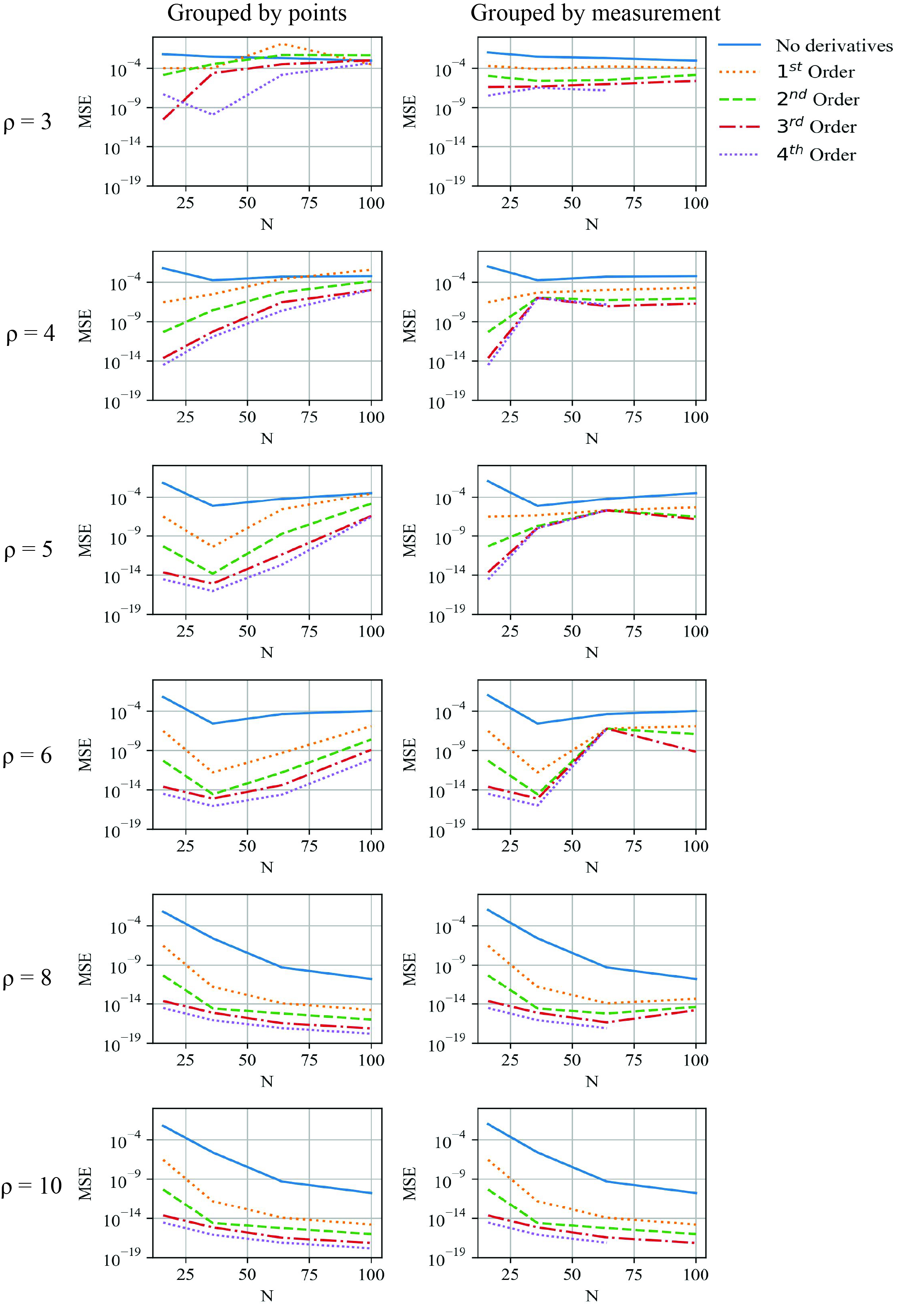}
	\caption{Results of numerical experiments 2D Griewank function from sparse Cholesky GP for different N, $\rho$, and order of derivatives}
	\label{fig:S_GP_error_2D_N}
\end{figure}

\section*{\MakeUppercase{Appendix B. Additional Readings}}
\label{appB}
\setcounter{algorithm}{0}
\setcounter{figure}{0}
\renewcommand{\thealgorithm}{B\arabic{algorithm}}

\noindent \textbf{B.1.} \textbf{\textit{Details of the Outlier Detection Algorithm Mentioned in Section 4}}

\begin{algorithm}[H]
	\caption{Outlier detection }
	\label{alg:outlier_detect}
	\begin{algorithmic}[1]
		\Procedure{IsOutlier}{$x_{new}, X_{train}$}
		\State $X_{exist} \gets X_{train}$ 
		\State Let $D_{k}$ be an empty list and $\eta_{out}$ be the percentile set for outlier detection
		\State Let
		\For{each point $\mathbf{x}_i \in X_{exist}$}
		\State Let $d_k(\mathbf{x}_i, X_{exist})$ be the distance from $\mathbf{x}_i$ to its $k$-th nearest neighbor in $X_{exist}$.
		\State Append $d_k(\mathbf{x}_i, X_{exist})$ to $D_{k}$.
		\EndFor
		\State $\tau \gets \text{Percentile}(D_{k}, \eta_{out})$ \Comment{The outlier threshold}
		\Statex
		
		\State Let $d_k(\mathbf{x}_{new}, X_{exist})$ be the distance from $\mathbf{x}_{new}$ to its $k$-th nearest neighbor in $X_{exist}$.
		\If{$d_k(\mathbf{x}_{new}, X_{exist}) > \tau$}
		\State $  \text{outlier} \gets \text{True} $
		\Else
		\State $\text{outlier} \gets \text{False}$
		\EndIf
		\State \Return outlier
		\EndProcedure
	\end{algorithmic}
\end{algorithm}

\noindent \textbf{B.2.} \textbf{\textit{Details of the Computational Cost Analysis For Section 4}}

The computational cost of the proposed dynamic update algorithms can be analyzed in two different ways. The first corresponds to the practical implementation adopted in this work, where the dynamic supernodes are explicitly re-evaluated after new measurements become available. The second refers to the ideal sparse-update complexity obtained when only the affected columns of the sparse Cholesky factor are updated. Let $N$ denote the total number of training points and let $s$ denote the average number of nonzero entries per column in the sparse Cholesky factor. A complete retraining of the sparse Cholesky GP requires $\mathcal{O}(Ns^2)$ cost with storage complexity $\mathcal{O}(Ns)$.

In SU-Approach1, all dynamic supernodes are merged into a single dynamic supernode, $\mathcal{SN}_{dyn}$. Let $M$ denote the number of points belonging to the dynamic region. Whenever a new point is inserted, the dynamic region is reordered and the Cholesky factors associated with $\mathcal{SN}_{dyn}$ are recomputed. Therefore, the practical update complexity is $\mathcal{O}(Ms^2)$. Since, $M \ll N$, the update remains significantly cheaper than complete retraining. From the perspective of sparse-update, only a small subset of the dynamic supernode is directly affected by the insertion of a new point. Let $k$ denote the number of neighboring points whose sparsity pattern changes due to the insertion. Neighbor identification requires $\mathcal{O}(k\log N)$ cost, while updating the affected Cholesky block requires $\mathcal{O}(k^2)$ cost. Consequently, the ideal sparse-update complexity is $\mathcal{O}(k^2+k\log N)$. This complexity represents the lower bound achievable if only the affected columns of the sparse Cholesky factor are updated. 

For the case of SU-Approach2, dynamic supernodes are maintained separately. Let $M_{loc}$ denote the size of the local dynamic supernode containing the newly inserted point. Since only the affected dynamic supernode is updated, the practical update complexity is $\mathcal{O}(M_{loc}s^2)$. Since $M_{loc} < M$, SU-Approach2 generally requires less computation than SU-Approach1. Similarly, if only the affected columns are updated, the sparse-update complexity becomes $\mathcal{O}(k_{loc}^{2}+k_{loc}\log N)$, where $k_{loc}$ denotes the number of affected neighbors within the local dynamic supernode.  

Thus, the computational comparisons can be summarized as: full retraining ($\mathcal{O}(Ns^2)$), SU-Approach1-practical ($\mathcal{O}(Ms^2)$), SU-Approach1-localized update ($\mathcal{O}(k^2+k\log N)$), SU-Approach2-practical ($\mathcal{O}(M_{loc}s^2)$), and SU-Approach2-localized update ($\mathcal{O}(k_{loc}^{2}+k_{loc}\log N)$). The complexities involving $M$ and $M_{loc}$ correspond to the implementation adopted in this work, where dynamic supernodes are explicitly re-evaluated. The complexities involving $k$ and $k_{loc}$ correspond to the ideal sparse-update cost obtained when only the affected columns of the sparse Cholesky factor are modified. Since, $k_{loc} \le k \ll M_{loc} < M \ll N$, the sparse-update formulation provides the theoretical lower bound on the computational cost, whereas the supernode-based formulation provides a more representative estimate of the computational effort required by the implementation.

\subsection*{\textbf{B.3. Additional Details Regarding the Interpolatory Structure of Derivative-Informed GPs}}
\label{app:interpolatory_structure_GP}

In the limiting case and given sufficient smoothness, $\lim_{\sigma_f^2,\, \sigma_{\nabla f}^2, \, \cdots, \, \sigma_{\nabla f}^d \to 0} \text{GP}$, we recover the exact interpolation of the noise-free GP, i.e., 

\[
\text{GP}_{\text{noisy}} 
\xrightarrow[\sigma_f^2,\, \sigma_{\nabla f}^2,\, \ldots,\, \sigma_{\nabla^d f}^2 \to 0]{} 
\text{GP}_{\text{noise-free}}.
\]

From a computational perspective, the additive noise terms, $\sigma_f^2 \mathbf{I}$, $\sigma_{\nabla f}^2 \mathbf{I}$, $\cdots$, $\sigma_{\nabla^d f}^2 \mathbf{I}$, act similarly to Tikhonov regularization in kernel ridge regression \cite{Hoerl01021970} where the addition of derivative noise introduces a form of regularization, and correspondingly uncertainties in higher-order derivatives impose smoothness constraints on the posterior mean. This point is relevant to interpreting some of our results in Section \ref{sec:App_DT}.

To extend the formulation in \eq~\ref{eq:interpolation_eq} to the derivative-informed setting, we modify \textbf{X} and $\mathbf{y}$ in \eq~\ref{eq:interpolation_eq} to include the derivatives as follows:

\[
\mathbf{X}_{\text{der}} =
\begin{bmatrix}
\mathbf{X}^{(0)} \\[4pt]
\mathbf{X}^{(1)} \\[2pt]
\vdots \\[2pt]
\mathbf{X}^{(d)}
\end{bmatrix},
\qquad
\mathbf{y}_{\text{der}} =
\begin{bmatrix}
f(\mathbf{X}^{(0)}) \\[4pt]
\nabla f(\mathbf{X}^{(1)}) \\[2pt]
\vdots \\[2pt]
\nabla^d f(\mathbf{X}^{(d)})
\end{bmatrix},
\]
\noindent where \(\mathbf{X}^{(k)}\) denotes the set of input points at which the derivatives of order \(d\) are presented. Using this notation, the posterior mean for the derivative-informed GP setting can then be expressed as:
\[
\bar{f}(\mathbf{x}) = \mathbf{K}_{\text{der}}(\mathbf{x}^{*}, \mathbf{X}_{\text{der}})
\big[\mathbf{K}_{\text{der}}(\mathbf{X}_{\text{der}}, \mathbf{X}_{\text{der}}) + \mathbf{R} \big]^{-1}
\mathbf{y}_{\text{der}},
\]
\noindent where $\mathbf{K}_{\text{der}}$ is the covariance matrix from \eq~\ref{eq:GP_der_noisy} and $\mathbf{K}_\text{der}(\mathbf{x}^{*}, \mathbf{X}_\text{der})$ denotes the derivative-informed covariance vector between the test point, $\mathbf{x}^{*}$, and $\textbf{X}_\text{der}$. 

This formulation, which includes derivative information, preserves the interpolatory structure of the derivative-free GP model, i.e., the posterior mean \(\bar{f}(\mathbf{x})\) remains a kernel-based interpolant in the noise-free setting, and the inclusion of $\mathbf{R}$ maintains regularization. By incorporating derivative information through $\mathbf{K}_{\text{der}}$, the mean function becomes a ``curvature-aware interpolant".

		
		
		


\newtheorem{theorem}{Theorem}[section]
\newtheorem{lemma}[theorem]{Lemma}
\newtheorem{proposition}[theorem]{Proposition}

\appendix
\renewcommand{\thesection}{C}
\section*{APPENDIX C. PROOFS}
\addcontentsline{toc}{section}{Appendix C}

\noindent \textbf{C.1.} \textbf{\textit{Proof of Lemma 2.4.1.}} 

By construction, the entries of $\mathbf{K}_{\text{der}}$ are
\[
\mathbf{K}_{\text{der}^{i, j}} \;=\; k^{(n_i,n_j)}(\mathbf{x}^{(i)}, \mathbf{x}^{(j)}),
\]
where $k^{(n_i,n_j)}$ is the mixed partial derivative of the kernel, of order $n_i$ in row $i$ and order $n_j$ in column $j$.

\noindent Since $k$ is positive definite and differentiable, all derivative blocks $k^{(n_i,n_j)}$ satisfy
\[
\sum_{i,j} c_i c_j \, k^{(n_i,n_j)}(\mathbf{x}^{(i)}, \mathbf{x}^{(j)}) \;\ge\; 0 \qquad \forall\, c \in \mathbb{R}^N.
\]

\noindent $\mathbf{K}_{\text{der}}$ is symmetric because mixed derivatives commute for smooth kernels:
\[
\frac{\partial^{n_i+n_j} k}{\partial (\mathbf{x}^{(i)})^{n_i}\,\partial (\mathbf{x}^{(j)})^{n_j}}
\;=\;
\frac{\partial^{n_j+n_i} k}{\partial (\mathbf{x}^{(j)})^{n_j}\,\partial (\mathbf{x}^{(i)})^{n_i}}.
\]

\noindent Therefore, $\mathbf{K}_{der}$ is symmetric and positive definite.
\hfill $\Box$

\vspace{1em}
\noindent \textbf{C.2.} \textbf{\textit{Proof of Lemma 2.4.2.}} 

The GP posterior variance at a test point $\mathbf{x}^*$ is given by

\[
\sigma_d^2(\mathbf{x}^*) =
k(\mathbf{x}^*, \mathbf{x}^*)
- \mathbf{K}_{\mathrm{der}, *, d}\,
  \mathbf{K}_{\mathrm{der}, d}^{-1}\,
  \mathbf{K}_{\mathrm{der}, *, d}^\top.
\]

\noindent where $\mathbf{K}_{\mathrm{der}}^{d}$ is the covariance including derivatives up to order $d$, and $\mathbf{K}_{\mathrm{der}, *, d}$ is the covariance between testing and training points. Now, let $\mathbf{K}_{d-1}$ denote the covariance matrix including derivatives up to order $d-1$, and $\mathbf{K}_{\mathrm{der}, *, d-1}$ be the corresponding cross-covariance with $\mathbf{x}^*$. By construction, adding derivatives of order $d$ adds rows and columns to $\mathbf{K}_{d-1}$ to form $\mathbf{K}_{d}$. These additional blocks correspond to the covariance between the new derivative observations and all the previous observations.

\noindent Thus, $\mathbf{K}_{d}$ can be written as a block matrix:
\[
\mathbf{K}_{d} =
\begin{bmatrix}
\mathbf{K}_{d-1} & \mathbf{B} \\
\mathbf{B}^\top & \mathbf{C}
\end{bmatrix},
\]
where $\mathbf{B}$ is the covariance between the order $d$ derivatives and the existing observations, and $\mathbf{C}$ is the covariance between the order $d$ derivatives. 

\noindent Using Lemma 2.4.1, $\mathbf{C}$ is semi-definite. 
Similarly, the cross-covariance $\mathbf{K}_{d}$ can be written as

\[
\mathbf{K}_{*,d} =
\begin{bmatrix}
\mathbf{K}_{*,d-1} & \mathbf{D}
\end{bmatrix}.
\]

\noindent where $\mathbf{D}$ is the covariance between the test point and the order $d$ derivatives.

\noindent Furthermore, the posterior variance can be written using the Schur complement as follows:
\[
\sigma_d^2(\mathbf{x}^*) = k(\mathbf{x}^*, \mathbf{x}^*) - 
\begin{bmatrix} \mathbf{K}_{*,d-1} & \mathbf{D} \end{bmatrix}
\begin{bmatrix} \mathbf{K}_{d-1} & \mathbf{B} \\ \mathbf{B}^\top & \mathbf{C} \end{bmatrix}^{-1}
\begin{bmatrix} \mathbf{K}_{*,d-1} \\ \mathbf{D} \end{bmatrix}.
\tag{C2.1}
\]

\noindent By the property of Schur complements for a positive semi-definite block $\mathbf{C}$:
\[
\begin{bmatrix} \mathbf{K}_{d-1} & \mathbf{B} \\ \mathbf{B}^\top & \mathbf{C} \end{bmatrix} \succeq \mathbf{K}_{d-1},
\]

\noindent and therefore, the following inequality holds:
\[
\mathbf{K}_{*,d} \mathbf{K}_d^{-1} \mathbf{K}_{*,d}^\top \ge \mathbf{K}_{*,d-1} \mathbf{K}_{d-1}^{-1} \mathbf{K}_{*,d-1}^\top.
\tag{C2.2}
\]

\noindent Using Eq. (C2.2) inequality in the posterior variance formula in Eq. (C2.1) leads to the following expression:
\[
\sigma_d^2(\mathbf{x}^*) = k(\mathbf{x}^*, \mathbf{x}^*) - \mathbf{K}_{*,d} \mathbf{K}_d^{-1} \mathbf{K}_{*,d}^\top
\le k(\mathbf{x}^*, \mathbf{x}^*) - \mathbf{K}_{*,d-1} \mathbf{K}_{d-1}^{-1} \mathbf{K}_{*,d-1}^\top
= \sigma_{d-1}^2(\mathbf{x}^*).
\tag{C2.3}
\]

\noindent Therefore, we can conclude that the following holds true:

\[
\mathbb{E}_{\mathbf{x}^*}[\sigma_d^2(\mathbf{x}^*)] \le \mathbb{E}_{\mathbf{x}^*}[\sigma_{d-1}^2(\mathbf{x}^*)].
\]




\hfill $\Box$

\vspace{1em}
\noindent \textbf{C.3.} \textbf{\textit{Proof of Lemma 2.4.3.}}

The kernel is given by

\[
    k(\mathbf{x}, \mathbf{y})=\sigma^2\exp\!\Big(-\frac{\|\mathbf{x}-\mathbf{y}\|^2}{2\delta^2}\Big),
\]

\noindent By differentiating the kernel, using the chain and product rules, every mixed partial derivative in the derivative-informed kernel can be written in the form
\[
\partial_\mathbf{x}^\alpha\partial_\mathbf{y}^\beta k(\mathbf{x},\mathbf{y}) = p_{\alpha,\beta}(\mathbf{x}-\mathbf{y})\,k(\mathbf{x},\mathbf{y}),
\]
\noindent where $p_{\alpha,\beta}$ is a polynomial whose degree depends only on $|\alpha|+|\beta|$. Hence, the inequality
\[
\big|\partial_\mathbf{x}^\alpha\partial_\mathbf{y}^\beta k(\mathbf{x},\mathbf{y})\big| \le \sup_{z\in\mathbb{R}^p} |p_{\alpha,\beta}(z)|\; k(\mathbf{x},\mathbf{y})
\] 
\noindent holds. By setting $\gamma:=1/(2\delta^2)$ in 
\[
    k(\mathbf{x},\mathbf{y})=\sigma^2\exp\!\Big(-\frac{\|\mathbf{x}-\mathbf{y}\|^2}{2\delta^2}\Big),
\]
\noindent and by choosing $C_{\alpha,\beta}:=\sigma^2\sup_{z}|p_{\alpha,\beta}(z)|$, the following bound is obtained

\[
\big| \partial_\mathbf{x}^\alpha \partial_\mathbf{y}^\beta k(\mathbf{x},\mathbf{y})\big| \le C_{\alpha,\beta}\, \exp\!\big(-\gamma\|\mathbf{x}-\mathbf{y}\|^2\big).
\]

\hfill $\Box$

\vspace{1em}
\noindent \textbf{C.4.} \textbf{\textit{Proof of Lemma 3.0.1.}} 

The number of unique derivative terms of order $d$ in $p$ dimensions is given by:
\[
\binom{p+d-1}{d}.
\]
Multiplying by $N$ gives the total size $N_d$. The symmetry of $\mathbf{K}_{\text{der}}$ follows from kernel derivative symmetry, and positive definiteness simply follows from Lemma 2.4.1.

\hfill $\Box$

\vspace{1em}
\noindent \textbf{C.5.} \textbf{\textit{Proof of Lemma 3.0.2.}} 

Adding higher-order derivatives increases the differences between nearby points, which makes $\mathbf{K}_{\text{der}}$ more ill-conditioned. Mathematically, the derivative magnitude scales roughly as 
\[
\frac{\partial^d k(\mathbf{x}^{(i)}, \mathbf{x}^{(j)})}{\partial x^d} \sim \delta^{-d} k(\mathbf{x}^{(i)}, \mathbf{x}^{(j)}),
\]
\noindent So increasing $d$ increases the condition number and increasing $l$ smoothens the kernel, which reduces the derivative block magnitude.

\hfill $\Box$

\renewcommand{\thesection}{D}
\section*{APPENDIX D. ADDITIONAL THEORETICAL RESULTS}
\addcontentsline{toc}{section}{Appendix D}

\begin{lemma}[Localization of block covariances]
\label{thm:C6}
Let there be two fixed training points $\mathbf{x}^{(i)}$ and $\mathbf{x}^{(j)}$ and let $\mathbf{B}_{ij}$ denote the covariance block coupling any finite collection of derivative components at $\mathbf{x}^{(i)}$ with any finite collection at $\mathbf{x}^{(j)}$, then there exits constants $C, \gamma>0$ such that
\[
\|\mathbf{B}_{ij}\| \le C \exp(-\gamma\|\mathbf{x}^{(i)}-\mathbf{x}^{(j)}\|^2)
\]
Essentially, block coupling decays exponentially with the square of the distance.

\begin{proof}
Since each entry of $\mathbf{B}_{ij}$ is of the form $\partial_\mathbf{x}^\alpha\partial_\mathbf{y}^\beta k(\mathbf{x}^{(i)},\mathbf{x}^{(j)})$, so Lemma 2.3.3. gives an exponential bound on each entry. The operator norm of the finite block is bounded by a fixed multiple of the maximal absolute entry, so the same exponential decay holds for the block norm.
\end{proof}

\end{lemma}
\vspace{1em}

\begin{lemma}[Supernode aggregation and computational cost]
\label{thm:C9}
Suppose the columns of the covariance matrix, $\mathbf{K}_{der}$, are aggregated into $n$ supernodes, each of size at most $m$, such that each supernode interacts with at most $O(m)$ neighbors, then:
\begin{enumerate}
    \item Building or updating the cholesky factorization under the sparsity structure requires $O(n m^2)$ computational work.
    \item Restructuring a single supernode in a dynamic update costs $O(m^3)$ arithmetic operation and touches $O(m^2)$ entries.
\end{enumerate}
Thus, dynamic updates involving only one supernode are substantially cheaper than restructuring the entire factorization.
\end{lemma}

\begin{proof}
Each supernode can be treated as a dense block matrix of size at most $m$. The dense cholesky factorization of a block of size $m$ costs $O(m^3)$ operations. Since there are $n$ such blocks, and each interacts with only $O(m)$ neighbors, the total work for assembling or updating the global factorization scales as $O(n m^2)$. This includes both the factorization of each supernode and the updates to its neighboring blocks.
In the case of a dynamic update where only one supernode changes, the update requires recomputing the dense factorization of its blocks, which costs $O(m^3)$ arithmetic operations. The propagation of updates to adjacent blocks needs modifying $O(m^2)$ entries because each neighboring interaction is at most of size $m \times m$. Thus, the dynamic update cost is cubic in $m$. 

Now consider a dynamic update where only one supernode changes. The update requires recomputing the dense factorization of its block, which costs $O(m^3)$ arithmetic. The propagation of updates to adjacent blocks requires modifying $O(m^2)$ entries, since each neighboring interaction is at most of size $m \times m$. Therefore, the dynamic update cost is cubic in $m$.
\end{proof}
\vspace{1em}

\begin{lemma}[Number of supernodes in ``point-wise ordering algorithm 1" and ``measurement-wise ordering algorithm 1"]
\label{thm:C10}
Given a p-dimensional space and derivatives up to the order d. Suppose that the number of supernodes created using point-wise ordering algorithm 1 is $SN_{p}$ and the number of supernodes created using the measurement-wise ordering algorithm 1 is $SN_{m}$. Then, $SN_{p}$ and $SN_{m}$ are related by the following equation:
\[
SN_{m} = z \, SN_{p}
\]

\noindent where z is given by
\[
z = \sum_{k=0}^d \binom{p + k - 1}{k}
\]
\end{lemma}

\begin{proof}
For each measurement type in the measurement-wise ordering algorithm 1, the relative ordering of the points within that type is similar to the point-wise ordering algorithm 1. Hence, the supernode partition that applies to the point-wise ordering algorithm 1 is the same in size within each measurement-type block in the measurement-wise ordering algorithm 1. Therefore, each measurement type within the measurement-wise ordering algorithm 1 contributes exactly $SN_{p}$ supernodes, so the total number of supernodes in measurement-wise ordering algorithm 1 is given by
\[
SN_{m} = z \, SN_{p}
\]

\noindent Here, $z$ is the number of distinct measurement types per point in the measurement-type ordering algorithm 1.

\noindent Now, a partial derivative of $f$ can be indexed by a multi-index
\[
\alpha = (\alpha_{1}, \alpha_{2}, \alpha_{3}, ....., \alpha_{p}) \in \mathbb{N}^{p}
\]

\noindent where, $\nabla^{\alpha} f(\mathbf{x}) = \frac{\partial^{|\alpha|} f}{\partial \mathbf{x}_1^{\alpha_1} \cdots \partial \mathbf{x}_p^{\alpha_p}}, \qquad
|\alpha| := \alpha_1 + \alpha_2 + \cdots + \alpha_p.$

\noindent Fixing k $\geq 0$ and then the set of all order k-partial derivatives corresponds to the set
\[
A_{p. k} = {\alpha \in \mathbb{N}^{p} : |\alpha| = k}
\]

\noindent The cardinality of this set is the number of nonnegative integer solutions to
\[
\alpha_{1} + \alpha_{2} + \alpha_{3} + .... + \alpha_{p} = k
\]

\noindent By the stars-and-bars theorem from combinatorics, 
\[
|A_{p, k}| = \binom{p + k - 1}{k}
\]

\noindent Summing up to the order of derivatives $d$ gives,
\[
\sum_{k = 0}^d |A_{p, k}| = \sum_{k = 0}^d \binom{p + k - 1}{k}
\]

\noindent which is the value of $z$, i.e.,

\[
z = \sum_{k = 0}^d \binom{p + k - 1}{k}.
\]

\end{proof}

\bibliographystyle{unsrtnat}
\bibliography{ref}










\end{document}